%% file: main.tex
\documentclass[12pt]{article}
\usepackage[margin=1in]{geometry}
\usepackage{times}
\usepackage{hyperref}

\usepackage{url,booktabs}
\usepackage{nicefrac,microtype,xcolor}
\usepackage{graphicx,tikz}
\usetikzlibrary{calc,arrows.meta}

\usepackage{amsmath,amsthm,amssymb,amsfonts}
\usepackage{mathtools,mathrsfs,bm}

\usepackage{microtype}
\usepackage{graphicx}
\usepackage{caption, subcaption}
\usepackage{booktabs} 
\usepackage{tabularx}

\usepackage{xcolor}
\definecolor{verylightgray}{rgb}{0.9, 0.9, 0.9}
\definecolor{lightyellow}{rgb}{0.97, 0.97, 0.6}
\definecolor{lightcyan}{rgb}{0.8, 1, 1}
\definecolor{peach}{rgb}{1, 0.8, 0.6}
\definecolor{lightpurple}{rgb}{0.9, 0.7,0.7}
\usepackage{utfsym}
\usepackage{hyperref}
\usepackage{acronym}
\usepackage{mdframed}

\usepackage[dvipsnames]{xcolor}
\usepackage{amsmath}
\usepackage{amssymb}
\usepackage{mathtools}
\usepackage{amsthm}
\usepackage{multirow}
\usepackage{colortbl}
\usepackage{wrapfig}
\usepackage{tikz}
\newcommand*\circled[1]{\tikz[baseline=(char.base)]{
            \node[shape=circle,draw,inner sep=0.5pt] (char) {#1};}}
            
\usepackage[capitalize,noabbrev]{cleveref}
\acrodef{llm}[LLM]{large language model}
\acrodef{ood}[OOD]{out-of-distribution}
\acrodef{qa}[QA]{question answering}
\acrodef{gd}[GD]{gradient descent}
\acrodef{mlp}[MLP]{multi-layer perceptron}
\acrodef{pt}[PT]{pre-training}
\acrodef{ft}[FT]{fine-tuning}
\acrodef{svd}[SVD]{singular value decomposition}
\acrodef{ntp}[NTP]{next token prediction}

\usepackage[textsize=tiny]{todonotes}

\let\oldfrac\frac
\renewcommand{\frac}[2]{%
	\mathchoice
	{\oldfrac{#1}{#2}}
	{#1/#2}
	{#1/#2}
	{#1/#2}
}

\input{macros}

\usepackage{caption}
\usepackage{subcaption}
\usepackage{algorithm, algpseudocode}
\usepackage{authblk}
\usepackage{etoc}
\usepackage{wrapfig} 

\usepackage{titlesec}
\titlespacing*{\paragraph}{0pt}{3.2pt plus 1pt minus 1pt}{1em}

\hypersetup{
  colorlinks,
  linkcolor={red!60!black},
  citecolor={blue!60!black},
}

\ifdefined\usebigfont

\usepackage{times}
\usepackage[fontsize=13pt]{scrextend}
\makeatletter
\@ifpackageloaded{geometry}{\AtBeginDocument{\newgeometry{letterpaper,left=1.56in,right=1.56in,top=1.71in,bottom=1.77in}}}{\usepackage[letterpaper,left=1.56in,right=1.56in,top=1.71in,bottom=1.77in]{geometry}}
\AtBeginDocument{\newgeometry{letterpaper,left=1.56in,right=1.56in,top=1.71in,bottom=1.77in}}
\usepackage{hyperref} 

\else
\fi

\title{Provable Knowledge Acquisition and Extraction in One-Layer Transformers}

\usepackage{tikz}

\usepackage{csquotes}
\MakeOuterQuote{"}

\author{
{Ruichen Xu}\thanks{Correspondence to: rcxu642@gmail.com}\qquad
{Kexin Chen}
}
\date{}
\begin{document}

\maketitle 

\vspace{-3mm}

\begin{abstract}%
Large language models may encounter factual knowledge during pre-training yet fail to reliably use that knowledge after fine-tuning. 
Despite growing empirical evidence that MLP layers store factual associations and fine-tuning affects factual recall, the training-dynamics mechanisms linking next-token pre-training, knowledge storage, and post-fine-tuning extraction remain poorly understood. 
We study this problem in a stylized one-layer transformer with self-attention and MLP modules, trained by next-token prediction and subsequently fine-tuned on question-answering data.
Under suitable regularity conditions, we first prove that the model reaches near-optimal pre-training loss while learning structured attention patterns and relation-specific feature directions, giving a mechanism for factual knowledge acquisition. 
We then show that fine-tuning can turn the Q\&A prompt format into a trigger for pre-trained relation features, enabling the model to extract facts that are not revisited during fine-tuning. 
Our analysis yields a \textbf{relation-covering characterization} of knowledge extraction: fine-tuning need not revisit every stored subject-answer pair, but it must cover enough latent relation-template directions through which facts were encoded during pre-training. 
Consequently, extraction improves with pre-training multiplicity and fine-tuning coverage, but becomes harder as the relation-template universe grows. 
Conversely, insufficient coverage leads to a failure regime in which facts may be stored but remain inaccessible, providing a stylized mechanism for hallucination.
The theory applies to \emph{both full and low-rank fine-tuning}, offering insight into why low-rank adaptation can recover pre-trained factual knowledge when relation coverage is sufficient. 
Experiments on synthetic data and PopQA-based GPT-2/Llama models support the predicted trends.
\end{abstract}

\input{main_text}

\bigskip

\bigskip

\bibliography{ref}

\newpage
\tableofcontents

\appendix
\newpage
\crefalias{section}{appendix}

\input{appendix}

\end{document}

%% file: macros.tex
\usepackage{amsmath, amssymb, amsthm, amsfonts}
\usepackage{times}
\usepackage{xcolor}
\usepackage{graphicx}
\usepackage{enumitem}
\usepackage{cleveref}
\usepackage{titlesec}
\usepackage{mathrsfs}
\usepackage{physics}

\usepackage{booktabs}
\usepackage{tensor}
\usepackage{nicematrix}
\usepackage{thm-restate}
\usepackage[numbers]{natbib}

\DeclareMathOperator{\softmax}{softmax}

\titlespacing*{\section}{0pt}{6pt}{0pt}
\titlespacing*{\subsection}{0pt}{6pt}{0pt}

\DeclareMathOperator{\poly}{poly}

\newtheorem{lemma}{Lemma}

\newtheorem{theorem}{Theorem}
\newtheorem{proposition}{Proposition}

\newtheorem{condition}{Condition}

\theoremstyle{remark}
\newtheorem{remark}{Remark}

\ifdefined\usebigfont

\usepackage{times}
\usepackage[fontsize=13pt]{scrextend}
\AtBeginDocument{
	\newgeometry{letterpaper,left=1.56in,right=1.56in,top=1.71in,bottom=1.77in}
}
\else
\fi

%% file: main_text.tex
\allowdisplaybreaks

\section{Introduction}
Modern Large Language Models (LLMs) function as powerful knowledge repositories, demonstrating a remarkable ability to perform a wide range of tasks, from answering complex questions to writing creative text \citep{brown2020language}. 
These capabilities are not explicitly programmed but emerge from the models' training process. 
During their training phase, these models implicitly encode vast amounts of information from massive datasets directly into their billions of internal parameters. 
When a user provides a prompt, it acts as a key to navigate this complex parameter space, causing the model to synthesize the relevant information by predicting a statistically probable sequence of text based on the prompt's context and its encoded knowledge.

A critical challenge, however, is that knowledge encoded during training is not always accessible at inference time.
Even when an LLM has successfully learned a piece of information, it can fail to retrieve or correctly apply that knowledge when prompted \citep{allen2023physics1,kandpal2023large}.
This retrieval failure often leads the model to generate plausible but incorrect information, a phenomenon widely known as "hallucination." Therefore, demystifying the precise mechanics of how these models acquire and extract knowledge is fundamental to building genuinely reliable and trustworthy AI systems.

Research into the knowledge mechanisms of \acp{llm} has progressed from empirical observation to theoretical modeling. 
Empirically, studies have identified the MLP layers as the primary location for knowledge storage, where they act as key-value memories \citep{geva2020transformer,dong2025attention,yao2024knowledge}.
Others have explored and established approaches for probing and steering the factual knowledge stored within an LLM's parameters \citep{allen2023physics1,allen2023physics2,meng2022locating}.
However, while these empirical findings reveal the mechanisms at play in knowledge storage, they don't fully explain why these knowledge mechanisms emerge from the training process. 
To bridge this gap, a growing body of theoretical research now seeks to formally characterize the knowledge acquisition process by analyzing knowledge acquisition mechanisms with simplified models at distinct phases.
For example, \citet{nichani2024understanding} investigated the storage capacities of one-layer transformers in conjunction with associative memories.
\citet{ghosal2024understanding} studied the gradients of one-layer attention-only transformers during fine-tuning with subject-answer pairs.
These works provide important insights into storage capacity and local fine-tuning behavior, but they leave open the generalization question central to knowledge extraction: why can fine-tuning on some facts unlock other facts acquired during pre-training, and when does this transfer fail?

Answering this question requires an end-to-end account of how factual knowledge is acquired during pre-training, how it is represented in model features, and how fine-tuning connects a new question-answering format to those pre-trained representations. Such an account must jointly address three ingredients that are central to modern LLMs but difficult to analyze together:

\begin{itemize}
\item \textbf{Interplay of self-attention and MLP.}
Since analyzing a full transformer is analytically challenging, the majority of existing theoretical analyses on training dynamics concentrate on attention-only transformer architectures.\footnote{See Section \ref{subsec: related} for a detailed discussion.}
However, this simplification overlooks the critical role of the MLP, which extensive empirical evidence identifies as the primary component for knowledge storage in transformer-based language models (e.g., \cite{geva2020transformer, meng2022locating, dong2025attention, yao2024knowledge}).  
To address this limitation, we examine a simplified one-layer transformer comprising both self-attention and \ac{mlp} modules. 
Although this model is simplified from a full transformer, elucidating its knowledge acquisition mechanisms remains highly nontrivial due to the intricate interplay between self-attention and \ac{mlp}: the output of the self-attention module directly feeds into the \ac{mlp}, resulting in tightly coupled parameters and entangled gradients.

\item \textbf{Fine-tuning for adaptivity.}
The capacity for transformers to generalize their knowledge to downstream tasks \citep{bubeck2023sparks,dubey2024llama} is typically unlocked via fine-tuning, a process encompassing methods from task-specific adaptation \citep{devlin2019bert,radford2018improving} to large-scale instruction tuning \cite{ouyang2022training}.
However, a central challenge lies in characterizing the dynamics of this adaptation process.
In contrast to training from scratch, which is extensively studied, fine-tuning dynamics are critically dependent on the rich features inherited from the pre-trained model.
This interdependency necessitates a joint theoretical consideration of both the pre-training and fine-tuning phases.

\item \textbf{Next-token prediction.}
\ac{ntp} is the standard loss function for large language models pre-training \citep{achiam2023gpt,dubey2024llama}.
Exploring the implicit bias introduced by \ac{ntp} is necessary for understanding modern LLMs. 
However, analyzing \ac{ntp} is challenging because of its non-i.i.d. data structure.
\ac{ntp} diverges from simpler, well-studied learning paradigms in two fundamental ways: first, the strong correlations among data samples drawn from the same sequence; and second, the multi-label nature of the samples, where a single context may admit multiple valid subsequent tokens.
\end{itemize}

In this paper, we seek to fill the above gap to answer the following key open research questions:
\begin{enumerate}
    \item \textit{How do transformers acquire knowledge during \ac{pt} via \ac{ntp} and extract it after \ac{ft}?}
    \item \textit{Under which conditions can pre-trained transformers perform generalization after full and low-rank \ac{ft}?}
\end{enumerate}

To answer these questions, we analyze the feature-learning dynamics of one-layer ``self-attention + MLP'' transformers across both next-token pre-training and question-answering fine-tuning. Our main contributions are as follows.
\begin{itemize}
\item \textbf{A training-dynamics framework for knowledge acquisition and extraction.}
We develop a theoretical framework for tracking how self-attention and MLP features evolve during pre-training and fine-tuning. Under suitable regularity conditions, we prove that the model reaches near-optimal next-token prediction loss during pre-training. Our analysis shows that self-attention learns to suppress irrelevant context, while the MLP stores structured subject-relation-answer features.

\item \textbf{A relation-covering characterization of post-fine-tuning extraction.}
We show that successful knowledge extraction after fine-tuning is governed by a relation-covering condition. Fine-tuning need not revisit every stored subject-answer pair; instead, it must cover enough latent relation-template directions through which facts were encoded during pre-training. When this condition holds, the fine-tuned model achieves low generalization error on factual knowledge acquired during pre-training but not revisited during fine-tuning.

\item \textbf{Success and failure regimes for factual extraction.}
We provide explicit generalization bounds showing that extraction improves with pre-training multiplicity $K$ and fine-tuning size $\widetilde{N}_f$, but becomes harder as the relation-template universe $|\mathcal{R}|$ grows. In particular, sufficient relation coverage yields exponentially small error, whereas insufficient coverage, e.g., $\widetilde{N}_fK=\mathcal{O}(|\mathcal{R}|)$, can leave stored facts inaccessible under the downstream prompt format, producing a stylized hallucination regime.

\item \textbf{A theoretical explanation for low-rank fine-tuning.}
We extend the analysis to low-rank fine-tuning and show that the fine-tuning gradients are dominated by a low-rank component aligned with the question-answering format. This explains why low-rank adaptation can transfer the downstream prompt format to pre-trained relation features under stronger coverage requirements.
\end{itemize}

We complement the theory with experiments on synthetic data and PopQA-based settings using GPT-2 and Llama-3.2-1B, which support the predicted dependence on pre-training multiplicity, fine-tuning coverage, and the size of the relation-template universe.


\subsection{Related Work}\label{subsec: related}
Our work is most closely related to theoretical analyses of transformer training dynamics. 
We review the recent literature on this topic below. 
Due to the space constraints, a broader discussion of related topics, including transformer memorization, feature learning, and empirical studies of knowledge acquisition and extraction, is deferred to Appendix \ref{app: related_work}.
Extensive theoretical research has examined the (pre-)training and fine-tuning of transformers, often focusing on training dynamics within simplified architectural settings. 
For example, the dynamics of one-layer attention-only transformers have been explored for tasks such as next-token prediction in long sequences \citep{tian2023scan} and image classification \citep{jelassi2022vision}.
Additionally, the full training dynamics of linear one-layer attention-only transformers have been studied for factual recall \citep{nichani2024understanding} and in-context learning \citep{zhang2024trained}. 
Other architectures have also been explored.
For example, \citet{cabannes2024learning} analyzed the training dynamics of learning associative memories with a single linear layer.

A complementary line of research examines transformer training through gradients and critical points. 
For example, \citet{bietti2023birth} investigated gradients in two-layer attention-only transformers to elucidate the emergence of induction heads, while \citet{ghosal2024understanding} analyzed single-step gradients in one-layer attention-only transformers during fine-tuning. 
However, for non-convex transformers, analyses limited to critical points and gradients are insufficient to capture the full training dynamics or the properties of the converged solution.

Our work distinguishes itself by characterizing the full joint training dynamics of self-attention and MLP components during \ac{ntp}-based \ac{pt}, as well as the full and low-rank joint dynamics during subsequent \ac{ft}. 
Moreover, we evaluate the generalization performance of the fine-tuned transformers. 
A detailed comparison of our approach and characterizations with prior work is presented in Table \ref{table: comparison}.

\subsection{Notation}
We use lowercase letters for scalars, lowercase boldface letters for vectors, and uppercase boldface letters for matrices. 
The set $\{1,\cdots,m\}$ is denoted by $[m]$.
Given two sequences $\{x_n\}$ and $\{y_n\}$, we denote $x_n = \mathcal{O}(y_n)$ if $|x_n|\le C_1|y_n|$ for some positive constant $C_1$ and $x_n = \Omega(y_n)$ if $|x_n|\ge C_2|y_n|$ for some positive constant $C_2$.
We use $x_n = \Theta(y_n)$ if $x_n = \mathcal{O}(y_n)$ and $x_n = \Omega(y_n)$ both hold.
The notations $\tilde{\mathcal{O}}(\cdot), \tilde{\Theta}(\cdot),$ and $\tilde{\Omega}(\cdot)$ are used to suppress logarithmic factors.
For a matrix $\mathbf{A}$, both $A_{i,j}$ and $(\mathbf{A})_{i,j}$ denote the element of $\mathbf{A}$ located at the $i^{th}$ row and the $j^{th}$ column.

\section{Simplified One-Layer Transformer Architecture} \label{sec: model}

\begin{figure}
    \centering
	\includegraphics[width=.5\linewidth]{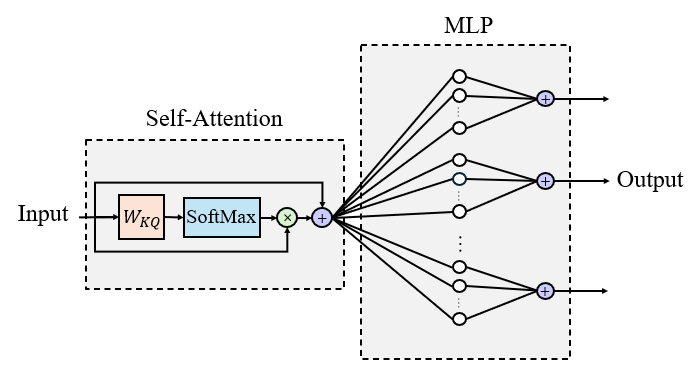}
	\caption{Illustration of the simplified one-layer transformer.}
	\label{fig: transformer}
\end{figure}

Standard one-layer transformers typically comprise seven weight matrices and two normalization operations. 
Analyzing their learning dynamics is challenging due to the complex interactions among these matrices and the non-smooth, non-convex loss landscape. 
To address these issues and enhance analytical tractability, we introduce several simplifications to the standard architecture. 
The resulting model, as empirically validated (Table \ref{table: ood1} in Appendix \ref{app: table}), retains the core mechanisms of knowledge acquisition and extraction. 
An overview of the architecture is depicted in Figure \ref{fig: transformer}. 
We describe its components below.

\paragraph{Self-Attention.} 
The attention mechanism computes scores for an input sequence $\mathbf{X} \in \mathbb{R}^{d\times L}$ (where $d$ is the embedding dimension and $L$ is the sequence length), using the query derived from the last token.
The attention scores are computed by:
\begin{align}
	\boldsymbol{\alpha}(\mathbf{Z}, \mathbf{X})
	 = \!\text{softmax}\!\left(\frac{\mathbf{X}^\top\overbrace{\mathbf{W}_K^\top\mathbf{W}_Q}^{\mathbf{Z}}\mathbf{X}[-1]}{\sqrt{d}}\right),
\end{align}

where $\mathbf{X}[-1] \in \mathbb{R}^{d \times 1}$ denotes the embedding of the last token.\footnote{We consider the standard setting where the transformer predicts the next token using only the output at the final position. 
In a one-layer transformer, this equates to querying solely from the last token.} 
Following established simplifications \cite{tian2023scan}, we reparameterize $\mathbf{W}_K^\top \mathbf{W}_Q$ as a single matrix $\mathbf{Z} \in \mathbb{R}^{d \times d}$. 
The self-attention output $\mathbf{x}_a \in \mathbb{R}^{d}$ is then:
\begin{equation}
	\begin{aligned}
		\mathbf{x}_a(\mathbf{Z},\mathbf{X}) = \mathbf{X}\boldsymbol{\alpha}(\mathbf{Z}, \mathbf{X}) + \mathbf{X}[-1].
	\end{aligned}
\end{equation}

\paragraph{Multi-Layer Perceptron.}
The \ac{mlp} processes the self-attention output $\mathbf{x}_a$. 
With trainable parameters $\mathbf{W} \in \mathbb{R}^{d m \times d}$ (first layer) and fixed parameters $\mathbf{W}_2 \in \mathbb{R}^{d \times d m}$ (second layer), the MLP output is:
\begin{align}
	\mathbf{x}_f(\mathbf{W},\mathbf{Z}, \mathbf{X}) =\mathbf{W}_2 \sigma(\mathbf{W}\mathbf{x}_a(\mathbf{Z},\mathbf{X})),
\end{align}
where $\sigma(\cdot)$ denotes the ReLU activation function.
We fix the matrix $\mathbf{W}_2$ as:
\begin{equation}
	\mathbf{W}_2 = 
	\begin{bmatrix}
		\oldfrac{1}{m}\mathbf{1}^\top & \mathbf{0}^\top & \cdots & \mathbf{0}^\top\\
		\vdots & \vdots & \vdots & \vdots\\
		\mathbf{0}^\top & \mathbf{0}^\top & \cdots & \oldfrac{1}{m}\mathbf{1}^\top
	\end{bmatrix},
\end{equation}
where $\mathbf{0},\mathbf{1}\in \mathbb{R}^m$ are all-zeros and all-ones vectors, respectively.
This block-diagonal structure, commonly employed in MLP analyses \citep{allen2020towards, cao2022benign}, decouples neurons across output dimensions, thereby simplifying the theoretical treatment.

\paragraph{Linear Layer.}
Then, the MLP output is fed into a final linear layer:
\begin{align}\label{equ: output}
	\mathbf{x}^\text{output}(\mathbf{W},\mathbf{Z}, \mathbf{X}) = \lambda\mathbf{I}\mathbf{x}_f(\mathbf{W},\mathbf{Z}, \mathbf{X}),
\end{align}
where $\lambda$ is a constant scaling factor and $\mathbf{I}$ is the identity matrix. 
This scaling mimics learnable gains or fixed normalization in layers like RMSNorm, which typically scales inputs to have an L2 norm of $\sqrt{d}$.

\begin{remark} 
    Our simplified transformer omits positional encodings, compelling the model to rely exclusively on semantic content for next-token prediction. 
\end{remark}


\section{Data}  \label{sec:data}
In this section, we define a synthetic data generation process designed to capture the essential characteristics of sentences, including the representation of factual knowledge. 
Our synthetic data aims to mimic structures commonly observed in real-world text. 
A pre-training example is:
\begin{align}\nonumber
	&\underbrace{\text{ \colorbox{peach}{XXX (Irrelevant information) Here is a record}}}_\text{Context}
	\underbrace{ \colorbox{yellow}{.}}_\text{Ending}\underbrace{\text{\colorbox{lime}{Alice}}}_\text{Subject}\text{ } \underbrace{\text{\colorbox{lightcyan}{was born in}}}_\text{Relation phrase}\text{ } \underbrace{\text{\colorbox{lightpurple}{California}}}_\text{Answer},
\end{align}
and a corresponding \ac{ft} example is:
\begin{align}\nonumber
	&\underbrace{\colorbox{verylightgray}{\textbf{Question}:} }_\text{Q\&A format 1}\underbrace{\text{\colorbox{lime}{Alice}}}_\text{Subject}\underbrace{\colorbox{verylightgray}{\text{'s hometown is} \underline{\hspace{0.7cm}}\text{?}\,\textbf{Answer}: }}_\text{Q\&A format 2}\underbrace{\text{\colorbox{lightpurple}{California}}}_\text{Answer}	
\end{align}

To model these structures, we introduce the following data distributions.
\paragraph{Pre-Training Sentence Generation (Sentence Set $\mathcal T$).}
We begin with a subject-answer set $\mathcal{A} = \{(\mathbf{s}_j, \mathbf{a}_j)\}_{j=1}^N$ comprising $N$ unique pairs, where $\mathbf{s}_j \in \mathbb{R}^d$ is the subject embedding and $\mathbf{a}_j \in \mathbb{R}^d$ is the answer embedding. 
This set is partitioned into $N_f$ "frequent" pairs (\textsf{freq}) and $N_r$ "rare" pairs (\textsf{rare}). 
Frequent pairs are associated with $K = \Theta(1)$ semantically equivalent relation phrases,\footnote{For instance, "was born in" is semantically equivalent to "comes from," "is a native of," and "hails from." 
For simplicity, we consider one relation with multiple semantically equivalent phrases; our analysis extends readily to multiple relations.} while rare pairs are linked to a single unique relation phrase. 
This setup emulates data augmentation techniques such as "\textsl{Multi}-$K$" in \cite{allen2023physics1}, where facts are presented via multiple templates (i.e., semantically equivalent relation phrases).

The sentence generation process for $\mathcal{T}$ proceeds as follows:

\begin{mdframed}
    Given context token embedding $\mathbf{o} \in \mathbb{R}^d$ and ending token embedding $\mathbf{d} \in \mathbb{R}^d$, let $\mathcal{R}$ denote the universal set of semantically equivalent relation phrase embeddings. 
    For each subject-answer pair $(\mathbf{s}_j, \mathbf{a}_j) \in \mathcal{A}$:
\begin{enumerate}
	\item 
    Determine the number of relation phrases $\tilde{K}(j) = K\cdot\mathbb{I}(j\in\textsf{freq}) + \mathbb{I}(j\in\textsf{rare})$.
	\item Sample a subset $\tilde{\mathcal{R}}\subseteq\mathcal{R}$ of $\tilde{K}(j)$ relation phrase embeddings uniformly at random without replacement.
    \item For each $i\in\tilde{\mathcal{R}}$, construct a corresponding tuple of the form $($\colorbox{peach}{$\mathbf{o}$}$,$\colorbox{yellow}{$\mathbf{d}$}$,$\colorbox{lime}{$\mathbf{s}_j$}$,$\colorbox{lightcyan}{$\mathbf{r}_i$}$,$\colorbox{lightpurple}{$\mathbf{a}_j$}$)$.
    

\end{enumerate}

\end{mdframed}


As in prior works \citep{allen2020towards, tian2023scan}, we assume all token embeddings---$\mathbf{o}$, $\mathbf{d}$, $\mathbf{s}_j$, $\mathbf{r}_i$, $\mathbf{a}_j$ for all $j \in [N]$ and $i \in \mathcal{R}$---are orthogonal.\footnote{We adopt orthogonal embeddings for simplicity; analysis of random embeddings is left for future work.} 
We further assume $\|\mathbf{s}_j\|_2 = \|\mathbf{r}_i\|_2 = \|\mathbf{a}_j\|_2 = \|\mathbf{d}\|_2 = 1$, while $\|\mathbf{o}\|_2 = \Theta(\sqrt{d})$ to embed subject-relation-answer tokens within a "rich" contextual environment. 
In this paper, we consider the simplest case with a single context; our analysis extends to $\Theta(1)$ contexts assigned uniformly at random.


\paragraph{Next-Token Prediction Dataset Generation (Dataset $\mathcal{P}$).}
The transformer is pre-trained using an \ac{ntp} objective.
We construct the \ac{ntp} dataset $\mathcal{P}$ by creating input-target pairs from the sentence set $\mathcal{T}$:

\begin{mdframed}
    From each sentence $(\mathbf{o},\mathbf{d},\mathbf{s},\mathbf{r},\mathbf{a}) \in \mathcal{T}$, generate
$(\mathbf{o},\mathcal{I}(\mathbf{d})),([\mathbf{o}\:\mathbf{s}], \mathcal{I}(\mathbf{r})),([\mathbf{o} \:\:\mathbf{s}\:\:\mathbf{r}], \mathcal{I}(\mathbf{a}))$,
where $\mathcal{I}(\cdot)$ maps the embeddings to their indices.
\end{mdframed}

Two simplifications are applied here.
First, we omit the prediction of the subject $\mathbf{s}$ from $\mathbf{od}$, i.e., $([\mathbf{o}\:\:\mathbf{d}], \mathcal{I}(\mathbf{s}))$, as it resembles a random guess.
Second, since $\mathbf{d}$ invariably follows $\mathbf{o}$ and its magnitude is dominated by that of $\mathbf{o}$, we ignore the ending token $\mathbf{d}$ in $([\mathbf{o}\:\:\mathbf{d}\:\:\mathbf{s}], \mathcal{I}(\mathbf{r}))$ and $([\mathbf{o}\:\:\mathbf{d}\:\:\mathbf{s}\:\:\mathbf{r}], \mathcal{I}(\mathbf{a}))$.
Therefore, the total number of \ac{ntp} data samples is $n = 3\times N_f\times K + 3\times N_r$.

\begin{remark}
    
    Despite these simplifications, $\mathcal{P}$ preserves the core attributes of autoregressive NTP. 
The training requires the transformer to process the subject ($[\mathbf{o}, \mathbf{s}]$) before predicting the subsequent relation ($\mathbf{r}$). 
This sequential dependency instructs the model to interpret the relation as prerequisite to determining the answer, thereby discouraging direct subject-answer memorization.

Consequently, the model relies heavily on the relation token. 
Absent this token during inference, the model falters in producing accurate answers due to the missing relational cue. 
Fine-tuning is thus essential to adapt the model for answer prediction using only the subject, enabling robust performance without explicit relations.
\end{remark}

\paragraph{Q\&A Dataset Generation (Dataset $\mathcal{Q}$).}
For FT and evaluation in a question-answering paradigm, we generate dataset $\mathcal{Q}$, where each sample entails predicting answer $\mathbf{a}$ given subject embedding $\mathbf{s}$ and format embedding $\mathbf{p}$.\footnote{We model questions as $[\mathbf{s}, \mathbf{p}]$ for simplicity; extensions to pre-subject formats or complex templates are feasible.} 
The generation process is:

\begin{mdframed}
    Given a format token embedding $\mathbf{p}\in\mathbb{R}^d$, for each subject-answer pair $(\mathbf{s},\mathbf{a})\in\mathcal{A}$, output $([$\colorbox{lime}{$\mathbf{s}$}$\:\:$\colorbox{verylightgray}{$\mathbf{p}$}$],\mathcal{I}($\colorbox{lightpurple}{$\mathbf{a}$}$))$.

\end{mdframed}

Without loss of generality, we assume that the embedding of format token $\mathbf{p}$ is orthogonal to all other token embeddings $\mathbf{o}, \mathbf{d}, \mathbf{s}_j, \mathbf{r}_i, \mathbf{a}_j$ for all $j\in[N],i\in\mathcal{R}$ and has unit magnitude, i.e., $\left\|\mathbf{p}\right\|_2 = 1$.

\paragraph{Fine-Tuning Q\&A Dataset Generation (Dataset $\mathcal{Q}_s$).}
To examine FT's impact across varying PT multiplicities, we form $\mathcal{Q}_s$ as a random subset comprising proportion $\beta$ of Q\&A pairs from $\mathcal{Q}$ corresponding to \textsf{freq} pairs:
\begin{mdframed}
    Let $\mathcal{Q}_\textsf{freq} = \{ ([\mathbf{s}_j\:\:\mathbf{p}],\mathcal{I}(\mathbf{a}_j)) \in \mathcal{Q}\ |\ j\in \textsf{freq}\}$.
    Then, $\mathcal{Q}_s$ is a randomly selected subset of proportion $\beta$ from $\mathcal{Q}_\textsf{freq}$.
\end{mdframed}
Our PT and FT datasets naturally extend conventional subject ($s$)-relation ($r$)- answer ($a$) knowledge modeling \citep{ghosal2024understanding, meng2022locating, haviv2022understanding, meng2023mass} in two key ways:
\begin{itemize}
    \item \textbf{Practical Relevance:} They better represent knowledge embedded in longer contexts. 
    Given NTP's autoregressive nature, knowledge triples $(s, r, a)$ are typically prefixed with irrelevant information.
    \item \textbf{Technical Challenge:} The inclusion of noise ($\mathbf{o}$, $\mathbf{d}$) precludes learning facts via MLPs with uniform attention alone. 
    Transformers must develop targeted attention patterns to filter noise and attain near-optimal loss.
\end{itemize}


\section{Problem Setting}   \label{sec: settings}
In this section, we formalize the optimization problem, encompassing the loss function, training algorithms, and parameter initialization.
\paragraph{Loss Function.}
We employ the cross-entropy loss to train the model parameters $\mathbf{W}$ and $\mathbf{Z}$ on the next-token prediction dataset with input-target pairs $\mathcal{P}=\{(\mathbf{X}_i,y_i)\}_{i=1}^n$:
\begin{align}
	\mathcal{L}_\mathcal{P}(\mathbf{W},\mathbf{Z}) = \frac{1}{n}\sum_{(\mathbf{X},y)\in\mathcal{P}}\mathcal{L}(\mathbf{W},\mathbf{Z},\mathbf{X},y),
\end{align}
where the per-sample loss is $\mathcal{L}(\mathbf{W},\mathbf{Z},\mathbf{X},y)= -\log\left(\text{logit}_y(\mathbf{W},\mathbf{Z},\mathbf{X})\right)$, and $\mathbf{logit}(\mathbf{W}, \mathbf{Z}, \mathbf{X}) =\softmax(\mathbf{x}^\text{output}(\mathbf{W}, \mathbf{Z}, \mathbf{X}))$.

\paragraph{Pre-Training Algorithm.}
We train the model with \ac{gd}:
\begin{align}\nonumber
	&\mathbf{Z}^{(t+1)} = \mathbf{Z}^{(t)} - \frac{\eta^{(t)}}{n}\sum_{(\mathbf{X},y)\in\mathcal{P}} \nabla_{\mathbf{Z}^{(t)}} \mathcal{L}(\mathbf{W}^{(t)},\mathbf{Z}^{(t)},\mathbf{X},y),\\\nonumber
	&\mathbf{W}^{(t+1)} = \mathbf{W}^{(t)} - \frac{\eta^{(t)}}{n}\sum_{(\mathbf{X},y)\in\mathcal{P}} \nabla_{\mathbf{W}^{(t)}} \mathcal{L}(\mathbf{W}^{(t)},\mathbf{Z}^{(t)},\mathbf{X},y),
\end{align}
where $\eta^{(t)}$ is the learning rate at iteration $t$.
For the pre-training phase, we adopt a three-stage learning-rate schedule with two pre-defined thresholds $h_1$ and $h_2$: $\eta^{(t)} = \eta_1$ for $t \le h_1$, $\eta_2$ for $h_1 < t \le h_2$, and $\eta_3$ for $t > h_2$.
As shown in Section \ref{sec:dynamic}, this schedule facilitates convergence to near-optimal training loss.

\paragraph{Fine-Tuning Algorithm.}
For full \ac{ft}, we adopt the same \ac{gd} procedure as in \ac{pt} on dataset $\mathcal{Q}_s$, with a fixed learning rate $\eta_f$.
For low-rank fine-tuning, we instead update using the best rank-1 approximation of the gradient: for a matrix $\mathbf{A}$, this is obtained by minimizing $\|\mathbf{A} - \mathbf{a} \mathbf{b}^\top\|_F$ over $\mathbf{a}, \mathbf{b}$. 
Further details are provided in Appendix \ref{app: low-rank FT}. 
This method aligns with practical low-rank adaptation techniques, such as \cite{zhang2023adalora}, which prioritize components associated with dominant singular values.

\paragraph{Initialization.}
All entries of $\mathbf{Z}^{(0)}, \mathbf{W}^{(0)}$ are initialized independently from a Gaussian distribution: for all $i,j\in[d]$ and $k\in[md]$ , $Z_{i,j}^{(0)}, W_{k,j}^{(0)} \sim \mathcal{N}(0, \sigma_0^2)$.

Building on this formulation, we present our main theoretical results in the subsequent section, providing formal guarantees on convergence and generalization.


\section{Conditions for Post Fine-Tuning Generalization}    \label{sec: results}
In this section, we address the key question: Under what conditions do transformers achieve generalization after fine-tuning? 
We provide an answer by deriving convergence bounds for the \ac{pt} phase over $T_p$ iterations and generalization bounds after \ac{ft} over $T_f$ iterations. 

Our theoretical analyses rely on the following conditions:
\begin{condition}\label{condition: condition}
	Assume there exist sufficiently large positive constants $C_1', C_2', C_3', C_4', C_5', C_6',C_7'$. 
    For a probability parameter $\delta \in (0,1)$, the following hold:
	\begin{enumerate}
		\item The learning rates satisfy:
        \begin{equation}
            \begin{cases}
            C_1'\frac{\sqrt{mn}}{\lambda d^{1/4}} \le \eta_1 \le C_2'\frac{m\log(d)}{\lambda^2}, \\
                \eta_2 \le nm\log(d)/(C_3'\lambda^2),\\
                \eta_3\le  nm\log(d)/(C_4'\lambda^2).
            \end{cases}
        \end{equation}
        
		\item The embedding dimension satisfies:
        \begin{equation}
            d\ge C_5'\cdot \max\left\{\frac{mN(N+|\mathcal{R}|)^2}{\lambda^2}, \frac{\lambda^2}{(n^2\delta^2)}\right\}.
        \end{equation}
		\item The initialization magnitude satisfies: $\sigma_0 = \Theta(1/\sqrt{d})$.
		\item The width (number of neurons) of the \ac{mlp} satisfies: $m \ge C_6'\cdot\log(\frac{(NK)}{\delta})$.
        \item The size of training data satisfies: $N\ge C_7'\log(m/\delta)$.
	\end{enumerate}
\end{condition}

Condition 1 ensures learning rates are appropriately bounded to minimize training loss. 
Condition 2 mandates a sufficiently large embedding dimension for stable per-sample learning during PT, akin to assumptions in prior work \citep{kou2023benign, frei2022benign, chatterji2021finite}. 
Condition 3 specifies an initialization scale common in practical LLMs, such as GPT-2 \citep{radford2019language} and Llama-2 \citep{touvron2023llama}. 
Condition 4 imposes a mild logarithmic lower bound on MLP width, ensuring a non-negligible number of neurons ($\Omega(m)$ neurons) activate per sample.
Condition 5 is a mild logarithmic lower bound on training data size, ensuring all the $m$ neurons are covered.

We define the \emph{knowledge extraction loss} (or generalization loss) as the misclassification rate on unseen Q\&A pairs during \ac{ft} ($(1 - \beta) N_f + N_r$ pairs):
\begin{align}   \label{eq: success}
	\mathcal{L}_e (\mathbf{W},\mathbf{Z}) \!=\! \frac{\sum_{(\mathbf{X},y)\in \mathcal{Q}\backslash \mathcal{Q}_s}}{|\mathcal{Q}\backslash \mathcal{Q}_s|}\mathbb{P}\left[(\mathbf{x}^\text{output}(\mathbf{W},\mathbf{Z}, \mathbf{X}))_{y} \le \max_{i\neq y}\{(\mathbf{x}^\text{output}(\mathbf{W},\mathbf{Z}, \mathbf{X}))_{i}\}\right].
\end{align}
We now analyze the PT training loss and post-FT knowledge extraction loss for full and low-rank FT.

\begin{theorem}\label{theorem: main}
	For any constant $\delta>0$, under Condition \ref{condition: condition}, 
    select the third-stage pre-training learning rate $\eta_3 \le \kappa / 64$ (with $\kappa \le 1/2$ as a constant) and fine-tuning learning rate $\eta_f = \tilde{\Theta}(m |\mathcal{R}| / \lambda^2)$. 
    Then, with probability at least $1 - \delta$:
    \begin{enumerate}
        \item During PT over $T_p = \eta_2^{-1} \poly(n, m, \lambda^{-1}, K^{-1}) + \eta_3^{-1} \poly(m, N, |\mathcal{R}|, \lambda^{-1})$ iterations, there exists $0 \le t_p \le T_p$ where the model acquires knowledge:
        \begin{align}
        \mathcal{L}_\mathcal{P}(\mathbf{W}^{(t_p)}, \mathbf{Z}^{(t_p)}) < 0.001 + (1 + \kappa) H.
        \end{align}
        \item If $\beta N_f \ge C_1 |\mathcal{R}|$ (full FT) or $\beta N_f \ge C_2 m |\mathcal{R}|^2$ (low-rank FT), then after $T_f = \mathcal{O}(\beta N_f d^{-0.01} / |\mathcal{R}|)$ iterations, the model extracts knowledge:
        \begin{align}\label{equ: extract full}
        \mathcal{L}_e(\mathbf{W}^{(T_p + T_f)}, \mathbf{Z}^{(T_p + T_f)}) \le |\mathcal{R}|\exp\left( -\frac{\beta N_f K^2}{2 |\mathcal{R}|^2} \right).
        \end{align}
        \item If $\beta N_f K / |\mathcal{R}| \le C_3$, the full or low-rank fine-tuned model may hallucinate:
        $$ \mathcal{L}_e(\mathbf{W}^{(T_p + T_f)}, \mathbf{Z}^{(T_p + T_f)}) \ge c_1. $$
    \end{enumerate}
    Here, $C_1, C_2, C_3, c_1 > 0$ are positive constants, and $H$ is the entropy of the NTP training dataset (optimal cross-entropy loss).
\end{theorem}

Theorem \ref{theorem: main} demonstrates that FT on frequent facts enhances generalization, even to rare facts, aligning with empirical findings \citep{ghosal2024understanding, allen2023physics1}. 
The bound in \eqref{equ: extract full} reveals that higher multiplicity $K$ (diverse fact presentations in PT) reduces generalization error, particularly for large $|\mathcal{R}|$, explaining the success of knowledge augmentation \citep{allen2023physics1}.\footnote{Large $|\mathcal{R}|$ is common in practice, reflecting linguistic flexibility, e.g., equivalents to "was born in" like "comes from," "is a native of," or "hails from."}

\begin{remark}
    A novel theory-derived insight is that post-FT accuracy declines with increasing $|\mathcal{R}|$, a phenomenon obscured in empirical studies by natural language complexity.
\end{remark}

\paragraph{Relation covering.}
The key intuition of post-\ac{ft} knowledge extraction is a relation-covering characterization. 
During \ac{pt}, a fact is encoded through the relation templates in which it appears. 
During \ac{ft}, the Q\&A format must connect to these pre-trained relation directions. 
Thus, \ac{ft} need not relearn every factual association, but it must cover enough relation-template directions to make the new format usable for other stored facts. 
Increasing the multiplicity $K$ or the number of fine-tuning examples $\beta N_f$ improves coverage, while increasing the relation-template universe $|\mathcal R|$ makes coverage harder. 
This explains the benefit of knowledge augmentation \citep{allen2023physics1}.

A coupon-collector calculation gives the intuition behind the exponential term. 
If each fine-tuning example activates roughly $K$ relation directions, then after fine-tuning with $\beta N_f$ examples, 
\[
\mathbb P\left[
\exists \mathbf{r}_i\in\mathcal R:\  \text{relation }i \text{ is not sufficiently covered}
\right]
\le
|\mathcal{R}|\exp\left(
-\frac{\beta N_f K^2}{2|\mathcal{R}|^2}
\right).
\]
When coverage is sufficient, the Q\&A format becomes aligned with the pre-trained relation features, allowing extraction of facts acquired during \ac{pt} but not revisited during \ac{ft}. 
When coverage is insufficient, the facts may remain stored but inaccessible under the downstream format, yielding the hallucination regime in Theorem~\ref{theorem: main}.

\section{Mechanics of Knowledge Acquisition and Extraction}
\label{sec:dynamic}
This section provides a detailed characterization of the training dynamics during \ac{pt} with \ac{ntp} and subsequent knowledge extraction via \ac{ft}, serving as a proof sketch for Theorem \ref{theorem: main}.
\paragraph{Pre-Training}
As outlined in Section \ref{sec: settings}, our analysis of pre-training dynamics employs a three-stage learning rate schedule:
\begin{itemize}
    \item In Stage \uppercase\expandafter{\romannumeral1} ($0\le t < T_1=2$), the learning rate is $\eta_1$; 
    \item In Stage \uppercase\expandafter{\romannumeral2} ($T_1 \le t \le T_2 = \Theta(nm\log(d)/(\lambda^2\eta_2 K))$), the learning rate is $\eta_2$; 
    \item In Stage \uppercase\expandafter{\romannumeral3} ($T_2 < t \le T_p$), the learning rate is $\eta_3$.
\end{itemize}
We characterize the training dynamics of self-attention and \ac{mlp} separately.

\begin{proposition} [Training dynamics of self-attention]\label{prop: attention}
    Under Condition \ref{condition: condition}, with probability at least $1 - \delta$, after Stage I ($T_1 \le t \le T_p$):
    \begin{enumerate}
        \item The transformer learns to filter out the irrelevant context token $\mathbf{o}$: For all $(\mathbf{o},\mathbf{d},\mathbf{s},\mathbf{r},\mathbf{a})\in\mathcal{T}$, the attention scores\footnote{Attention scores are entries of  $\boldsymbol{\alpha}(\mathbf{Z},\mathbf{X})$.The subscript $k$ denotes the attention score on the $k^{th}$ token of the input sequence $\mathbf{X}$.} satisfy:
    	\begin{align}
    		\alpha_{1}(\mathbf{Z}^{(t)},[\mathbf{o}\:\:\mathbf{s}])\le \frac{1}{d} \quad \text{and}\quad \alpha_{1}(\mathbf{Z}^{(t)},[\mathbf{o}\:\:\mathbf{s}\:\:\mathbf{r}]) \le \frac{1}{d}.
    	\end{align}

        \item The transformer's attention scores on $\mathbf{s}_j$ and $\mathbf{r}_i$ are comparable:
		For all $(\mathbf{o},\mathbf{d},\mathbf{s},\mathbf{r},\mathbf{a})\in\mathcal{T}$:
        \begin{equation}
            0.5\le \frac{\alpha_{2}(\mathbf{Z}^{(t)}, [\mathbf{o}\:\:\mathbf{s}\:\:\mathbf{r}])}{\alpha_{3}(\mathbf{Z}^{(t)},[\mathbf{o}\:\:\mathbf{s}\:\:\mathbf{r}])} \le 2.
        \end{equation}
    \end{enumerate}
\end{proposition}

\begin{proposition}[Training dynamics of MLP] \label{prop: mlp}
    During \ac{pt}, under Condition \ref{condition: condition}, with probability $1-\delta$, the followings hold:
    \begin{enumerate}
        \item After Stage \uppercase\expandafter{\romannumeral1} ($T_1 \le t \le T_p$), the transformer learns to predict $\mathbf{o}\rightarrow \mathbf{d}$:
        \begin{equation}
            \mathbf{x}^\textnormal{output}_{\mathcal{I}(\mathbf{d})}(\mathbf{W}^{(t)},\mathbf{Z}^{(t)},\mathbf{o}) = \tilde{\Omega}(1).
        \end{equation}
        
        \item During Stage \uppercase\expandafter{\romannumeral1} and \uppercase\expandafter{\romannumeral2} ($0 \le t \le T_2$), the transformer rapidly learns to predict $\mathbf{o}\mathbf{s}_j\rightarrow \mathbf{r}_i$ and $\mathbf{o}\mathbf{s}_j\mathbf{r}_i\rightarrow \mathbf{a}_j$:
		For all $(\mathbf{o},\mathbf{d},\mathbf{s}_j,\mathbf{r}_i,\mathbf{a}_j)\in\mathcal{T}$, 
        \begin{equation}
            1\!-\!\textnormal{logit}_{\mathcal{I}(\mathbf{a}_j)}(\mathbf{W}^{(t)},\mathbf{Z}^{(t)},[\mathbf{o}\:\:\mathbf{s}_j\:\:\mathbf{r}_i]) \!=\! \Theta(1) \text{ and }
            1\!-\! \tilde{K}(j)\textnormal{logit}_{\mathcal{I}(\mathbf{r}_i)}(\mathbf{W}^{(t)},\mathbf{Z}^{(t)},[\mathbf{o}\:\:\mathbf{s}_j]) \!=\! \Theta(1).
        \end{equation}
        
        \item At the end of Stage \uppercase\expandafter{\romannumeral3}, the pre-training loss converges to nearly optimal:
        \begin{align}
            \frac{1}{T_p-T_2}\sum_{t=T_2+1}^{T_p}\mathcal{L}_\mathcal{P}(\mathbf{W}^{(t)},\mathbf{Z}^{(t)}) \le \frac{A}{(2-\kappa)\eta_3 (T_p-T_2)} + 0.0005 + (1+\kappa)H,
        \end{align}
        where $A = \tilde{\Theta}(\frac{(m+N^2/\lambda^2+|\mathcal{R}|^2/\lambda^2)}{\eta_3})$.
    \end{enumerate}
\end{proposition}

Propositions \ref{prop: attention} and \ref{prop: mlp} elucidate the transformer's functional division: self-attention filters irrelevant context (visualized in Appendix \ref{app: attention visual}), while the MLP memorizes the filtered features. 
The stages unfold as follows: 
Stage \uppercase\expandafter{\romannumeral1} rapidly establishes attention patterns to filter out noise with a relatively large learning rate $\eta_1$;
Stage \uppercase\expandafter{\romannumeral2} develops subject and relation features with a moderate learning rate $\eta_2$;
Stage \uppercase\expandafter{\romannumeral3} refines the loss to near-optimality with a small learning rate $\eta_3$ (and thus small $\kappa$). 
This progression underscores the importance of learning rate scheduling in optimization.

\paragraph{Fine-Tuning.}
During the \ac{ft} phase, the transformer integrates the Q\&A format embedding $\mathbf{p}$ into its pre-learned features.
We formalize this process below.
\begin{proposition}   \label{prop: ft gradient}
	For all $i\in\mathcal{R}$ and $j\in[N]$, when $\beta N_f > C_1|\mathcal{R}|$, the gradient for $\mathbf{W}$ in the first iteration of \ac{ft} satisfies:
    \begin{align*} \label{lemma: ft grad}
        & \circled{1}\ \left\langle \Bar{\boldsymbol{\nabla}}_{\mathcal{I}(\mathbf{r}_i)},\mathbf{p}\right\rangle = \Theta\left(\frac{\lambda}{m|\mathcal{R}|}\right), 
        & \circled{2}\ \left\langle \Bar{\boldsymbol{\nabla}}_{\mathcal{I}(\mathbf{a}_j)},\mathbf{p}\right\rangle = - \Theta\left(\frac{\lambda}{m\beta N_f}\right),\\
        & \circled{3}\ \left\langle \Bar{\boldsymbol{\nabla}}_{\mathcal{I}(\mathbf{r}_i)},\mathbf{s}_j\right\rangle = \Theta\left(\frac{\lambda}{m\beta N_f}\right),
        & \circled{4}\ \left\langle \Bar{\boldsymbol{\nabla}}_{\mathcal{I}(\mathbf{a}_j)}, \mathbf{s}_j\right\rangle =-\Theta\left(\frac{\lambda}{m\beta N_f}\right),
	\end{align*}
    where $\Bar{\boldsymbol{\nabla}}_{\mathcal{I}(\cdot)} = \frac{1}{m}\sum_{l=1}^m \nabla_{\mathbf{w}^{(T_p)}_{\mathcal{I}(\cdot),l}} \mathcal{L}_{\mathcal{Q}_s}(\mathbf{W}^{(T_p)},\mathbf{Z}^{(T_p)})$.
\end{proposition}

Proposition \ref{prop: ft gradient} reveals that when the \ac{ft} dataset size $\beta N_f$ is much larger than $|\mathcal{R}|$, the gradient exhibits low-rank structure, with $\mathbf{p}$ components (\circled{1}, \circled{2}) dominating others (\circled{3}, \circled{4}), explaining low-rank FT efficacy. 
As gradients prioritize $\mathbf{p}$, relation neurons swiftly incorporate the format (\circled{1} prevails), enabling generalization on unseen Q\&A pairs.

\begin{remark}
    FT may induce overfitting to the format token $\mathbf{p}$: 
    When the number of \ac{ft} iterations satisfies $T_f = \tilde{\Omega}\left(\frac{\lambda\beta N_f}{|\mathcal{R}|}\right)$, feature increments $\left\langle\Bar{\boldsymbol{\nabla}}_{\mathcal{I}(\mathbf{a}_j)},\mathbf{p}\right\rangle$ accumulate, potentially yielding $\sum_{l=1}^m\left\langle \mathbf{w}_{\mathcal{I}(\mathbf{a}_i),l}^{(T_p+T_f)},\mathbf{p}\right\rangle/m = \Omega\left(\sum_{l=1}^m\left\langle \mathbf{w}_{\mathcal{I}(\mathbf{a}_j),l}^{(T_p+T_f)},\mathbf{s}_j\right\rangle/m\right)$ for some unseen $i \in [N]$. 
    Appendix \ref{app: ft steps} visualizes accuracy versus FT steps.
\end{remark}

\section{Experiments}   \label{sec:experiments}
In this section, we empirically validate our theoretical results through experiments on both synthetic and real-world datasets.
\begin{figure}[h] 
	\centering
\begin{subfigure}{0.49\linewidth}
	\centering
	\includegraphics[width=0.8\linewidth]{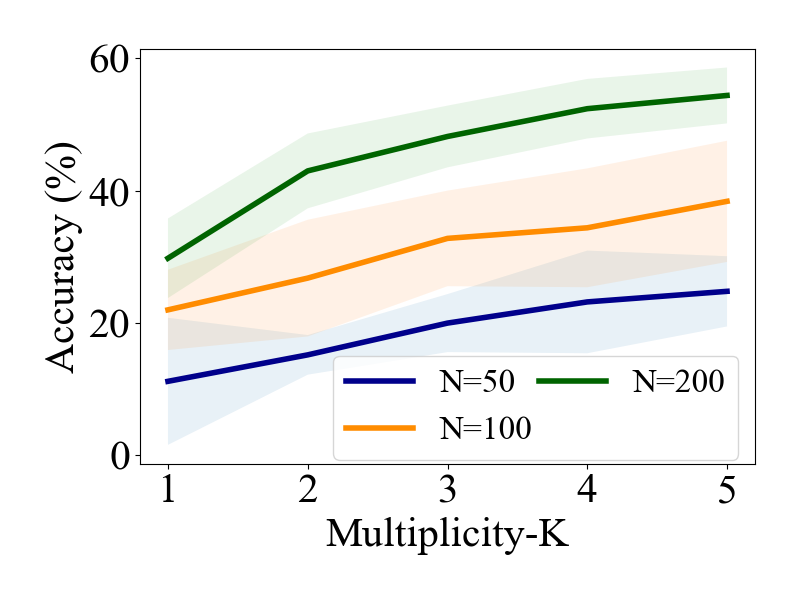}
	\caption{Impact of the number of subject-answer pairs $N$.}
	\label{fig:nks}
\end{subfigure}
	 \hfill
    \begin{subfigure}{0.49\linewidth}
	\centering
	\includegraphics[width=0.8\linewidth]{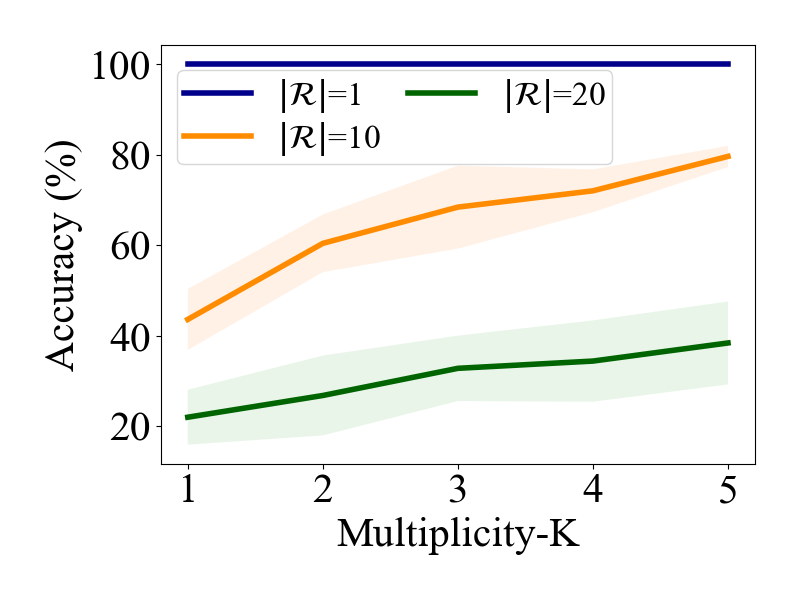}
	\caption{Impact of the size of relation phrases $|\mathcal{R}|$.}
	\label{fig:rks}
\end{subfigure}
\caption{Generalization accuracy of our simplified
one-layer transformers with 5-token data. In panel (b), the \(|\mathcal R|=1\) horizontal line is a reference baseline evaluated at \(K=1\), since sampling \(K\) distinct templates without replacement requires \(K\le |\mathcal R|\).}

\end{figure}
\begin{figure}
	\centering
	\begin{subfigure}{0.49\linewidth}
		\centering
		\includegraphics[width=0.8\linewidth]{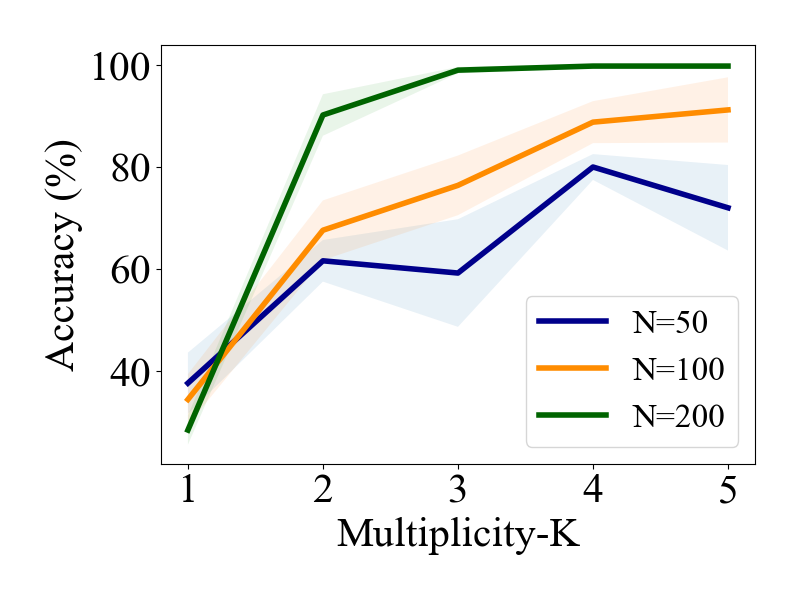} 
		\caption{Impact of the number of subject-answer pairs $N$.}
		\label{fig:nkg}
	\end{subfigure}
    \hfill
	\begin{subfigure}{0.49\linewidth}
		\centering
		\includegraphics[width=0.8\linewidth]{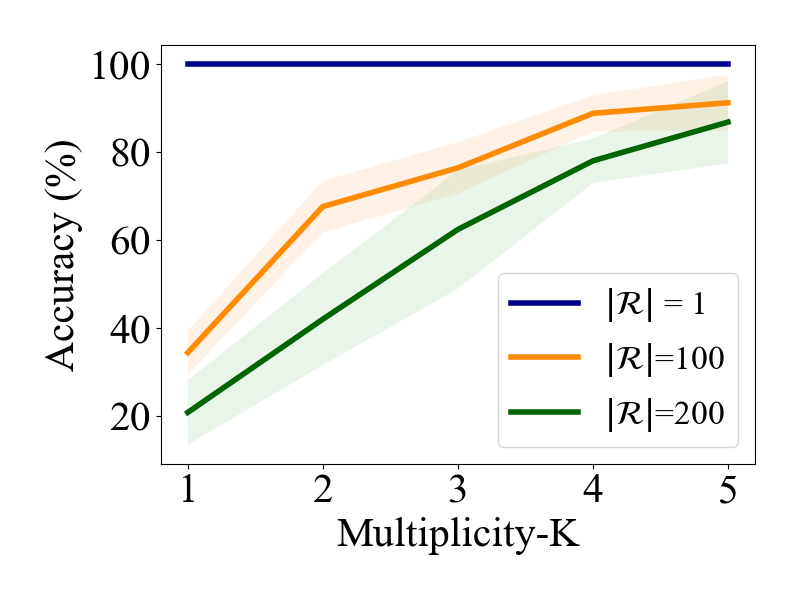}
		\caption{Impact of the size of relation phrases $|\mathcal{R}|$.}
		\label{fig:rkg}
	\end{subfigure}
    \caption{Generalization accuracy of GPT-2 with 5-token data.}
\end{figure}

We generate a synthetic dataset adhering strictly to the process outlined in Section \ref{sec:data}. 
Experiments are conducted using two architectures: 
(1) our simplified one-layer transformer (Section \ref{sec: model}) 
and (2) a standard 12-layer, 12-head GPT-2 model \citep{radford2019language}. 
Additional experimental details are provided in Appendix \ref{app: experiemntal settings}. 
\textit{Notably, both models display similar performance trends, indicating that our simplified framework captures key dynamics of modern large language models (LLMs).}

As illustrated in Figures \ref{fig:nks} and \ref{fig:nkg}, generalization accuracy improves with increasing pre-training multiplicity $K$ and fine-tuning dataset size $\beta N_f$. 
Furthermore, Figures \ref{fig:rks} and \ref{fig:rkg} show that accuracy declines as the size of the universal relation set $|\mathcal{R}|$ grows. 
\textit{These trends align precisely with our theoretical predictions (Theorem \ref{theorem: main}).}

\subsection{Real-World Dataset}
To further assess the applicability of our theory to practical scenarios, we conducted experiments on modern LLMs (GPT-2 and Llama-3.2-1B \citep{grattafiori2024llama}) with PopQA dataset \citep{mallen2022not}.

\paragraph{Pre-Train and Fine-Tune GPT-2.}
We pre-train GPT-2 from scratch and subsequently fine-tune it on a curated subset of PopQA focused on movies and their directors. 
Pre-training data incorporates multiple phrasings per fact (e.g., "[movie] was directed by [director]" or "[movie] hails from [director]"), simulating multiplicity $K$. 
Fine-tuning uses a question-answering format (e.g., "Question: Who is the director of [movie]? Answer:"). 
Data examples and training details are in Appendices \ref{app: director data} and \ref{app: qa details}, respectively.
Results in Figure \ref{fig:qa} show that exact match accuracy on unseen facts rises with both fine-tuning dataset size and multiplicity $K$, \textit{corroborating Theorem \ref{theorem: main}}.

\begin{figure}[H]
    \centering 
    \begin{minipage}{0.48\textwidth}
        \centering
	\includegraphics[width=0.8\linewidth]{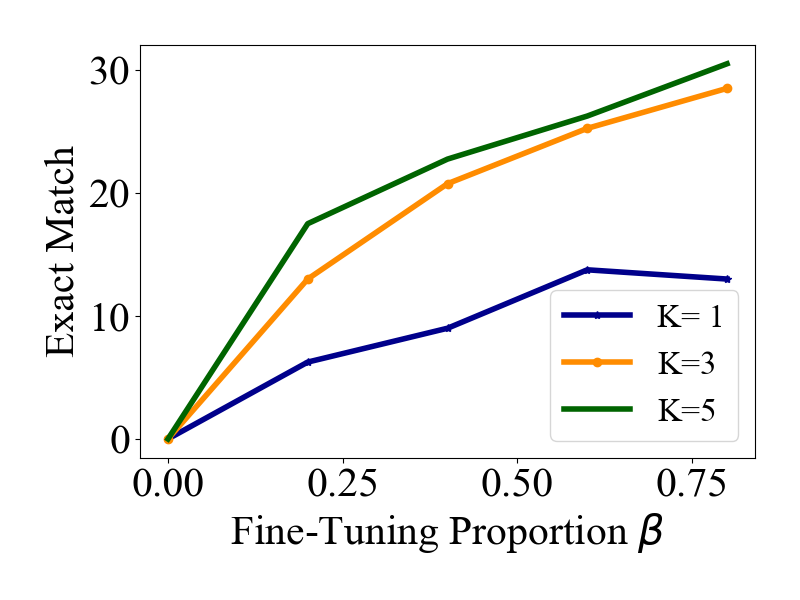}
        \captionof{figure}{Exact match accuracy of pre-trained and fine-tuned GPT-2 on PopQA.}
	\label{fig:qa}
    \end{minipage}
    \hfill 
    \begin{minipage}{0.48\textwidth}
        \centering
        \captionof{table}{Exact match accuracy and F1 scores of Llama-3.2-1B fine-tuned with frequent and rare facts.}
        \label{tab:mytable}
	\begin{tabular}{ccc}
        \toprule
		\textbf{FT Data} & \textbf{Exact Match} & \textbf{F1 score}\\
		\midrule
		Frequent & 26.36 & 28.68\\
		Rare & 20.93 & 22.74\\
        \bottomrule
	\end{tabular}
    	\label{table: performance}
    \end{minipage}
\end{figure}

\paragraph{Fine-Tune Llama-3.2-1B.} 
We further test our insights by fine-tuning the pre-trained Llama-3.2-1B model on a PopQA subset concerning capitals, formatted as questions (e.g., "Question: Where is the capital of [location]? Answer:"). 
Details are in Appendices \ref{app: capital data} and \ref{app: table details}.
Table \ref{table: performance} reveals superior exact match accuracy and F1 scores when fine-tuning on high-frequency data (higher Wikipedia page views) compared to low-frequency data, \textit{affirming our theoretical conditions for effective knowledge extraction.}


\section{Conclusions}\label{sec: conclusions}
In this paper, we develop a theoretical framework to elucidate the mechanisms by which transformers acquire and extract knowledge. 
Our analysis demonstrates that transformers can attain near-optimal training loss, thereby enabling effective knowledge acquisition during pre-training. 
Furthermore, we delineate precise conditions under which transformers exhibit strong generalization following full or low-rank fine-tuning, thus facilitating robust knowledge extraction.
Empirical validations on synthetic datasets, as well as the real-world PopQA dataset using large language models such as GPT-2 and Llama-3.2-1B, corroborate our theoretical insights.

Although this work establishes foundational principles for knowledge acquisition and extraction in transformers, our derivations rely on idealized assumptions, including orthogonal embeddings and a block-diagonal structure for the second MLP layer, to mitigate feature interference. 
Contemporary multi-layer architectures likely encompass more intricate dynamics beyond those examined here. 
As such, investigating the roles of advanced components, such as multiple layers, attention heads, and optimizers like AdamW, in these processes remains a compelling avenue for future research.

%% file: appendix.tex
\allowdisplaybreaks

\section{An \ac{ntp} experiment}\label{app: table}
In this appendix section, we present additional experimental results for one-layer attention-only transformers.
To replicate the realistic next-token prediction and isolate the effect of the contextual noise, we use a 3-token 'subject-relation-answer' dataset as follows.

We conduct experiments with the following data distribution.
We begin with a subject-answer set $\mathcal{A}$ containing $N$ unique subject-answer pairs.
Each pair associates with $K$ relation phrases.
This setup mirrors data augmentation techniques like "\textsl{Multi}-$K$" in \cite{allen2023physics1}, where each fact is presented through multiple templates (relation phrases).
The generation process for sentences in $\mathcal T$ is as follows:

\begin{mdframed}
Let $\mathcal{R}$ be a predefined set of relation phrase embeddings. 
    For each subject-answer pair $(\mathbf{s}_j,\mathbf{a}_j) \in \mathcal{A}$:
\begin{enumerate}
	\item Sample a subset $\tilde{\mathcal{R}}\subseteq\mathcal{R}$ of $K$ relation phrase embeddings uniformly at random without replacement.
	\item Generate a total of $K$ tuples by iterating through each index $i$ in the set $\tilde{\mathcal{R}}$ and constructing a corresponding tuple of the form 
    $( \mathbf{s}_j, \mathbf{r}_i,\mathbf{a}_j)$.
    
    \item Convert each tuple $( \mathbf{s}_j, \mathbf{r}_i,\mathbf{a}_j)$ to \ac{ntp} pre-training samples, $(\mathbf{s}_j,\mathcal{I}(\mathbf{r}_i))$ and $([\mathbf{s}_j\:\: \mathbf{r}_i],\mathcal{I}(\mathbf{a}_j))$.
    \item Randomly select half tuples from $\mathcal{T}$ and convert each of them to fine-tuning samples, $([\mathbf{s}_j\:\: \mathbf{p}],\mathcal{I}(\mathbf{a}_j))$.
\end{enumerate}
\end{mdframed}
Based on the sentence set $\mathcal{T}$, we construct the \textbf{pre-training dataset} as follows:
\begin{mdframed}
    For each $(\mathbf{s}_j, \mathbf{r}_i,\mathbf{a}_j) \in \mathcal{T}$:
\begin{enumerate}
    \item Convert each tuple $( \mathbf{s}_j, \mathbf{r}_i,\mathbf{a}_j)$ to \ac{ntp} pre-training samples, $(\mathbf{s}_j,\mathcal{I}(\mathbf{r}_i))$ and $([\mathbf{s}_j\:\: \mathbf{r}_i],\mathcal{I}(\mathbf{a}_j))$.
\end{enumerate}
\end{mdframed}
We construct the \textbf{fine-tuning dataset} and the test dataset as:
\begin{mdframed}
    For each $(\mathbf{s}_j, \mathbf{r}_i,\mathbf{a}_j) \in \mathcal{T}$:
\begin{enumerate}
    \item Randomly select half tuples from $\mathcal{T}$ and convert each of them to fine-tuning samples, $([\mathbf{s}_j\:\: \mathbf{p}],\mathcal{I}(\mathbf{a}_j))$.
    \item Convert each of the remaining tuples to test samples, $([\mathbf{s}_j\:\: \mathbf{p}],\mathcal{I}(\mathbf{a}_j))$.
\end{enumerate}
\end{mdframed}

We assume all token embeddings—$\mathbf{s}_j, \mathbf{r}_i, \mathbf{a}_j$, for all $j\in[N]$ and $i\in \mathcal{R}$ are generated from random Gaussian distributions $\mathcal{N}(0,\sigma^2\mathbf{I})$.
We set the parameters as $N = 1000, K=5, |\mathcal{R}|=5, \sigma = 1, d= 50$.
We pre-train the transformers (‘standard' and 'Uniform attention + MLP') with AdamW optimizer with learning rate $\eta \in \{1e-5, 1e-4, 1e-3\}$, weight decay parameter 0.1, and full batch. 
We fine-tune the pre-trained transformers (‘standard' and 'Uniform attention + MLP') with Adam optimizer, using a learning rate $\eta \in \{1e-5, 1e-4\}$ and a full batch. 
We pre-train the simplified transformer with the GD optimizer with a learning rate of 0.5. We fine-tune the pre-trained simplified transformer using the GD optimizer with a learning rate of $\eta_f = 8$.

Performance of the trained models (until convergence) is shown in Table \ref{table: ood1}. 
Table \ref{table: ood1} shows that attention-only transformers fail to acquire knowledge (achieve near-optimal training loss), implying that \ac{mlp} is necessary for transformers to acquire knowledge during pre-training. 
Notably, even 'uniform attention + MLP' can acquire and extract knowledge under this three-token dataset.
To explore the function of the self-attention module, we extend the above three-token dataset to a five-token 'context-subject-relation-answer' dataset, described in Section \ref{sec:data}, under which 'uniform attention + MLP' fails to acquire and extract knowledge while our simplified transformer can.
\begin{table}[h]
	\centering
    \caption{Training loss and generalization \textsl{arg max} accuracy on the above three-token 'subject-relation-answer' dataset. We put the cross-entropy loss lower bound, the data entropy, in the bracket. The simplified model details are provided in Sections \ref{sec: model}. We do not report the accuracy when the training loss is large, as the extraction accuracy is somehow meaningless when the model cannot acquire the knowledge.}
	\begin{tabular}{ccc}
        \toprule
		\textbf{Architecture} & \textbf{Training loss} & \textbf{Accuracy} \\
		\midrule
		Self-attention + MLP (Standard) &  0.8064 (0.8047) & 100\% \\
		Uniform attention + MLP &  0.8068 (0.8047) & 100\% \\
		Attention-only & 5.5875 (0.8047) & - \\
        \rowcolor{gray!20}
        Simplified self-attention + MLP (Ours) &  0.8200 (0.8047) & 99\% \\
        \bottomrule
	\end{tabular}
	\label{table: ood1}
\end{table}

This capability derives from the extensive parameter count of the MLP, which offers an advantage of order $m$ over attention-only transformers \citep{nichani2024transformers}.

\section{Additional related works}\label{app: related_work}
\begin{table}[h]
	\footnotesize
	\centering
	\caption{Comparison with existing theoretical works on transformer training dynamics. 
		Given a sequence '$a,b,c$', NTP constructs a training pair for every adjacent position, $a\rightarrow b$ and $ab\rightarrow c$, whereas "seq"-prefix prediction treats the entire prefix as context, producing a single pair $ab\rightarrow c$. 
	}
	\begin{tabular}{cccccccc}
		\toprule
		&\textbf{Task}&\textbf{Characterization} & \textbf{Obj.}& \textbf{PT}& \textbf{FT} & \textbf{Generalization}\\
		\midrule
		\rowcolor{gray!20}
		Our Paper& Knowledge & \!\!Full dynamics\!\! &NTP &$\usym{2713}$&$\usym{2713}$ &$\usym{2713}$ \\
		\midrule
		\cite{nichani2024understanding,zhu2024towards}& \multirow{2}{*}{Knowledge} & \multirow{2}{*}{\!\!Full dynamics\!\!} &\multirow{2}{*}{seq} &\multirow{2}{*}{$\usym{2713}$}&\multirow{2}{*}{-}& \multirow{2}{*}{-}\\
		\cite{cabannes2024learning,tian2023scan}&  & & & & & \\
		\midrule
		\cite{ghosal2024understanding}& Knowledge & \!\!Gradients\!\! & seq &-&$\usym{2713}$&-\\
		\midrule
		\cite{wang2025learning}& In-context &\!\!Layerwise dynamics\!\!&seq&$\usym{2713}$&-&-\\
		\midrule
		\cite{nichani2024transformers}& In-context &\!\!Layerwise dynamics\!\!&seq&$\usym{2713}$&-&$\usym{2713}$\\
		\midrule
		\cite{ahn2023transformers,mahankali2023one}& In-context & \!\!Critical points\!\! & seq &$\usym{2713}$&-& - \\
		\midrule
		\cite{bietti2023birth}& In-context & \!\!Gradients\!\! & seq &$\usym{2713}$&-&- \\
		\midrule
		\cite{zhang2024trained}& In-context & \!\!Full dynamics\!\! & seq &$\usym{2713}$&-&$\usym{2713}$ \\
		\midrule
		\cite{huang2023context,zhang2024trained}& In-context & \!\!Gradients\!\! & seq &$\usym{2713}$&-&- \\
		\midrule
		\cite{jelassi2022vision}& Img. cls. & \!\!Full dynamics\!\! & - &$\usym{2713}$&$\usym{2713}$&- \\
		\midrule
		\cite{sakamoto2024benign}& Img. cls. & \!\!Full dynamics\!\! & - &$\usym{2713}$&-&-\\
		\midrule
		\cite{wang2024transformers}& Others & \!\!Full dynamics\!\! & - &$\usym{2713}$&-&- \\
		\midrule
		\cite{tian2023joma}& Others & \!\!Full dynamics\!\! & - &$\usym{2713}$&-&- \\
		\bottomrule
	\end{tabular}
	
	\label{table: comparison}
\end{table}
\begin{table}[h]
	\footnotesize
	\centering
	\caption{Extension of Table \ref{table: comparison}.
	}
	\begin{tabular}{ccccc}
		\toprule
		&\textbf{Task}&\textbf{Characterization} & \textbf{Architectures} \\
		\midrule
		\rowcolor{gray!20}
		Our Paper& Knowledge & \!\!Full dynamics\!\! &\!1-layer self-attention + MLP\\
		\midrule
		\cite{nichani2024understanding,zhu2024towards}& \multirow{2}{*}{Knowledge} & \multirow{2}{*}{\!\!Full dynamics\!\!}  &\multirow{2}{*}{\!1-layer linear or attention-only}\\
		\cite{cabannes2024learning,tian2023scan}&  & &  \\
		\midrule
		\cite{ghosal2024understanding}& Knowledge & \!\!Gradients\!\! & \!1-layer attention-only\!\\
		\midrule
		\cite{wang2025learning}& In-context &\!\!Layerwise dynamics\!\!&\!$L$-layer ($L\ge 2$) attention-only\!\\
		\midrule
		\cite{nichani2024transformers}& In-context &\!\!Layerwise dynamics\!\!&\!2-layer attention-only\!\\
		\midrule
		\cite{ahn2023transformers,mahankali2023one}& In-context & \!\!Critical points\!\!& Linear\!\\
		\midrule
		\cite{bietti2023birth}& In-context & \!\!Gradients\!\! & \!2-layer attention-only\!\\
		\midrule
		\cite{zhang2024trained}& In-context & \!\!Full dynamics\!\! & \!1-layer linear attention-only\!\\
		\midrule
		\cite{huang2023context,zhang2024trained}& In-context & \!\!Gradients\!\!& \!1-layer attention-only\!\\
		\midrule
		\cite{jelassi2022vision}& Img. cls. & \!\!Full dynamics\!\! & \!1-layer attention-only\!\\
		\midrule
		\cite{sakamoto2024benign}& Img. cls. & \!\!Full dynamics\!\! & \!1-layer attention-only\!\\
		\midrule
		\cite{wang2024transformers}& Others & \!\!Full dynamics\!\! & \!1-layer attention-only\!\\
		\midrule
		\cite{tian2023joma}& Others & \!\!Full dynamics\!\! &\! 1-layer uniform attention + MLP\!\\
		\bottomrule
	\end{tabular}
	
	\label{table: comparison2}
\end{table}

\paragraph{Memorization of transformers.}
Many recent works have studied the memorization capability of transformers. 
A line of work studied the memorization capacity of transformers.
\citet{yun2019transformers} demonstrated that transformers are universal approximators.
\citet{kim2023provable} proved that transformers can memorize sequence mappings of length-$n$ $d$-dimensional inputs with $\tilde{O}(d + n +\sqrt{nN})$ parameters.
Later, \citet{kajitsuka2024optimal} proved lower and upper bounds of the memorization capacity of transformers in next token prediction and sequence-to-sequence settings.
Additionally, \citet{mahdavi2023memorization} studied the memorization capacity of multi-head attention-only transformers.
\citet{madden2024next} proved upper and lower bounds of memorization capacity of one-layer transformers in next token prediction.
Another line of work studied associative memories in transformers.
Some works showed that associative memories scale with model sizes for linear models \citep{cabannes2023scaling} and one-layer transformers \citep{nichani2024understanding}.
Despite the powerful capacity shown in the above papers, it is unclear how the transformer learns to store and extract knowledge.
In this paper, we answer this question by analyzing the training dynamics of transformers.

\paragraph{Feature learning.}
A line of work theoretically studied the training dynamics of neural networks by analyzing the feature learning process.
Feature learning demonstrates benign overfitting over several data structures \citep{kou2023benign,xu2025rethinking} and reveals the benefits of knowledge distillation \citep{allen2020towards}, data augmentations \citep{zou2023benefits,shen2022data}, and mixture of experts \citep{chen2022towards}.
However, existing feature learning studies are currently limited to two-layer neural networks and attention-only transformers \citep{jiang2024unveil,jelassi2022vision} in image classification tasks. 
In this paper, we develop a new theoretical framework that incorporates self-attention and \ac{mlp} modules.

\paragraph{Knowledge acquisition and extraction of transformers.}
Many studies explored the mechanistic interpretability of knowledge acquisition and extraction of transformers.
A line of work tries to identify knowledge circuits in pre-trained \acp{llm}.
For example, \citet{geva2020transformer} studied the function of \ac{mlp} and highlighted that \acp{mlp} are key-value memories capturing input patterns.
Based on this, \citet{meng2022locating} localized factual associations in \acp{mlp} and proposed a knowledge editing approach. 
In contrast to knowledge circuits, \citet{ghosal2024understanding} investigated the impact of subject entity frequency on knowledge extraction through fine-tuning, demonstrating that increasing frequency enhances the factuality of downstream tasks.
\citet{allen2023physics1} designed controlled experiments to empirically identify the conditions under which pre-trained transformer models can perform generalizations, highlighting the importance of knowledge augmentations.
In our paper, we \emph{theoretically} prove some of their findings and characterize conditions for transformers to perform generalization.

\section{Attention score visualization}\label{app: attention visual}
We visualize the attention heatmap in Figure \ref{fig:attn_pt}. 
It verifies the characterization of the training dynamics in Section 6, which shows that attention scores on $\mathbf{o}$ are small and attention scores on subject and relation tokens are comparable.
\begin{figure}[H]
    \centering
    \includegraphics[width=0.6\linewidth]{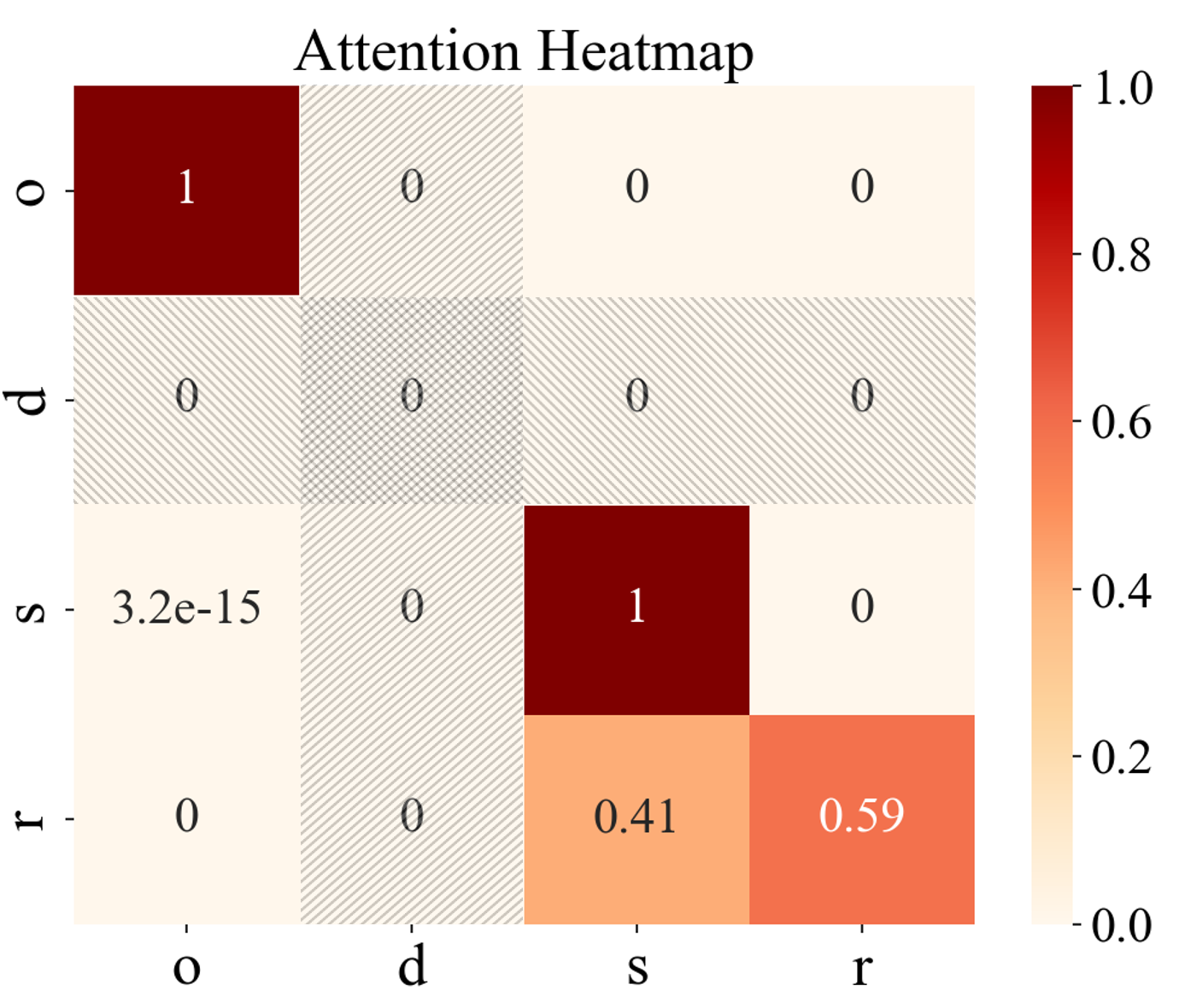}
    \caption{Attention heatmap after pre-training.}
    \label{fig:attn_pt}
\end{figure}

\section{Discussion of the tightness of low-rank FT}\label{app:low-rank}
For low-rank FT, the extra $m|\mathcal{R}|$ factor in the current upper bound comes
from controlling the gap between the full FT gradient and its best rank-1 approximation, so we
believe the dependence is likely loose. This is consistent with the intuition from Proposition 7 and with our additional experiments (Table 5), which show that low-rank FT behaves similarly to full FT.

\begin{table}[h]\label{tab:low-rank}
\centering
\caption{Comparison of full-rank and low-rank fine-tuning with $|\mathcal{R}|=30$. The two methods achieve nearly identical generalization accuracy across fine-tuning data sizes.}
\begin{tabular}{lccccc}
\hline
Fine-tuning data size & 5 & 10 & 15 & 20 & 40 \\
\hline
Full-rank FT generalization accuracy & 20\% & 38\% & 61\% & 91\% & 100\% \\
Low-rank FT generalization accuracy & 21\% & 38\% & 61\% & 91\% & 100\% \\
\hline
\end{tabular}
\label{table: low}
\end{table}

\section{Generalization accuracy and FT steps}\label{app: ft steps}
The visualization results for generalization accuracy and FT steps are shown in Figure \ref{fig:ft steps}. 
The results show that the accuracy after FT may decrease after a large number of training steps due to overfitting, indicating the effectiveness of early stopping in FT.
\begin{figure}[H]
    \centering
    \includegraphics[width=0.6\linewidth]{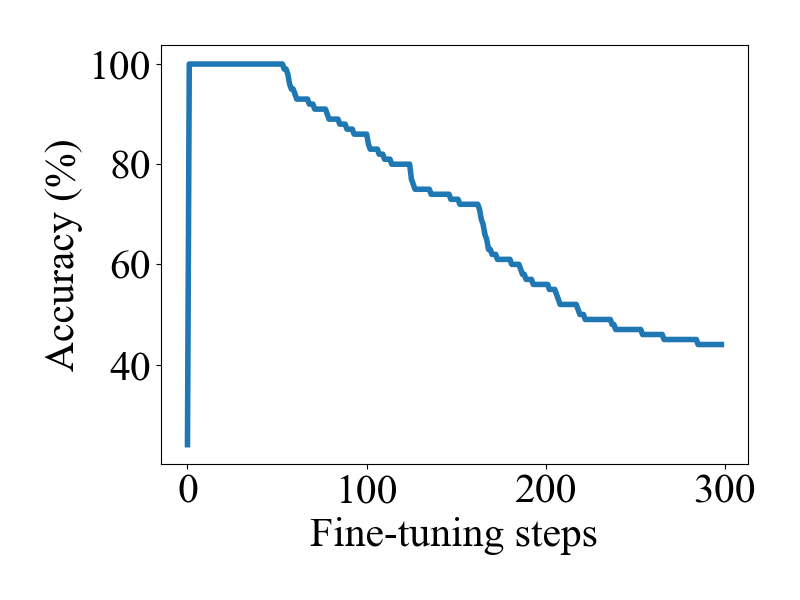}
    \caption{Generalization accuracy versus FT steps.}
    \label{fig:ft steps}
\end{figure}

\section{Preliminaries}\label{app: proof lemma}
This section provides the preliminaries essential for the proofs of our main results. We begin by deriving the explicit form of the model gradients with respect to its parameters and then present several supporting lemmas.

\subsection{Gradient computation}
For simplicity, we first rewrite the model output $\mathbf{x}^\text{output}$ in (\ref{equ: output}) in the following form:
\begin{align}
	x^\text{output}_i(\mathbf{W},\mathbf{Z},\mathbf{X}) = \frac{\lambda}{m}\sum_{k=1}^{m} \sigma(\langle\mathbf{w}_{i,k} , \mathbf{x}_a(\mathbf{Z},\mathbf{X})\rangle),
\end{align}
where $\mathbf{w}_{i,k}$ is the $k^{th}$ neuron associated with the $i^{th}$ element of the final output in the MLP.

We denote 
\begin{align}
    \textbf{logit}^{(t)}(\mathbf{X}) = \text{softmax}(\mathbf{x}^\text{output}(\mathbf{W}^{(t)},\mathbf{Z}^{(t)},\mathbf{X})).
\end{align}

Next, we present the model gradients in the following lemma.
\begin{lemma}\label{lemma: gradient}
	Under the simplified model architecture, the gradients of model parameters $\mathbf{W}$ and $\mathbf{Z}$ with respect to an example $(\mathbf{X}, y)$ with $\left\|\mathbf{X}[-1]\right\|_2 = 1$ are
	\begin{equation}
		\begin{aligned}
			&\nabla_{\mathbf{w}_{i,k}} \mathcal{L}(\mathbf{W},\mathbf{Z},\mathbf{X}, y) = -\frac{\lambda}{m}\left(\mathbb{I}(i=y)-\textnormal{logit}_i(\mathbf{X})\right)\sigma'(\mathbf{w}_{i,k}^\top\mathbf{x}_a)\mathbf{x}_a,
		\end{aligned}
	\end{equation}
	and
	\begin{equation}
		\begin{aligned}
			& \left(\nabla_{\mathbf{Z}} \mathcal{L}(\mathbf{W},\mathbf{Z},\mathbf{X}, y)\mathbf{X}[-1]\right)^\top\\
			=& -\frac{\lambda}{\sqrt{d}}\left(\mathbf{e}_y\!-\textbf{logit}(\mathbf{X})\right)^\top	\!\mathbf{W}_2\!\begin{bmatrix}
				\sigma'(\mathbf{w}_{1,1}^\top\mathbf{x}_a(\mathbf{Z},\mathbf{X}))\mathbf{w}_{1,1}\\
				\vdots\\
				\sigma'(\mathbf{w}_{1,m}^\top\mathbf{x}_a(\mathbf{Z},\mathbf{X}))\mathbf{w}_{1,m}\\
				\vdots\\
				\sigma'(\mathbf{w}_{d,m}^\top\mathbf{x}_a(\mathbf{Z},\mathbf{X}))\mathbf{w}_{d,m}
			\end{bmatrix}\!
			\mathbf{X}\!\left(\textnormal{diag}(\boldsymbol{\alpha}(\mathbf{Z},\mathbf{X}))\!-\!\boldsymbol{\alpha}(\mathbf{Z},\mathbf{X})\boldsymbol{\alpha}(\mathbf{Z},\mathbf{X})^\top\right)\!\mathbf{X}^\top.
		\end{aligned}
	\end{equation}
\end{lemma}
\begin{proof}
For notational simplicity, we treat all derivatives in the proof of Lemma 1 as Jacobian matrices. 
Consequently, the gradient of a scalar function is represented as a row vector.
To match the conclusions presented in Lemma \ref{lemma: gradient} (which assumes column vectors), the final gradient vectors derived here must be transposed.

Our goal is to compute the gradients of the loss function $\mathcal{L}(\mathbf{W},\mathbf{Z},\mathbf{X},y)$ with respect to $\mathbf{w}_{i,k}$ and $\mathbf{Z}$.
Using the following transformation: 
    \begin{align}\label{equ: grad_aux}
        \nabla_{\mathbf{Z}} \mathcal{L}(\mathbf{W},\mathbf{Z},\mathbf{X}, y) = \mathbf{X}[-1]\nabla_{\mathbf{Z}\mathbf{X}[-1]} \mathcal{L}(\mathbf{W},\mathbf{Z},\mathbf{X}, y).
    \end{align}
and chain rule, we can break down the gradients into several components:
	\begin{align}\label{equ: gradient1}
		&\nabla_{\mathbf{w}_{i,k}} \mathcal{L}(\mathbf{W},\mathbf{Z},\mathbf{X}, y) = \nabla_{\mathbf{x}^\text{output}} \mathcal{L}(\mathbf{W},\mathbf{Z},\mathbf{X}, y)\nabla_{\mathbf{w}_{i,k}}\mathbf{x}^\text{output},
	\end{align}
	and
	\begin{align}\label{equ: gradient2}
    \nabla_{\mathbf{Z}} \mathcal{L}(\mathbf{W},\mathbf{Z},\mathbf{X}, y)
    =& \mathbf{X}[-1]\nabla_{\mathbf{x}^\text{output}} \mathcal{L}(\mathbf{W},\mathbf{Z},\mathbf{X}, y)\nabla_{\mathbf{R}}\mathbf{x}^\text{output}\nabla_{\mathbf{x}_a}\mathbf{R}\nabla_{\mathbf{ZX}[-1]} \mathbf{x}_a,
	\end{align}
	where $\mathbf{R}\triangleq \sigma(\mathbf{W}_1\mathbf{x}_a(\mathbf{Z},\mathbf{X}))$.
	
	We then compute the components in (\ref{equ: gradient1}) and (\ref{equ: gradient2}) as follows:
	\begin{align}
		\nabla_{\mathbf{x}^\text{output}} \mathcal{L}(\mathbf{W},\mathbf{Z},\mathbf{X}, y) = 
		-\left(\mathbf{e}_{y}-\textbf{logit}(\mathbf{X})\right)^\top,
	\end{align}
	\begin{align}
		\nabla_{\mathbf{w}_{i,k}} x^\text{output}_i = \frac{\lambda}{m}\sigma'(\langle\mathbf{w}_{i,k}, \mathbf{x}_a(\mathbf{Z},\mathbf{X})\rangle)\mathbf{x}_a^\top(\mathbf{Z},\mathbf{X}),
	\end{align}
	\begin{align}
		\nabla_{\mathbf{R}}\mathbf{x}^\text{output} = \lambda\mathbf{W}_2,
	\end{align}
	\begin{align}
		\nabla_{\mathbf{x}_a}\mathbf{R} = 
		\begin{bmatrix}
			\sigma'(\mathbf{w}_{1,1}^\top\mathbf{x}_a(\mathbf{Z},\mathbf{X}))\mathbf{w}_{1,1}\\
			\vdots\\
			\sigma'(\mathbf{w}_{1,m}^\top\mathbf{x}_a(\mathbf{Z},\mathbf{X}))\mathbf{w}_{1,m}\\
			\vdots\\
			\sigma'(\mathbf{w}_{d,m}^\top\mathbf{x}_a(\mathbf{Z},\mathbf{X}))\mathbf{w}_{d,m}
		\end{bmatrix},
	\end{align}

	\begin{align}
		\nabla_{\mathbf{Z}\mathbf{X}[-1]}\mathbf{x}_a(\mathbf{Z},\mathbf{X})
        =& \nabla_{\boldsymbol{\alpha}(\mathbf{Z},\mathbf{X})}\mathbf{x}_a (\mathbf{Z},\mathbf{X})\nabla_{\mathbf{Z}\mathbf{X}[-1]}\boldsymbol{\alpha}(\mathbf{Z},\mathbf{X})\\ =& \frac{1}{\sqrt{d}}\mathbf{X}\left(\text{diag}(\boldsymbol{\alpha}(\mathbf{Z},\mathbf{X}))-\boldsymbol{\alpha}(\mathbf{Z},\mathbf{X})\cdot\boldsymbol{\alpha}(\mathbf{Z},\mathbf{X})^\top\right)\mathbf{X}^\top.
	\end{align}
Combining them and transposing the derived final gradient vectors completes the proof.
\end{proof}
Without loss of generality, we assume the gradient of the ReLU activation function at 0 to be $\sigma'(0)=1$.

\subsection{Notations}
For notational convenience in the subsequent proofs, we define the following shorthand at iteration $t$.
First, let:
\begin{align}
	\boldsymbol{\zeta}^{(t)}(\mathbf{X}) := \mathbf{W}_2\begin{bmatrix}
		\sigma'((\mathbf{w}_{1,1}^{(t)})^\top\mathbf{x}_a(\mathbf{Z}^{(t)},\mathbf{X}))\mathbf{w}_{1,1}^{(t)}\\
		\vdots\\
		\sigma'((\mathbf{w}_{1,m}^{(t)})^\top\mathbf{x}_a(\mathbf{Z}^{(t)},\mathbf{X}))\mathbf{w}_{1,m}^{(t)}\\
		\vdots\\
		\sigma'((\mathbf{w}_{d,m}^{(t)})^\top\mathbf{x}_a(\mathbf{Z}^{(t)},\mathbf{X}))\mathbf{w}_{d,m}^{(t)}
	\end{bmatrix}.
\end{align}
Then, for all $j\in[N], i\in\mathcal{R}$, we define:
\begin{align}
	\Xi^{(t)}([\mathbf{o}\:\:\mathbf{s}_j]) :=& \mathbf{x}_a(\mathbf{Z}^{(t)},[\mathbf{o}\:\:\mathbf{s}_j]),\\
    \tilde{\Xi}^{(t)}([\mathbf{o}\:\:\mathbf{s}_j]) :=& \mathbf{x}_a(\mathbf{Z}^{(t)},[\mathbf{o}\:\:\mathbf{s}_j]) - \mathbf{s}_j,\\
    \Xi^{(t)}([\mathbf{o}\:\:\mathbf{s}_j\:\:\mathbf{r}_i]) :=& \mathbf{x}_a(\mathbf{Z}^{(t)},[\mathbf{o}\:\:\mathbf{s}_j\:\:\mathbf{r}_i]),\\
    \tilde{\Xi}^{(t)}([\mathbf{o}\:\:\mathbf{s}_j\:\:\mathbf{r}_i]) :=& \mathbf{x}_a(\mathbf{Z}^{(t)},[\mathbf{o}\:\:\mathbf{s}_j\:\:\mathbf{r}_i]) - \mathbf{r}_i,\\
    \boldsymbol{\alpha}^{(t)}(\mathbf{X}) :=& \boldsymbol{\alpha}(\mathbf{Z}^{(t)},\mathbf{X}).
\end{align}
For any $\mathbf{b}_1,\mathbf{b}_2\in \{\mathbf{o}\}\cup\{\mathbf{d}\}\cup\{\mathbf{s}_j\}_{j=1}^N\cup\{\mathbf{r}_i\}_{i\in\mathcal{R}}\cup\{\mathbf{a}_j\}_{j=1}^N$, we define:
\begin{align}
    \tilde{Z}^{(t)}_{\mathcal{I}(\mathbf{\mathbf{b_1}}),\mathcal{I}(\mathbf{\mathbf{b_2}})} = \mathbf{b}_1^\top\mathbf{Z}^{(t)}\mathbf{b}_2.
\end{align}

\subsection{Lemmas}
To facilitate the analysis, we present the following useful lemmas.
\begin{lemma}\label{lemma: Gaussian_tail}
    A variable $x\sim \mathcal{N}(0,\sigma_0)$ satisfies
    \begin{align}
        \mathbb{P}\left[|x|> q\right] \le 2\exp\left(-\frac{q^2}{2\sigma_0^2}\right).
    \end{align}
\end{lemma}
\begin{lemma}\label{lemma: Gaussian tail lower}
    For a variable $x\sim \mathcal{N}(0,\sigma_0^2)$ and $t>0$, we have
    \begin{align}
        \mathbb{P}[|x| > t] \ge 1- \frac{2t}{\sqrt{2\pi}\sigma_0}.
    \end{align}
\end{lemma}
\begin{proof}
    As the Gaussian probability density function $f(\cdot)$ for $\mathcal{N}(0,\sigma_0^2)$ satisfies $f(x)\le \frac{1}{\sqrt{2\pi}\sigma_0}$, we have
    \begin{align}
        \mathbb{P}[|x| \le t] \le \frac{2t}{\sqrt{2\pi}\sigma_0}.
    \end{align}
    As a result, we have
    \begin{align}
        \mathbb{P}[|x| > t] \ge 1- \frac{2t}{\sqrt{2\pi}\sigma_0}.
    \end{align}
    This finishes the proof.
\end{proof}

\begin{lemma}\label{lemma: attn_init_order}
    Suppose $\delta>0$ and each element in a matrix $\mathbf{Z} \in \mathbb{R}^{d\times d}$ is generated following $\mathcal{N}(0,\sigma_0^2)$.
    With probability $1-\delta$, for each $i,j\in[d]$, we have 
    \begin{align}
        |Z_{i,j}| \le \sqrt{2\log(2d^2/\delta)}\sigma_0.
    \end{align}
\end{lemma}
\begin{proof}
   By Lemma \ref{lemma: Gaussian_tail}, for each $i,j\in[d]$, with probability $1-\delta/d^2$, we have that 
    \begin{align}
        |Z_{i,j}|< \sqrt{2\log(2d^2/\delta)}\sigma_0.
    \end{align}
    Using the Union bound over all $i,j\in[d]$, we finish the proof.
\end{proof}

\begin{lemma}\label{lemma: sqrtd}
	Suppose $d>0$ and there are two variables $x$ and $y$ satisfying $y\ge 2x$.
	The following holds:
	\begin{align}
		\frac{\exp(x)}{\exp(x) + \exp(y) + d} \le \frac{1}{\sqrt{d}}.
	\end{align}
\end{lemma}
\begin{proof}
	If $\exp(y)\le d$, we have
	\begin{align}
		\exp(x) \le \sqrt{d},
	\end{align}
	resulting in
	\begin{align}
		\frac{\exp(x)}{\exp(x) + \exp(y) + d} \le \frac{1}{\sqrt{d}}.
	\end{align}
	
	If $\exp(y)> d$, we have
	\begin{align}
		\frac{\exp(x)}{\exp(x) + \exp(y) + d} \le \frac{\exp(x)}{\exp(x) + \exp(y)} \le \frac{1}{\sqrt{d}}.
	\end{align}
	This completes the proof.
\end{proof}
\begin{lemma}\label{lemma: select_num}
	Suppose at each time $i\in [n]$, the set $\mathcal{A}_i$ contains $K$ entries that are i.i.d. and uniformly selected without replacement from the set $[q]$. 
	For all $j\in[q]$, the following holds:
	\begin{align}
		\mathbb{P}\left[\sum_{i=1}^{n}\mathbb{I}(j\in \mathcal{A}_i)\ge \frac{K}{2q}n\right] \ge 1-\exp\left(-\frac{nK^2}{2q^2}\right).
	\end{align}
\end{lemma}
\begin{proof}
	In each time $i\in[n]$, a number $j\in[q]$ has probability $\frac{K}{q}$ to selected.
	By Hoeffding's inequality, for any constant $t>0$, we have:
	\begin{align}
		\mathbb{P}\left[\sum_{i=1}^{n}\mathbb{I}(j\in \mathcal{A}_i) - \frac{K}{q}n\le -t\right] \le \exp\left(-\frac{2t^2}{n}\right).
	\end{align}
	Setting $-t$ to $-\oldfrac{K}{2q}n$, we have:
	\begin{align}
		\mathbb{P}\left[\sum_{i=1}^{n}\mathbb{I}(j\in \mathcal{A}_i) - \frac{K}{q}n\le -\frac{K}{2q}n\right] \le \exp\left(-\frac{nK^2}{2q^2}\right).
	\end{align}
        Reorganizing the results, we prove: 
        \begin{align}
            \mathbb{P}\left[\sum_{i=1}^{n}\mathbb{I}(j\in \mathcal{A}_i)\ge \frac{K}{2q}n\right] \ge 1-\exp\left(-\frac{nK^2}{2q^2}\right).
        \end{align}
\end{proof}
\begin{lemma}\label{lemma: init}
    With probability $1-\delta$, for all $i\in[md]$ and $j\in[d]$, we have
    \begin{align}
        |w_{i,j}^{(0)}|\le \sqrt{2\log(2md^2/\delta)}\sigma_0.
    \end{align}
\end{lemma}
\begin{proof}
    For all $i\in[md]$ and $j\in[d]$, we have
    \begin{align}
        \mathbb{P}\left[|w_{i,j}^{(0)}|>q\right]< 2\exp\left(-\frac{q^2}{2\sigma_0^2}\right).
    \end{align}
    By union bound, with probability $1-\delta$, we have
    \begin{align}
        |w_{i,j}^{(0)}| < \sqrt{2\log(2md^2/\delta)}\sigma_0.
    \end{align}
\end{proof}

\begin{lemma}
    For a variable $x\sim \mathcal{N}(0,\sigma_0^2)$, we have
    \begin{align}
        \mathbb{P}\left[0<x<q\right] \le \frac{q}{\sqrt{2\pi}\sigma_0}.
    \end{align}
\end{lemma}
\begin{proof}
    As the probability density function of $x$ is
    \begin{align}
        f(x) = \frac{1}{\sqrt{2\pi}\sigma_0}\exp\left(-\frac{x^2}{2\sigma_0^2}\right) \le \frac{1}{\sqrt{2\pi}\sigma_0}.
    \end{align}
    Therefore, we have
\begin{align}
    \mathbb{P}\left[0<x<q\right] = \int_{0}^q f(x)dx \le \frac{q}{\sqrt{2\pi}\sigma_0}.
\end{align}
\end{proof}

\begin{lemma}\label{lemma: cover1}
    In each of a total of $n$ independent trials, we randomly select a subset $\mathcal{X}_i$ of $0.4m$ distinct elements from $[m]$ uniformly at random (without replacement within a trial).
    When $n \ge 50\log(m/\delta)$, with probability $1-\delta$, for all $l\in[m]$ we have 
    \begin{align}
        \sum_{i=1}^n \mathbb{I}\left(l\in\mathcal{X}_i\right)\ge 0.3n.
    \end{align}
\end{lemma}
\begin{proof}
For all $l\in[m]$, we have 
    \begin{align}
        \mathbb{E}\left[\frac{1}{n}\sum_{i=1}^n \mathbb{I}\left(l\in\mathcal{X}_i\right)\right] = 0.4.
    \end{align}
    By Hoeffding's inequality, we have
\begin{align}
    \mathbb{P}\left[\frac{1}{n}\sum_{i=1}^n \mathbb{I}\left(l\in\mathcal{X}_i\right)\le 0.3\right] \le \exp\left(-\frac{n}{50}\right).
\end{align}
Using union bound for all $l\in[m]$, we finish the proof.   
\end{proof}

\section{Proofs for the pre-training phase}\label{app: proof pt}
For the sake of indexing the appeared subjects and answers with relations, we first define the following sets:
\begin{align}
	\mathcal{B}_j := \{i\in\mathcal{R}: (\mathbf{o},\mathbf{d},\mathbf{s}_j,\mathbf{r}_i,\mathbf{a}_j)\in\mathcal{T}\},
\end{align}
and 
\begin{align}
	\mathcal{D}_l := \{j\in[N]:(\mathbf{o},\mathbf{d},\mathbf{s}_j,\mathbf{r}_l,\mathbf{a}_j)\in\mathcal{T}\}.
\end{align}

\subsection{Activation patterns at initialization}
In this section, we characterize the activation patterns, i.e., which neurons in the MLP are activated at initialization.
The latter results are built upon conclusions in this subsection, which hold with high probability.

For convenience, we define the following sets for indexing the activated neurons:
\begin{align}
	\mathcal{S}_{\text{s},k,j}^{(t)} := \{i\in[m]: \langle\mathbf{w}_{\mathcal{I}(\mathbf{r}_k),i}^{(t)}, \mathbf{s}_j\rangle >0\},
\end{align}
\begin{align}
	\mathcal{S}_{\text{r},j,k}^{(t)} := \{i\in[m]: \langle\mathbf{w}_{\mathcal{I}(\mathbf{a}_j),i}^{(t)}, \Xi^{(t)}([\mathbf{o}\:\:\mathbf{s}_j\:\:\mathbf{r}_k])\rangle >0\},
\end{align}
and 
\begin{align}
    \mathcal{S}_{\text{o}}^{(t)} := \{i\in[m]: \langle\mathbf{w}_{\mathcal{I}(\mathbf{d}),i}^{(t)}, 2\mathbf{o}\rangle >0\}.
\end{align}

Now, we characterize the activation pattern at initialization.
\begin{lemma}\label{lemma:init}
	Suppose that $\delta > 0$ and $m\ge 50\log(\frac{6NK}{\delta})$. For all $j\in [N]$ and $k\in\mathcal{B}_j$, with probability at least $1-\delta$, the followings hold:
	\begin{align}
		|\mathcal{S}_{\text{s},k,j}^{(0)}| \ge 0.4m, |\mathcal{S}_{\text{r},j,k}^{(0)}| \ge 0.4m,
        \quad\text{and}\quad
        |\mathcal{S}_{\text{o}}^{(0)}| \ge 0.4m.
	\end{align}
\end{lemma}
\begin{proof}
    At initialization, we have
    \begin{align}
        \mathbb{P}[\langle\mathbf{w}_{\mathcal{I}(\mathbf{r}_k),i}^{(0)}, \mathbf{s}_j\rangle>0] = 1/2, \mathbb{P}[\langle\mathbf{w}_{\mathcal{I}(\mathbf{a}_j),i}^{(0)}, \Xi^{(0)}([\mathbf{o}\:\:\mathbf{s}_j\:\:\mathbf{r}_k])\rangle] = 1/2,
        \quad\text{and}\quad
        \mathbb{P}[\langle\mathbf{w}_{\mathcal{I}(\mathbf{d}),i}^{(0)}, 2\mathbf{o}\rangle] = 1/2.
    \end{align}
    By Hoeffding's inequality, with probability 1-$\delta/(3NK)$, we have
    \begin{align}
        \left|\frac{|\mathcal{S}_{\text{s},k,j}^{(0)}|}{m}-\frac{1}{2}\right| \le \sqrt{\frac{\log(6NK/\delta)}{2m}}.
    \end{align}
    Similarly, with probability 1-$\delta/(6NK)$, we have
    \begin{align}
         \left|\frac{|\mathcal{S}_{\text{r},j,k}^{(0)}|}{m}-\frac{1}{2}\right| \le \sqrt{\frac{\log(6NK/\delta)}{2m}}.
    \end{align}
    Additionally, with probability 1-$\delta/3$, we have
    \begin{align}
        \left|\frac{|\mathcal{S}_{\text{o}}^{(0)}|}{m}-\frac{1}{2}\right| \le \sqrt{\frac{\log(6/\delta)}{2m}}.
    \end{align}
    
    As long as $m\ge 50\log(6NK/\delta)$, by union bound over all the combinations of $j\in [N]$ and $k\in\mathcal{B}_j$, we have
    \begin{align}
        		|\mathcal{S}_{\text{s},k,j}^{(0)}| \ge 0.4m, |\mathcal{S}_{\text{r},j,k}^{(0)}| \ge 0.4m,
                \quad\text{and}\quad
                |\mathcal{S}_{\text{o}}^{(0)}| \ge 0.4m,
    \end{align}
    for all $j\in [N]$ and $k\in\mathcal{B}_j$, with probability at least $1-\delta$.
    This finishes the proof. 
\end{proof}

The analysis of activation patterns throughout the entire pre-training dynamics is not a separate proof but is integrated into the proofs for Stages 1, 2, and 3, which are detailed in the following subsections.

\subsection{Proof of pre-training Stage 1}
In this subsection, we prove several key properties of pre-training Stage 1, including the attention scores, activation patterns, and MLP feature learning.

\begin{lemma}\label{lemma: aux1}
	Under Condition \ref{condition: condition}, with probability $1-\delta$, the followings hold in pre-training Stage 1 ($t\le T_1$) with $T_1 = 2$.
	\begin{enumerate}	
        \item We have $\oldfrac{1}{m}\sum_{k=1}^{m}\langle \mathbf{w}_{\mathcal{I}(\mathbf{d}),k}^{(T_1)}, 2\mathbf{o}\rangle = \Omega(\frac{\log(d)}{\lambda})$.
		\item For all $j\in[N]$ and $i\in \mathcal{B}_j$, we have $1-\textnormal{logit}^{(T_1)}_{\mathcal{I}(\mathbf{a}_j)}([\mathbf{o}\:\:\mathbf{s}_j\:\:\mathbf{r}_i])= \Theta(1)$, and $1- \tilde{K}(j)\textnormal{logit}^{(T_1)}_{\mathcal{I}(\mathbf{r}_i)}([\mathbf{o}\:\:\mathbf{s}_j]) = \Theta(1)$.
		\item For all $j\in[N]$ and $i\in\mathcal{B}_j$, we have $ 0.8\le\alpha^{(T_1)}_{2}([\mathbf{o}\:\:\mathbf{s}_j\:\:\mathbf{r}_i])/\alpha^{(T_1)}_{3}([\mathbf{o}\:\:\mathbf{s}_j\:\:\mathbf{r}_i]) \le 1.2$.
		
        \item At the end of Stage 1, for all $j\in[N], i\in\mathcal{B}_j$, we have
	\begin{align}\nonumber
		\alpha^{(T_1)}_{1}([\mathbf{o}\:\:\mathbf{s}_j])\le \frac{1}{4d^{2}} \text{ and }\alpha^{(T_1)}_{1}([\mathbf{o}\:\:\mathbf{s}_j\:\:\mathbf{r}_i]) \le \frac{1}{4d^{2}}.
	\end{align}\\
        \item For each $j\in[N]$ and all $k\in \mathcal{B}_j$, $l \in \mathcal{S}_{\text{s},k,j}^{(0)}$, we have $\langle\mathbf{w}_{\mathcal{I}(\mathbf{r}_k),l}^{(t)}, \mathbf{s}_j)\rangle > 0$.\\
        \item For each $j\in[N]$ and all $k\in\mathcal{B}_j$, $l\in\mathcal{S}_{\text{r},j,k}^{(0)}$, we have $\langle \mathbf{w}^{(t)}_{\mathcal{I}(\mathbf{a}_j),l}, \Xi^{(t)}([\mathbf{o}\:\:\mathbf{s}_j\:\:\mathbf{r}_k]) \rangle > 0$.
	\end{enumerate}
\end{lemma}
\begin{proof}	
    In this proof, we establish each of the six statements in Lemma \ref{lemma: aux1} sequentially.
	
    \textbf{Proof of the first statement: }
    During the first iteration, we have
	\begin{equation}\label{equ: s1s3 1}
		\begin{aligned}
			&\frac{1}{m}\sum_{k=1}^{m}\langle \mathbf{w}_{\mathcal{I}(\mathbf{d}),k}^{(1)}, 2\mathbf{o}\rangle \!-\! \frac{1}{m}\sum_{k=1}^{m}\langle \mathbf{w}_{\mathcal{I}(\mathbf{d}),k}^{(0)}, 2\mathbf{o}\rangle\\
			\overset{(a)}{=}&\Theta(1)\cdot \left(\frac{\eta_1}{n}\frac{n}{3m^2}\sum_{k=1}^m(1-\text{logit}^{(0)}_{\mathcal{I}(\mathbf{d})}(2\mathbf{o}))\sigma'(\langle \mathbf{w}_{\mathcal{I}(\mathbf{d}),k}^{(0)}, 2\mathbf{o}\rangle) \lambda\left\|2\mathbf{o}\right\|_2^2\right.\\
			&-\! \frac{2\eta_1 \tilde{K}(j)}{m^2n}\sum_{k=1}^m\sum_{j=1}^{N}\text{logit}_{\mathcal{I}(\mathbf{d})}^{(0)}([\mathbf{o}\:\:\mathbf{s}_j])\alpha_{1}^{(0)}([\mathbf{o}\:\:\mathbf{s}_j])\sigma'(\langle \mathbf{w}_{\mathcal{I}(\mathbf{d}),k}^{(0)}, \Xi^{(0)}([\mathbf{o}\:\:\mathbf{s}_j])\rangle)\lambda\left\|\mathbf{o}\right\|_2^2\\
			&-\! \left.\frac{2\eta_1}{m^2n}\sum_{k=1}^m\sum_{j=1}^{N}\sum_{i\in\mathcal{B}_j} \text{logit}^{(0)}_{\mathcal{I}(\mathbf{d})}([\mathbf{o}\:\:\mathbf{s}_{j}\:\:\mathbf{r}_i])\alpha_{1}^{(0)}([\mathbf{o}\:\:\mathbf{s}_j\:\:\mathbf{r}_i])\sigma'(\langle \mathbf{w}_{\mathcal{I}(\mathbf{d}),k}^{(0)}, \Xi^{(0)}([\mathbf{o}\:\:\mathbf{s}_j\:\:\mathbf{r}_i])\rangle)\lambda\left\|\mathbf{o}\right\|_2^2\right)\\
			\overset{(b)}{=}& \Theta\left(\frac{\eta_1\lambda d}{m}\right),
		\end{aligned}
	\end{equation}
        where $(a)$ follows Lemmas \ref{lemma: gradient}, and $(b)$ holds because the initial \text{logit} terms satisfy $1-\text{logit}^{(0)}_{\mathcal{I}(\mathbf{d})}(\mathbf{o})=\Theta(1),\ \text{logit}_{\mathcal{I}(\mathbf{d})}^{(0)}([\mathbf{o}\:\:\mathbf{s}_j]) = \mathcal{O}(1/d),\ \text{and}\ \text{logit}^{(0)}_{\mathcal{I}(\mathbf{d})}([\mathbf{o}\:\:\mathbf{s}_{j}\:\:\mathbf{r}_i]) = \mathcal{O}(1/d)$ and Lemma \ref{lemma:init} which ensures the number of activated neurons are $\Theta(m)$.
        
	In the second iteration, we consider two cases.
    
    Case 1: If $\oldfrac{1}{m}\sum_{k=1}^{m}\langle \mathbf{w}_{\mathcal{I}(\mathbf{d}),k}^{(2)}, 2\mathbf{o}\rangle \!-\! \oldfrac{1}{m}\sum_{k=1}^{m}\langle \mathbf{w}_{\mathcal{I}(\mathbf{d}),k}^{(1)}, 2\mathbf{o}\rangle \ge 0$, we have
	\begin{align}\label{equ: s1s3 2}
		\frac{1}{m}\sum_{k=1}^{m}\langle \mathbf{w}_{\mathcal{I}(\mathbf{d}),k}^{(2)}, 2\mathbf{o}\rangle \!-\! \frac{1}{m}\sum_{k=1}^{m}\langle \mathbf{w}_{\mathcal{I}(\mathbf{d}),k}^{(1)}, 2\mathbf{o}\rangle = \mathcal{O}\left(\frac{\eta_1\lambda d}{m}\right),
	\end{align}
    with similar derivations as (\ref{equ: s1s3 1}).
    
	Case 2: While if $\oldfrac{1}{m}\sum_{k=1}^{m}\langle \mathbf{w}_{\mathcal{I}(\mathbf{d}),k}^{(2)}, 2\mathbf{o}\rangle \!-\! \oldfrac{1}{m}\sum_{k=1}^{m}\langle \mathbf{w}_{\mathcal{I}(\mathbf{d}),k}^{(1)}, 2\mathbf{o}\rangle < 0$, by Lemma \ref{lemma: sqrtd}, we have
    \begin{equation}\label{equ: s1s3 3}
	\begin{aligned}
		&\frac{1}{m}\sum_{k=1}^{m}\langle \mathbf{w}_{\mathcal{I}(\mathbf{d}),k}^{(2)}, 2\mathbf{o}\rangle \!-\! \frac{1}{m}\sum_{k=1}^{m}\langle \mathbf{w}_{\mathcal{I}(\mathbf{d}),k}^{(1)}, 2\mathbf{o}\rangle\\
        \ge& 
        -\! \frac{2\eta_1 }{m^2n}\Theta(1)\sum_{k=1}^m\sum_{j=1}^{N}\tilde{K}(j)\text{logit}_{\mathcal{I}(\mathbf{d})}^{(1)}([\mathbf{o}\:\:\mathbf{s}_j])\alpha_{1}^{(1)}([\mathbf{o}\:\:\mathbf{s}_j])\sigma'(\langle \mathbf{w}_{\mathcal{I}(\mathbf{d}),k}^{(1)}, \Xi^{(1)}([\mathbf{o}\:\:\mathbf{s}_j])\rangle)\lambda\left\|\mathbf{o}\right\|_2^2\\
		&-\frac{\eta_1}{m^2n}\Theta(1)\sum_{k=1}^m\sum_{j=1}^{N}\sum_{i\in\mathcal{B}_j} \text{logit}^{(1)}_{\mathcal{I}(\mathbf{d})}([\mathbf{o}\:\:\mathbf{s}_{j}\:\:\mathbf{r}_i])\alpha_{1}^{(1)}([\mathbf{o}\:\:\mathbf{s}_j\:\:\mathbf{r}_i])\sigma'(\langle \mathbf{w}_{\mathcal{I}(\mathbf{d}),k}^{(1)}, \Xi^{(1)}([\mathbf{o}\:\:\mathbf{s}_j\:\:\mathbf{r}_i])\rangle)\lambda\left\|\mathbf{o}\right\|_2^2\\
        =& -\mathcal{O}\left(\frac{\eta_1\lambda \sqrt{d}}{m}\right).
	\end{aligned}
    \end{equation}
	Combining (\ref{equ: s1s3 1}), (\ref{equ: s1s3 2}) and (\ref{equ: s1s3 3}), we have
	\begin{align}\label{equ: s1s3}
		\frac{1}{m}\sum_{k=1}^{m}\langle \mathbf{w}_{\mathcal{I}(\mathbf{d}),k}^{(T_1)}, 2\mathbf{o}\rangle \!-\! \frac{1}{m}\sum_{k=1}^{m}\langle \mathbf{w}_{\mathcal{I}(\mathbf{d}),k}^{(0)}, 2\mathbf{o}\rangle = \Theta\left(\frac{\eta_1\lambda d}{m}\right)= \Omega\left(\frac{\log(d)}{\lambda}\right).
	\end{align}
    Therefore, we can conclude that
    \begin{align}
        \frac{1}{m}\sum_{k=1}^{m}\langle \mathbf{w}_{\mathcal{I}(\mathbf{d}),k}^{(T_1)}, 2\mathbf{o}\rangle = -\tilde{\mathcal{O}}(\sqrt{d}\sigma_0/\sqrt{m})+ \Omega\left(\frac{\log(d)}{\lambda}\right) = \Omega\left(\frac{\log(d)}{\lambda}\right).
    \end{align}

    \textbf{Proof of the second statement:}
	For $t\le T_1$, for all $k\in[m],j\in[N]$, and $i\in\mathcal{B}_j$, we have
	\begin{equation}\label{equ: s2r1}
		\begin{aligned}
			&\left\langle\mathbf{w}_{\mathcal{I}(\mathbf{r}_i),k}^{(t+1)}, \Xi^{(t+1)}([\mathbf{o}\:\:\mathbf{s}_j])\right\rangle \!-\! \left\langle\mathbf{w}_{\mathcal{I}(\mathbf{r}_i),k}^{(t)}, \Xi^{(t)}([\mathbf{o}\:\:\mathbf{s}_j])\right\rangle\\
			\overset{(a)}{=}&\mathcal{O}\left(\frac{\eta_1}{mn}\lambda\left\|\mathbf{o} + 2\mathbf{s}_j\right\|_2^2\right) + \Theta(d)\cdot\frac{\lambda\eta_1}{nm}(1-\tilde{K}(j)\text{logit}^{(t)}_{\mathcal{I}(\mathbf{r}_i)}([\mathbf{o}\:\:\mathbf{s}_j])\\
			\overset{(b)}{=}& \mathcal{O}\left(\frac{\lambda\eta_1 d}{mn}\right),
		\end{aligned}
	\end{equation}
where $(a)$ is obtained by Lemmas \ref{lemma: gradient} and \ref{lemma: pattern pt}, Cauchy–Schwarz inequality and the fact that $|\langle\mathbf{w}_{\mathcal{I}(\mathbf{r}_i),k}^{(t)},$ $ \Xi^{(t+1)}([\mathbf{o}\:\:\mathbf{s}_j])-\Xi^{(t)}([\mathbf{o}\:\:\mathbf{s}_j])\rangle| =\mathcal{O}(\frac{\lambda \eta_1\left\|\mathbf{o} + 2\mathbf{s}_j\right\|_2^2}{(nm)})$ .
$(b)$ is due to the fact that $\left\|\mathbf{o} + 2\mathbf{s}_j\right\|_2^2 = \mathcal{O}(d)$ and $1-\tilde{K}(j)\text{logit}^{(t)}_{\mathcal{I}(\mathbf{r}_i)}([\mathbf{o}\:\:\mathbf{s}_j]) \le 1$.

	As there are at most $T_1$ iterations during Stage 1, for $t< T_1$, using (\ref{equ: s2r1}), we have
	\begin{align}\label{equ: s1 1}											\frac{1}{m}\sum_{k=1}^{m}\langle\mathbf{w}_{\mathcal{I}(\mathbf{r}_i),k}^{(t)}, \Xi^{(t)}([\mathbf{o}\:\:\mathbf{s}_j])\rangle  =\mathcal{O}\left(\frac{\lambda\eta_1 d}{mn}T_1\right) = \mathcal{O}\left(\frac{\lambda\eta_1 d}{mn}\right).
	\end{align}
	By (\ref{equ: s1s3 1}) and (\ref{equ: s1 1}), the results
	 $\oldfrac{1}{m}\sum_{k=1}^{m}\langle \mathbf{w}_{\mathcal{I}(\mathbf{d}),k}^{(T_1)}, 2\mathbf{o}\rangle = \Theta(\lambda\eta_1 d/m)$ and $\oldfrac{1}{m}\sum_{k=1}^{m}\langle\mathbf{w}_{\mathcal{I}(\mathbf{r}_i),k}^{(T_1)}, \Xi^{(T_1)}([\mathbf{o}\:\:\mathbf{s}_j])\rangle = \mathcal{O} (\oldfrac{1}{m}\sum_{k=1}^{m}\langle \mathbf{w}_{\mathcal{I}(\mathbf{d}),k}^{(T_1)}, \Xi^{(T_1)}([\mathbf{o}\:\:\mathbf{s}_j])\rangle/n)$ hold.
	Thus, we have
	\begin{align}
		1-\tilde{K}(j)\text{logit}^{(T_1)}_{\mathcal{I}(\mathbf{r}_i)}([\mathbf{o}\:\:\mathbf{s}_j]) = 1 - \mathcal{O}\left(\frac{\tilde{K}(j)}{d^{1-1/n}}\right) = \Theta(1).
	\end{align}
	
	In addition, the increment of $\langle\mathbf{w}_{\mathcal{I}(\mathbf{a}_j),k}^{(t)}, \Xi^{(t)}([\mathbf{o}\:\:\mathbf{s}_j\:\:\mathbf{r}_i])\rangle$ at iterations $t$ satisfies
	\begin{equation} 
	\begin{aligned}
		&\langle\mathbf{w}_{\mathcal{I}(\mathbf{a}_j),k}^{(t+1)}, \Xi^{(t+1)}([\mathbf{o}\:\:\mathbf{s}_j\:\:\mathbf{r}_i])\rangle - \langle\mathbf{w}_{\mathcal{I}(\mathbf{a}_j),k}^{(t)}, \Xi^{(t)}([\mathbf{o}\:\:\mathbf{s}_j\:\:\mathbf{r}_i])\rangle\\
		\overset{(a)}{=}& \mathcal{O}\left(\frac{\eta_1}{nm}\sum_{i\in\mathcal{B}_j}\lambda\left\|\mathbf{o} + \mathbf{s}_j + 2 \mathbf{r}_i\right\|_2^2\right) + \mathcal{O}\left(\frac{\eta_1\lambda d}{nm}\sum_{i\in\mathcal{B}_j}(1-\text{logit}^{(t)}([\mathbf{o}\;\;\mathbf{s}_j\;\;\mathbf{r}_i]))\right)\\
		\overset{(b)}{=}& \mathcal{O}\left(\frac{\lambda\eta_1 d \tilde{K}(j)}{nm}\right),
	\end{aligned}
	\end{equation}
	where $(a)$ is by Lemma \ref{lemma: gradient} and the facts that $1-\text{logit}^{(t)}_{\mathcal{I}(\mathbf{a}_j)}([\mathbf{o}\:\:\mathbf{s}_j\:\:\mathbf{r}_i]) \le 1$ and $|\langle\mathbf{w}_{\mathcal{I}(\mathbf{a}_j),k}^{(t)},\Xi^{(t+1)}([\mathbf{o}\:\:\mathbf{s}_j\:\:\mathbf{r}_i])-\Xi^{(t)}([\mathbf{o}\:\:\mathbf{s}_j\:\:\mathbf{r}_i])\rangle| = \mathcal{O}(\frac{\eta_1}{(nm)}\sum_{i\in\mathcal{B}_j}\left\|\mathbf{o}+\mathbf{s}_j+2\mathbf{r}_i\right\|_2^2)$ for $t\le T_1$.
    $(b)$ is because of $1-\text{logit}^{(t)}_{\mathcal{I}(\mathbf{a}_j)}([\mathbf{o}\:\:\mathbf{s}_j\:\:\mathbf{r}_i]) \le 1$ and $\left\|\mathbf{o} + \mathbf{s}_j + 2 \mathbf{r}_i\right\|_2^2 = \mathcal{O}(d)$.
	
	After $T_1$ iterations, we have 
	\begin{align}
		\frac{1}{m}\sum_{k=1}^{m}\langle\mathbf{w}_{\mathcal{I}(\mathbf{a}_j),k}^{(T_1)}, \Xi^{(T_1)}([\mathbf{o}\:\:\mathbf{s}_j\:\:\mathbf{r}_i])\rangle = \mathcal{O}\left(\frac{\lambda\eta_1 d \tilde{K}(j)}{nm}\right).
	\end{align}
	Then, as $\oldfrac{1}{m}\sum_{k=1}^{m}\langle \mathbf{w}_{\mathcal{I}(\mathbf{d}),k}^{(T_1)}, 2\mathbf{o}\rangle = \Theta(\lambda\eta_1 d/m)$ and $\oldfrac{1}{m}\sum_{k=1}^{m}\langle\mathbf{w}_{\mathcal{I}(\mathbf{r}_i),k}^{(T_1)}, \Xi^{(T_1)}([\mathbf{o}\:\:\mathbf{s}_j\:\:\mathbf{r}_i])\rangle$
    
    $= \mathcal{O} (\oldfrac{1}{m}\sum_{k=1}^{m}\langle \mathbf{w}_{\mathcal{I}(\mathbf{d}),k}^{(T_1)}, \Xi^{(T_1)}([\mathbf{o}\:\:\mathbf{s}_j\:\:\mathbf{r}_i])\rangle/n)$, we have
	\begin{align}
		1-\text{logit}^{(T_1)}_{\mathcal{I}(\mathbf{a}_j)}([\mathbf{o}\:\:\mathbf{s}_j\:\:\mathbf{r}_i]) = 1 - \mathcal{O}\left(\frac{d^{\tilde{K}(j)/n}}{d}\right)= \Theta(1).
	\end{align}

    \textbf{Proof of the third statement:}
    For all $t\in[T_1]$, we have
	\begin{align}
		\frac{\alpha^{(t)}_{2}([\mathbf{o},\mathbf{s}_j,\mathbf{r}_i])}{\alpha^{(t)}_{3}([\mathbf{o},\mathbf{s}_j,\mathbf{r}_i])} = \frac{\exp\left(\tilde{Z}^{(t)}_{\mathcal{I}(\mathbf{s}_j),\mathcal{I}(\mathbf{r}_i)}/\sqrt{d}\right)}{\exp\left(\tilde{Z}^{(t)}_{\mathcal{I}(\mathbf{r}_i),\mathcal{I}(\mathbf{r}_i)}/\sqrt{d}\right)} = \exp\left(\left(\tilde{Z}^{(t)}_{\mathcal{I}(\mathbf{s}_j),\mathcal{I}(\mathbf{r}_i)} - \tilde{Z}^{(t)}_{\mathcal{I}(\mathbf{r}_i),\mathcal{I}(\mathbf{r}_i)}\right)/\sqrt{d}\right).
	\end{align}
	For all $t\in[T_1]$, the increments of the terms $\tilde{Z}^{(t)}_{\mathcal{I}(\mathbf{s}_j),\mathcal{I}(\mathbf{r}_i)}$ and $\tilde{Z}^{(t)}_{\mathcal{I}(\mathbf{r}_i),\mathcal{I}(\mathbf{r}_i)}$ satisfy
	\begin{equation}\label{equ: s1a1}
	\begin{aligned}
		&|\tilde{Z}^{(t+1)}_{\mathcal{I}(\mathbf{s}_j),\mathcal{I}(\mathbf{r}_i)}- \tilde{Z}^{(t)}_{\mathcal{I}(\mathbf{s}_j),\mathcal{I}(\mathbf{r}_i)}|\\		=& \frac{\eta_1\lambda}{n\sqrt{d}}\underbrace{\alpha_{2}^{(t)}([\mathbf{o}\:\:\mathbf{s}_j\:\:\mathbf{r}_i])}_{\le 1}\underbrace{\left|\left(\mathbf{e}_{\mathcal{I}(\mathbf{a}_j)}-\textbf{logit}^{(t)}([\mathbf{o}\:\:\mathbf{s}_{j}\:\:\mathbf{r}_i])\right)^\top\boldsymbol{\zeta}^{(t)}([\mathbf{o}\:\:\mathbf{s}_j\:\:\mathbf{r}_i])\left(\mathbf{s}_j - \tilde{\Xi}^{(t)}([\mathbf{o}\:\:\mathbf{s}_j\:\:\mathbf{r}_i])\right)\right|}_{=\mathcal{O}(\eta_1 \lambda d/m) \text{ by (\ref{equ: s1s3})}}\\
		=& \mathcal{O}\left(\frac{\lambda^2\eta_1^2\sqrt{d} }{nm}\right),
	\end{aligned}
	\end{equation}
	and
	\begin{equation}\label{equ: s1a2}
		\begin{aligned}
			&|\tilde{Z}^{(t+1)}_{\mathcal{I}(\mathbf{r}_i),\mathcal{I}(\mathbf{r}_i)}- \tilde{Z}^{(t)}_{\mathcal{I}(\mathbf{r}_i),\mathcal{I}(\mathbf{r}_i)}|\\
			=& \frac{\eta_1\lambda}{n\sqrt{d}}\\
            &\sum_{j\in\mathcal{D}_i}\underbrace{\alpha_{3}^{(t)}([\mathbf{o}\:\:\mathbf{s}_j\:\:\mathbf{r}_i])}_{\le 1}\underbrace{\left|\left(\mathbf{e}_{\mathcal{I}(\mathbf{a}_j)}-\textbf{logit}^{(t)}([\mathbf{o}\:\:\mathbf{s}_{j}\:\:\mathbf{r}_i])\right)^\top\boldsymbol{\zeta}^{(t)}([\mathbf{o}\:\:\mathbf{s}_j\:\:\mathbf{r}_i])\left(\mathbf{r}_i - \tilde{\Xi}^{(t)}([\mathbf{o}\:\:\mathbf{s}_j\:\:\mathbf{r}_i])\right)\right|}_{=\mathcal{O}(\eta_1 \lambda d/m) \text{ by (\ref{equ: s1s3})}}\\
			=& \mathcal{O}\left(\frac{\lambda^2\eta_1^2\sqrt{d}}{m}\right).
		\end{aligned}
	\end{equation} 
	
	Furthermore, combining (\ref{equ: s1a1}), (\ref{equ: s1a2}), and the fact that $ |\tilde{Z}^{(0)}_{\mathcal{I}(\mathbf{s}_j),\mathcal{I}(\mathbf{r}_i)} - \tilde{Z}^{(0)}_{\mathcal{I}(\mathbf{r}_i),\mathcal{I}(\mathbf{r}_i)}| = \tilde{\mathcal{O}}(\sigma_0)$ by Lemma \ref{lemma: init}, we have 
    \begin{equation}
    \begin{aligned}
		&\frac{|\tilde{Z}^{(2)}_{\mathcal{I}(\mathbf{s}_j),\mathcal{I}(\mathbf{r}_i)} - \tilde{Z}^{(2)}_{\mathcal{I}(\mathbf{r}_i),\mathcal{I}(\mathbf{r}_i)}|}{\sqrt{d}}\\
        \le& \frac{\sum_{t=1}^2|\tilde{Z}^{(t)}_{\mathcal{I}(\mathbf{s}_j),\mathcal{I}(\mathbf{r}_i)} - \tilde{Z}^{(t-1)}_{\mathcal{I}(\mathbf{s}_j),\mathcal{I}(\mathbf{r}_i)}| + |\tilde{Z}^{(0)}_{\mathcal{I}(\mathbf{s}_j),\mathcal{I}(\mathbf{r}_i)} - \tilde{Z}^{(0)}_{\mathcal{I}(\mathbf{r}_i),\mathcal{I}(\mathbf{r}_i)}| + \sum_{t=1}^2|\tilde{Z}^{(t)}_{\mathcal{I}(\mathbf{r}_i),\mathcal{I}(\mathbf{r}_i)}-\tilde{Z}^{(t-1)}_{\mathcal{I}(\mathbf{r}_i),\mathcal{I}(\mathbf{r}_i)}|}{\sqrt{d}}\\
        =& \mathcal{O}\left(\eta_1^2\lambda^2/m\right).
    \end{aligned}
    \end{equation}
	Under Condition \ref{condition: condition}, we have
	\begin{align}\label{equ: s1s31}
		\exp\left(\left(\tilde{Z}^{(T_1)}_{\mathcal{I}(\mathbf{s}_j),\mathcal{I}(\mathbf{r}_i)} - \tilde{Z}^{(T_1)}_{\mathcal{I}(\mathbf{r}_i),\mathcal{I}(\mathbf{r}_i)}\right)/\sqrt{d}\right) = \exp\left(\mathcal{O}\left(\eta_1^2\lambda^2T_1/m\right)\right) \le 1.2.
	\end{align}
	Using similar proof, we also have 
	\begin{align}\label{equ: s1s32}
		\exp\left(\left(\tilde{Z}^{(T_1)}_{\mathcal{I}(\mathbf{s}_j),\mathcal{I}(\mathbf{r}_i)} - \tilde{Z}^{(T_1)}_{\mathcal{I}(\mathbf{r}_i),\mathcal{I}(\mathbf{r}_i)}\right)/\sqrt{d}\right)\ge 0.8.
	\end{align}

    \textbf{Proof of the fourth statement:}
    First, we prove the first part that $\alpha^{(T_1)}_{1}([\mathbf{o}\:\:\mathbf{s}_j])\le \frac{1}{4d^{2}}$.
    
    The increments of $ \tilde{Z}^{(t)}_{\mathcal{I}(\mathbf{o}),\mathcal{I}(\mathbf{s}_j)}$ in the first and the second iterations satisfy
\begin{equation}\label{equ: s1s41}
	\begin{aligned}
		&|\tilde{Z}^{(1)}_{\mathcal{I}(\mathbf{o}),\mathcal{I}(\mathbf{s}_j)}- \tilde{Z}^{(0)}_{\mathcal{I}(\mathbf{o}),\mathcal{I}(\mathbf{s}_j)}|\\
        		\le& \frac{\eta_1\lambda}{n\sqrt{d}}\underbrace{\alpha_{1}^{(0)}([\mathbf{o},\mathbf{s}_j])}_{\le 1}\sum_{i\in\mathcal{B}_j}\underbrace{\left|\left(\mathbf{e}_{\mathcal{I}(\mathbf{r}_i)}-\textbf{logit}^{(0)}([\mathbf{o}\:\:\mathbf{s}_{j}])\right)^\top\boldsymbol{\zeta}^{(0)}([\mathbf{o}\:\:\mathbf{s}_j])\left(\mathbf{o} - \tilde{\Xi}^{(0)}([\mathbf{o}\:\:\mathbf{s}_j])\right)\right|}_{=\mathcal{O}(\sigma_0\sqrt{d}) \text{ by Lemma \ref{lemma: init}}}\cdot\Theta(\sqrt{d})\\
		=& \mathcal{O}\left(\frac{\lambda\eta_1 \sigma_0\sqrt{d}K }{n}\right),
	\end{aligned}
\end{equation}
and
	\begin{equation}\label{equ: s1s42}
	\begin{aligned}
		&\tilde{Z}^{(2)}_{\mathcal{I}(\mathbf{o}),\mathcal{I}(\mathbf{s}_j)}- \tilde{Z}^{(1)}_{\mathcal{I}(\mathbf{o}),\mathcal{I}(\mathbf{s}_j)}\\
		=& \frac{\eta_1\lambda}{n\sqrt{d}}\underbrace{\alpha_{1}^{(1)}([\mathbf{o},\mathbf{s}_j])}_{=\Theta(1)}\sum_{i\in\mathcal{B}_j}\left(\mathbf{e}_{\mathcal{I}(\mathbf{r}_i)}-\textbf{logit}^{(1)}([\mathbf{o}\:\:\mathbf{s}_{j}])\right)^\top\boldsymbol{\zeta}^{(1)}([\mathbf{o}\:\:\mathbf{s}_j])\left(\mathbf{o} - \tilde{\Xi}^{(1)}([\mathbf{o}\:\:\mathbf{s}_j])\right)\cdot\Theta(\sqrt{d})\\
		\overset{(a)}{=}& -\Theta\left(\frac{\lambda^2\eta_1^2dK }{nm}\right),
	\end{aligned}
\end{equation}
where $(a)$ holds because according to (\ref{equ: s1s3}), $\zeta^{(1)}_{\mathcal{I}(\mathbf{d})}([\mathbf{o}\;\;\mathbf{s}_j]) = \Omega(\log(d))$ holds, resulting in $\text{logit}^{(1)}_{\mathcal{I}(\mathbf{d})}([\mathbf{o}\;\;\mathbf{s}_j]) = \Theta(1)$, which inturn yields $\left(\mathbf{e}_{\mathcal{I}(\mathbf{r}_i)}-\textbf{logit}^{(1)}([\mathbf{o}\:\:\mathbf{s}_{j}])\right)^\top\boldsymbol{\zeta}^{(1)}([\mathbf{o}\:\:\mathbf{s}_j])\left(\mathbf{o} - \tilde{\Xi}^{(1)}([\mathbf{o}\:\:\mathbf{s}_j])\right) = -\Theta(\eta_1\lambda d/m)$.

Additionally, the increment of $\tilde{Z}^{(t)}_{\mathcal{I}(\mathbf{s}_j),\mathcal{I}(\mathbf{s}_j)}$ for $t\in[T_1]$ satisfies
	\begin{equation}\label{equ: s1s43}
	\begin{aligned}
		&|\tilde{Z}^{(t+1)}_{\mathcal{I}(\mathbf{s}_j),\mathcal{I}(\mathbf{s}_j)}- \tilde{Z}^{(t)}_{\mathcal{I}(\mathbf{s}_j),\mathcal{I}(\mathbf{s}_j)}|\\
		\le& \frac{\eta_1\lambda}{n\sqrt{d}}\underbrace{\alpha_{2}^{(t)}([\mathbf{o},\mathbf{s}_j])}_{=\Theta(1)}\sum_{i\in\mathcal{B}_j}\underbrace{\left|\left(\mathbf{e}_{\mathcal{I}(\mathbf{r}_i)}-\textbf{logit}^{(t)}([\mathbf{o}\:\:\mathbf{s}_{j}])\right)^\top\boldsymbol{\zeta}^{(t)}([\mathbf{o}\:\:\mathbf{s}_j])\left(\mathbf{s}_j - \tilde{\Xi}^{(t)}([\mathbf{o}\:\:\mathbf{s}_j])\right)\right|}_{=\mathcal{O}(\eta_1\lambda d/m) \text{ by (\ref{equ: s1s3})}}\\
		=& \mathcal{O}\left(\frac{\lambda^2\eta_1^2\sqrt{d}K}{nm}\right).
	\end{aligned}
\end{equation}

Combining (\ref{equ: s1s41}), (\ref{equ: s1s42}) (\ref{equ: s1s43}) and the fact that $ |\tilde{Z}^{(0)}_{\mathcal{I}(\mathbf{o}),\mathcal{I}(\mathbf{s}_j)} - \tilde{Z}^{(0)}_{\mathcal{I}(\mathbf{s}_j),\mathcal{I}(\mathbf{s}_j)}| = \tilde{\mathcal{O}}(\sigma_0)$ by Lemma \ref{lemma: init}, we have
 \begin{equation}
    \begin{aligned}
		&\frac{\tilde{Z}^{(2)}_{\mathcal{I}(\mathbf{o}),\mathcal{I}(\mathbf{s}_j)} - \tilde{Z}^{(2)}_{\mathcal{I}(\mathbf{s}_j),\mathcal{I}(\mathbf{s}_j)}}{\sqrt{d}}\\
        \le& \frac{\sum_{t=1}^2\tilde{Z}^{(t)}_{\mathcal{I}(\mathbf{o}),\mathcal{I}(\mathbf{s}_j)} - \tilde{Z}^{(t-1)}_{\mathcal{I}(\mathbf{o}),\mathcal{I}(\mathbf{s}_j)} + |\tilde{Z}^{(0)}_{\mathcal{I}(\mathbf{o}),\mathcal{I}(\mathbf{s}_j)} - \tilde{Z}^{(0)}_{\mathcal{I}(\mathbf{s}_j),\mathcal{I}(\mathbf{s}_j)}| + \sum_{t=1}^2|\tilde{Z}^{(t)}_{\mathcal{I}(\mathbf{s}_j),\mathcal{I}(\mathbf{s}_j)}-\tilde{Z}^{(t-1)}_{\mathcal{I}(\mathbf{s}_j),\mathcal{I}(\mathbf{s}_j)}|}{\sqrt{d}}\\
        \le& -\Theta\left(\frac{\lambda^2\eta^2_1\sqrt{d}K}{nm}\right).
    \end{aligned}
    \end{equation}
Under Condition \ref{condition: condition}, we have
\begin{align}
	\exp\left(\left(\tilde{Z}^{(2)}_{\mathcal{I}(\mathbf{o}),\mathcal{I}(\mathbf{s}_j)} - \tilde{Z}^{(2)}_{\mathcal{I}(\mathbf{s}_j),\mathcal{I}(\mathbf{s}_j)}\right)/\sqrt{d}\right) \le\exp\left(-\Theta\left(\frac{\eta_1^2\lambda^2\sqrt{d}K}{nm}\right)\right)\le \frac{1}{4d^{2}}.
\end{align}

Next, we prove the second conclusion that $\alpha^{(T_1)}_{1}([\mathbf{o}\:\:\mathbf{s}_j\:\:\mathbf{r}_i]) \le \frac{1}{4d^{2}}$.
The increments of $\tilde{Z}_{\mathcal{I}(\mathbf{o}),\mathcal{I}(\mathbf{r}_i)}^{(t)}$ in the first and the second iterations satisfy
\begin{equation}\label{equ: s1r1}
	\begin{aligned}
		&|\tilde{Z}^{(1)}_{\mathcal{I}(\mathbf{o}),\mathcal{I}(\mathbf{r}_i)}- \tilde{Z}^{(0)}_{\mathcal{I}(\mathbf{o}),\mathcal{I}(\mathbf{r}_i)}|\\
		=& \frac{\eta_1\lambda}{n\sqrt{d}}\sum_{j\in\mathcal{D}_i}\underbrace{\alpha_{1}^{(0)}([\mathbf{o},\mathbf{s}_j,\mathbf{r}_i])}_{=\Theta (1)}\\
        &\cdot\underbrace{\left|\left(\mathbf{e}_{\mathcal{I}(\mathbf{a}_j)}-\textbf{logit}^{(0)}([\mathbf{o}\:\:\mathbf{s}_{j}\:\:\mathbf{r}_i])\right)^\top\boldsymbol{\zeta}^{(0)}([\mathbf{o}\:\:\mathbf{s}_j\:\:\mathbf{r}_i])\cdot\left(\mathbf{o} - \tilde{\Xi}^{(0)}([\mathbf{o}\:\:\mathbf{s}_j\:\:\mathbf{r}_i])\right)\right|}_{=\mathcal{O}(\sigma_0\sqrt{d}) \text{ by Lemma \ref{lemma: init}}}\cdot\Theta(\sqrt{d})\\
		=& \mathcal{O}\left(\frac{\lambda\eta_1 \sigma_0 \sqrt{d}}{n\sqrt{m}}\right),
	\end{aligned}
\end{equation}
and
\begin{equation}\label{equ: s1r2}
	\begin{aligned}
		&\tilde{Z}^{(2)}_{\mathcal{I}(\mathbf{o}),\mathcal{I}(\mathbf{r}_i)}- \tilde{Z}^{(1)}_{\mathcal{I}(\mathbf{o}),\mathcal{I}(\mathbf{r}_i)}\\
		=& \frac{\eta_1\lambda}{n\sqrt{d}}\sum_{j\in\mathcal{D}_i}\underbrace{\alpha_{1}^{(1)}([\mathbf{o}\:\:\mathbf{s}_{j}\:\:\mathbf{r}_i])}_{=\Theta(1)}\\
        &\cdot\left(\mathbf{e}_{\mathcal{I}(\mathbf{a}_j)}-\textbf{logit}^{(1)}([\mathbf{o}\:\:\mathbf{s}_{j}\:\:\mathbf{r}_i])\right)^\top\boldsymbol{\zeta}^{(1)}([\mathbf{o}\:\:\mathbf{s}_j\:\:\mathbf{r}_i])\cdot\left(\mathbf{o} - \tilde{\Xi}^{(1)}([\mathbf{o}\:\:\mathbf{s}_j\:\:\mathbf{r}_i])\right)\cdot\Theta(\sqrt{d})\\
		\overset{(a)}{=}& -\Omega\left(\frac{\lambda^2\eta_1^2d}{nm}\right),
	\end{aligned}
\end{equation}
where $(a)$ holds because according to (\ref{equ: s1s3}), $\zeta^{(1)}_{\mathcal{I}(\mathbf{d})}([\mathbf{o}\;\;\mathbf{s}_j\;\;\mathbf{r}_i]) = \Omega(\log(d))$ holds, resulting in $\text{logit}^{(1)}_{\mathcal{I}(\mathbf{d})}([\mathbf{o}\;\;\mathbf{s}_j\;\;\mathbf{r}_i]) = \Theta(1)$, which in turn yields $\left(\mathbf{e}_{\mathcal{I}(\mathbf{a}_j)}-\textbf{logit}^{(1)}([\mathbf{o}\:\:\mathbf{s}_{j}\:\:\mathbf{r}_i])\right)^\top\boldsymbol{\zeta}^{(1)}([\mathbf{o}\:\:\mathbf{s}_j\:\:\mathbf{r}_i])\cdot\left(\mathbf{o} - \tilde{\Xi}^{(1)}([\mathbf{o}\:\:\mathbf{s}_j\:\:\mathbf{r}_i])\right) = -\Theta(\eta_1\lambda d/m)$.

Consequently, combining (\ref{equ: s1r1}), (\ref{equ: s1r2}), (\ref{equ: s1a1}), and the fact that $|\tilde{Z}^{(0)}_{\mathcal{I}(\mathbf{o}),\mathcal{I}(\mathbf{r}_j)} - \tilde{Z}^{(0)}_{\mathcal{I}(\mathbf{s}_j),\mathcal{I}(\mathbf{r}_i)}| = \tilde{\mathcal{O}}(\sigma_0)$ by Lemma \ref{lemma: init}, we have
 \begin{equation}
    \begin{aligned}
		&\frac{\tilde{Z}^{(2)}_{\mathcal{I}(\mathbf{o}),\mathcal{I}(\mathbf{r}_i)} - \tilde{Z}^{(2)}_{\mathcal{I}(\mathbf{s}_j),\mathcal{I}(\mathbf{r}_i)}}{\sqrt{d}}\\
        \le& \frac{\sum_{t=1}^2\tilde{Z}^{(t)}_{\mathcal{I}(\mathbf{o}),\mathcal{I}(\mathbf{r}_i)} - \tilde{Z}^{(t-1)}_{\mathcal{I}(\mathbf{o}),\mathcal{I}(\mathbf{r}_i)} + |\tilde{Z}^{(0)}_{\mathcal{I}(\mathbf{o}),\mathcal{I}(\mathbf{r}_i)} - \tilde{Z}^{(0)}_{\mathcal{I}(\mathbf{s}_j),\mathcal{I}(\mathbf{r}_i)}| + \sum_{t=1}^2|\tilde{Z}^{(t)}_{\mathcal{I}(\mathbf{s}_j),\mathcal{I}(\mathbf{r}_i)}-\tilde{Z}^{(t-1)}_{\mathcal{I}(\mathbf{s}_j),\mathcal{I}(\mathbf{r}_i)}|}{\sqrt{d}}\\
        \le& -\Theta\left(\frac{\lambda^2\eta^2_1\sqrt{d}K}{nm}\right).
    \end{aligned}
\end{equation}

Under condition \ref{condition: condition}, we have
\begin{align}\label{s1s4}
	\exp\left(\tilde{Z}^{(2)}_{\mathcal{I}(\mathbf{o}),\mathcal{I}(\mathbf{r}_i)}- \tilde{Z}^{(2)}_{\mathcal{I}(\mathbf{s}_j),\mathcal{I}(\mathbf{r}_i)}/\sqrt{d}\right) = \exp\left(-\Theta\left(\frac{\eta_1^2\lambda^2\sqrt{d}}{nm}\right)\right) \le \frac{1}{4d^{2}}.
\end{align}
Combining (\ref{s1s4}) with (\ref{equ: s1s31}), (\ref{equ: s1s32}) yields the conclusion.

\textbf{Proof of the fifth statement:}
    By Lemma \ref{lemma: Gaussian tail lower}, for $k\in \mathcal{B}_j$ and $i \in \mathcal{S}_{\text{s},k,j}^{(0)}$, we have
    \begin{align}
        \mathbb{P}\left[\langle\mathbf{w}_{\mathcal{I}(\mathbf{r}_k),i}^{(0)}, \mathbf{s}_j\rangle \ge \Theta\left(\frac{\lambda}{nmd}\right)\right] \ge 1-\Theta\left(\frac{\lambda}{nmd\sigma_0}\right).
    \end{align}
    In Stage 1, by the gradient update, if a neuron $i\in[m]$ satisfies 
    \begin{align}
        \langle \mathbf{w}_{\mathcal{I}(\mathbf{r}_k),i}^{(0)}, \Xi^{(0)}\left([\mathbf{o}\;\;\mathbf{s}_j]\right) \rangle \le 0 \text{ or } \langle \mathbf{w}_{\mathcal{I}(\mathbf{r}_k),i}^{(0)},\Xi^{(0)}\left([\mathbf{o}\;\;\mathbf{s}_j\;\; \mathbf{r}_i]\right) \rangle \ge 0,
    \end{align}
    we have
    \begin{align}\label{equ: s1s51}
        \langle \mathbf{w}_{\mathcal{I}(\mathbf{r}_k),i}^{(t+1)}, \mathbf{s}_j \rangle - \langle \mathbf{w}_{\mathcal{I}(\mathbf{r}_k),i}^{(t)}, \mathbf{s}_j \rangle = -\mathcal{O}\left(\frac{\lambda\eta_1}{nmd}\right).
    \end{align}
    Using the union bound over $i \in \mathcal{S}_{\text{s},k,j}^{(0)}$, with probability $1-\Theta(\frac{\lambda}{(n\sqrt{d})})$, we have
        \begin{align}\label{equ: s1s52}
            \langle\mathbf{w}_{\mathcal{I}(\mathbf{r}_k),i}^{(0)}, \mathbf{s}_j\rangle \ge \Theta\left(\frac{\lambda}{nmd}\right).
        \end{align}
    Combining (\ref{equ: s1s51}) and (\ref{equ: s1s52}), for all $t\in[T_1]$, we have
\begin{align}
    \langle\mathbf{w}_{\mathcal{I}(\mathbf{r}_k),i}^{(t)}, \mathbf{s}_j\rangle \ge 0.
\end{align}

\textbf{Proof of the sixth statement:}
Based on Lemma \ref{lemma:init}, for all $l\in\mathcal{S}_{\text{r},j,k}^{(0)}$, the updates of $\langle \mathbf{w}^{(t)}_{\mathcal{I}(\mathbf{a}_j),l},\mathbf{r}_k\rangle$ and $ \langle \mathbf{w}^{(t)}_{\mathcal{I}(\mathbf{a}_j),l},\mathbf{s}_j\rangle$ satisfy:
\begin{equation}
	\begin{aligned}
		\langle \mathbf{w}^{(t+1)}_{\mathcal{I}(\mathbf{a}_j),l},\mathbf{r}_k \rangle
		\ge& \langle \mathbf{w}^{(t)}_{\mathcal{I}(\mathbf{a}_j),l},\mathbf{r}_k\rangle\\
		&+ \underbrace{\lambda\frac{\eta^{(t)}}{nm}(1-\text{logit}^{(t)}_{\mathcal{I}(\mathbf{a}_j)}([\mathbf{o}\:\:\mathbf{s}_j\:\:\mathbf{r}_k])) - 2\lambda\sum_{i\neq j,i\in\mathcal{D}_k}\frac{\eta^{(t)}}{nm}\text{logit}^{(t)}_{\mathcal{I}(\mathbf{a}_j)}([\mathbf{o}\:\:\mathbf{s}_i\:\:\mathbf{r}_k])}_{\Delta_{1,t}},
	\end{aligned}
\end{equation}
and 
\begin{equation}
	\begin{aligned}
		\langle \mathbf{w}^{(t+1)}_{\mathcal{I}(\mathbf{a}_j),l},\mathbf{s}_j \rangle
		\ge& \langle \mathbf{w}^{(t)}_{\mathcal{I}(\mathbf{a}_j),l},\mathbf{s}_j\rangle\\
		&\underbrace{-\lambda \frac{\eta^{(t)} \tilde{K}(j)}{nm}\text{logit}_{\mathcal{I}(\mathbf{a}_j)}^{(t)}([\mathbf{o}\:\:\mathbf{s}_j])\!+\!\lambda\frac{\Theta(1)\cdot\eta^{(t)}}{nm}\sum_{k\in \mathcal{B}_j}(1\!-\!\text{logit}^{(t)}_{\mathcal{I}(\mathbf{a}_j)}([\mathbf{o}\:\:\mathbf{s}_j\:\:\mathbf{r}_k]))}_{\Delta_{2,t}},
	\end{aligned}
\end{equation}
by the fact that $\left\|\mathbf{r}_k\right\|_2^2 \le \langle\Xi^{(t)}([\mathbf{o}\:\:\mathbf{s}_j\:\:\mathbf{r}_k]),\mathbf{r}_k\rangle \le 2\left\|\mathbf{r}_k\right\|_2^2$.

Similarly, for $\langle \mathbf{w}^{(t+1)}_{\mathcal{I}(\mathbf{a}_j),l},\mathbf{o} \rangle$, we have
\begin{equation}
\begin{aligned}
    &\langle \mathbf{w}^{(t+1)}_{\mathcal{I}(\mathbf{a}_j),l},\mathbf{o} \rangle\ge \langle \mathbf{w}^{(t)}_{\mathcal{I}(\mathbf{a}_j),l},\mathbf{o} \rangle \\
    &\underbrace{-\lambda \frac{\eta^{(t)} \tilde{K}(j)\sqrt{d}}{nm}\text{logit}_{\mathcal{I}(\mathbf{a}_j)}^{(t)}([\mathbf{o}\:\:\mathbf{s}_j])\!+\!\lambda\frac{\Theta(1)\cdot\eta^{(t)}\sqrt{d}}{nm}\sum_{k\in \mathcal{B}_j}(1\!-\!\text{logit}^{(t)}_{\mathcal{I}(\mathbf{a}_j)}([\mathbf{o}\:\:\mathbf{s}_j\:\:\mathbf{r}_k]))-\mathcal{O}\left(\frac{\lambda\eta^{(t)}}{md^{1/2}}\right)}_{\Delta^o_{t}},
\end{aligned}
\end{equation}
When $1\le t \le T_1$, $\langle \mathbf{w}^{(t)}_{\mathcal{I}(\mathbf{a}_j),l},\mathbf{r}_k\rangle$, since $1-\text{logit}^{(t)}_{\mathcal{I}(\mathbf{a}_j)}([\mathbf{o}\:\:\mathbf{s}_j\:\:\mathbf{r}_k]) = \Theta(1)$, we have$\Delta_{1,t}>0$, $\Delta_{2,t}>0$, and $\Delta^o_{t}>0$. 
Therefore, the inner products $\langle \mathbf{w}^{(t)}_{\mathcal{I}(\mathbf{a}_j),l},\mathbf{s}_j \rangle$ and $\langle \mathbf{w}^{(t)}_{\mathcal{I}(\mathbf{a}_j),l},\mathbf{o} \rangle$ remain positive. 
As a result, for all $t\in[T_1]$ and $l\in\mathcal{S}_{\text{r},j,k}^{(0)}$, we have
\begin{align}
    \langle \mathbf{w}^{(t)}_{\mathcal{I}(\mathbf{a}_j),l}, \Xi^{(t)}([\mathbf{o}\:\:\mathbf{s}_j\:\:\mathbf{r}_k]) \rangle > 0.
\end{align}
	This completes the proof.
\end{proof}

\subsection{Proof of pre-training Stage 2}
In this subsection, we will prove several key properties of pre-training Stage 1, including the attention scores, activation patterns, and MLP feature learning.
First, we prove the following lemma.
\begin{lemma}\label{lemma: mag_n}
    For all $t\in[0,T_p]$, the following holds:
    \begin{itemize}
	\item For all $j\in [N], l\in[m]$ and $k\neq j$, we have
    \begin{align}
		\frac{1}{m}\sum_{l=1}^m\langle\mathbf{w}^{(t)}_{\mathcal{I}(\mathbf{a}_j),l}, \mathbf{s}_k\rangle =\tilde{\mathcal{O}}\left(\sigma_0\right).
	\end{align}
    \item For all $j\in[N],l\in[m]$ and $i\notin\mathcal{B}_j$, we have
    \begin{align}
        \frac{1}{m}\sum_{l=1}^m\langle\mathbf{w}^{(t)}_{\mathcal{I}(\mathbf{a}_j),l}, \mathbf{r}_i\rangle =\tilde{\mathcal{O}}\left(\sigma_0\right).
    \end{align}
    \item For all $i,k\in\mathcal{R}$, $l\in[m]$, we have
    \begin{align}
        \frac{1}{m}\sum_{l=1}^m\langle\mathbf{w}^{(t)}_{\mathcal{I}(\mathbf{r}_i),l}, \mathbf{r}_k\rangle =\tilde{\mathcal{O}}\left(\sigma_0\right).
    \end{align}
    \end{itemize}
\end{lemma}
\begin{proof}
    We prove these three statements sequentially.

    \textbf{Proof of the first statement:}
    During pre-training $t\in[0,T_p-1]$, for all $l\in[m]$, $j\in[N]$, and $k\neq j$, we have
    \begin{equation}
    \begin{aligned}
        &\langle\mathbf{w}^{(t+1)}_{\mathcal{I}(\mathbf{a}_j),l},\mathbf{s}_k\rangle - \langle\mathbf{w}^{(t)}_{\mathcal{I}(\mathbf{a}_j),l},\mathbf{s}_k\rangle\\
        =& - \sum_{i\in\mathcal{B}_k}\alpha^{(t)}_2([\mathbf{o}\;\;\mathbf{s}_k\;\;\mathbf{r}_i])\frac{\eta^{(t)}\lambda}{mn}\mathbb{I}(\langle\mathbf{w}^{(t)}_{\mathcal{I}(\mathbf{a}_j),l},\Xi^{(t)}([\mathbf{o}\;\;\mathbf{s}_k\;\;\mathbf{r}_i])\rangle>0)\text{logit}_{\mathcal{I}(\mathbf{a}_j)}^{(t)}([\mathbf{o}\;\;\mathbf{s}_k\;\;\mathbf{r}_i])\\
        &-\sum_{i\in\mathcal{B}_k}\alpha_{2}^{(t)}([\mathbf{o}\;\;\mathbf{s}_k])\frac{\eta^{(t)}\lambda}{mn}\mathbb{I}(\langle\mathbf{w}^{(t)}_{\mathcal{I}(\mathbf{a}_j),l},\Xi^{(t)}([\mathbf{o}\;\;\mathbf{s}_k])\rangle>0)\text{logit}_{\mathcal{I}(\mathbf{a}_j)}^{(t)}([\mathbf{o}\;\;\mathbf{s}_k])\\
        \le& 0.
    \end{aligned}
    \end{equation}
    Therefore, for all $l\in[m]$, $t \in [T_p]$ and $j\in[N]$, $k\neq j$, by Lemma \ref{lemma: init}, we have
\begin{equation}
    \begin{aligned}
         \langle\mathbf{w}^{(t)}_{\mathcal{I}(\mathbf{a}_j),l},\mathbf{s}_k\rangle \le  \langle\mathbf{w}^{(0)}_{\mathcal{I}(\mathbf{a}_j),l},\mathbf{s}_k\rangle = \tilde{\mathcal{O}}(\sigma_0).
    \end{aligned}
    \end{equation}

    \textbf{Proof of the second statement:}
    During pre-training $t\in[0,T_p-1]$, for all $l\in[m]$, $j\in[N]$, and $i\notin\mathcal{B}_j$, we have
    \begin{equation}
    \begin{aligned}
        &\langle\mathbf{w}^{(t+1)}_{\mathcal{I}(\mathbf{a}_j),l},\mathbf{r}_i\rangle - \langle\mathbf{w}^{(t)}_{\mathcal{I}(\mathbf{a}_j),l},\mathbf{r}_i\rangle\\
        =& - \sum_{j\in\mathcal{D}_i}\alpha_{3}^{(t)}([\mathbf{o}\;\;\mathbf{s}_j\;\;\mathbf{r}_i])\frac{\eta^{(t)}\lambda}{mn}\mathbb{I}(\langle\mathbf{w}^{(t)}_{\mathcal{I}(\mathbf{a}_j),l},\Xi^{(t)}([\mathbf{o}\;\;\mathbf{s}_j\;\;\mathbf{r}_i])\rangle>0)\text{logit}_{\mathcal{I}(\mathbf{a}_j)}^{(t)}([\mathbf{o}\;\;\mathbf{s}_j\;\;\mathbf{r}_i])\\
        \le& 0.
    \end{aligned}
    \end{equation}
    Hence, for all $l\in[m],j\in[N]$ and $i\notin \mathcal{B}_j$, by Lemma \ref{lemma: init}, we have
    \begin{equation}
    \begin{aligned}
         \langle\mathbf{w}^{(t)}_{\mathcal{I}(\mathbf{a}_j),l},\mathbf{r}_i\rangle \le  \langle\mathbf{w}^{(0)}_{\mathcal{I}(\mathbf{a}_j),l},\mathbf{r}_i\rangle = \tilde{\mathcal{O}}(\sigma_0).
    \end{aligned}
    \end{equation}

    \textbf{Proof of the third statement:}
    During Pre-training $t\in[0,T_p-1]$, for all $i,k\in\mathcal{R}$ and $l\in[m]$, we have
    \begin{equation}
    \begin{aligned}
        &\langle\mathbf{w}^{(t+1)}_{\mathcal{I}(\mathbf{r}_i),l},\mathbf{r}_k\rangle - \langle\mathbf{w}^{(t)}_{\mathcal{I}(\mathbf{r}_i),l},\mathbf{r}_k\rangle\\
        =& - \sum_{j\in\mathcal{D}_k}\alpha_{3}^{(t)}([\mathbf{o}\;\;\mathbf{s}_j\;\;\mathbf{r}_k])\frac{\eta^{(t)}\lambda}{mn}\mathbb{I}(\langle\mathbf{w}^{(t)}_{\mathcal{I}(\mathbf{r}_i),l},\Xi^{(t)}([\mathbf{o}\;\;\mathbf{s}_j\;\;\mathbf{r}_k])\rangle>0)\text{logit}_{\mathcal{I}(\mathbf{r}_i)}^{(t)}([\mathbf{o}\;\;\mathbf{s}_j\;\;\mathbf{r}_k])\\
        \le& 0.
    \end{aligned}
    \end{equation}
    Hence, for all $l\in[m]$ and $i,k \in \mathcal{R}$, by Lemma \ref{lemma: init}, we have
    \begin{equation}
    \begin{aligned}
         \langle\mathbf{w}^{(t)}_{\mathcal{I}(\mathbf{r}_i),l},\mathbf{r}_k\rangle \le  \langle\mathbf{w}^{(0)}_{\mathcal{I}(\mathbf{r}_i),l},\mathbf{r}_k\rangle = \tilde{\mathcal{O}}(\sigma_0).
    \end{aligned}
    \end{equation}
    
This completes the proof.
\end{proof}

Now, we proceed to derive the key properties in Stage 2. 
\begin{lemma}\label{lemma: pt_stage2}
Under Condition \ref{condition: condition}, with probability $1-\delta$, in pre-training Stage 2 ($T_1\le t\le T_2 = \Theta(nm\log(d)/(\lambda^2\eta_2 K))$), the followings hold:
\begin{enumerate}
	\item For all $j\in[N]$ and $i\in\mathcal{B}_j$, we have $1-\textnormal{logit}^{(t)}_{\mathcal{I}(\mathbf{a}_j)}([\mathbf{o}\:\:\mathbf{s}_j\:\:\mathbf{r}_i]) = \Theta(1)$ and $1-\tilde{K}(j)\textnormal{logit}^{(t)}_{\mathcal{I}(\mathbf{r}_i)}([\mathbf{o}\:\:\mathbf{s}_j]) = \Theta(1)$.	\item For all $j\in[N]$ and $i\in\mathcal{B}_j$, we have $0.7 \le \frac{\alpha^{(t)}_{2}([\mathbf{o}\:\:\mathbf{s}_j\:\:\mathbf{r}_i])}{\alpha^{(t)}_{3}([\mathbf{o}\:\:\mathbf{s}_j\:\:\mathbf{r}_i])} \le 1.6$.
    \item For all $j\in[N]$ and $i\in\mathcal{B}_j$, we have
	\begin{align}
		\alpha^{(t)}_{1}([\mathbf{o}\:\:\mathbf{s}_j])\le \frac{2}{3d^{2}} \text{ and }\alpha^{(t)}_{1}([\mathbf{o}\:\:\mathbf{s}_j\:\:\mathbf{r}_i]) \le \frac{2}{3d^{2}}.
	\end{align}
    \item We have $\oldfrac{1}{m}\sum_{k=1}^{m}\langle \mathbf{w}_{\mathcal{I}(\mathbf{d}),k}^{(t)}, 2\mathbf{o}\rangle = \Omega(\frac{\log(d)}{\lambda})$ and $\oldfrac{1}{m}\sum_{k=1}^{m}\langle \mathbf{w}_{\mathcal{I}(\mathbf{d}),k}^{(t)}, 2\mathbf{o}\rangle = \mathcal{O}(\frac{d\log(d)}{\lambda})$.
    \item For each $j\in[N]$ and all $k\in \mathcal{B}_j$, $l \in \mathcal{S}_{\text{s},k,j}^{(0)}$, we have $\langle\mathbf{w}_{\mathcal{I}(\mathbf{r}_k),l}^{(t)}, \Xi^{(t)}([\mathbf{o}\;\;\mathbf{s}_j]))\rangle > 0$.\\
    \item For each $j\in[N]$ and all $k\in\mathcal{B}_j$, $l\in\mathcal{S}_{\text{r},j,k}^{(0)}$. we have $\langle \mathbf{w}^{(t)}_{\mathcal{I}(\mathbf{a}_j),l}, \Xi^{(t)}([\mathbf{o}\:\:\mathbf{s}_j\:\:\mathbf{r}_k]) \rangle > 0$.
    \item For all $j\in[N]$, $l\in[m]$, we have 
    \begin{align}
    \langle\mathbf{w}^{(t)}_{\mathcal{I}(\mathbf{r}_i),l},\mathbf{o}\rangle \le \mathcal{O}\left(\frac{\log(d)}{\lambda}\right),
        \langle\mathbf{w}^{(t)}_{\mathcal{I}(\mathbf{a}_j),l},\mathbf{o}\rangle \le \mathcal{O}\left(\frac{\log(d)}{n\lambda}\right), 
    \end{align}
\end{enumerate}
\end{lemma}
\begin{proof}
    In this proof, we establish each of the seven statements in Lemma \ref{lemma: pt_stage2} by induction.
    It is straightforward to demonstrate that the conclusions hold at $T_1$ using Lemma \ref{lemma: aux1}.
    Suppose that the conclusions hold at iteration $T_1\le t\le T_2-1$. 
    We now prove that the conclusions hold at iteration $t+1$ sequentially.
    
    \textbf{Proof of the first statement:}
	For any $j\in[N]$, $i\in\mathcal{B}_j$, the model updates satisfy
	\begin{equation}
	\begin{aligned}
		&\frac{1}{m}\sum_{k=1}^{m}\langle \mathbf{w}^{(t+1)}_{\mathcal{I}(\mathbf{r}_i),k},\Xi^{(t+1)}([\mathbf{o}\:\:\mathbf{s}_j])\rangle- \frac{1}{m}\sum_{k=1}^{m}\langle \mathbf{w}^{(t)}_{\mathcal{I}(\mathbf{r}_i),k}, \Xi^{(t)}([\mathbf{o}\:\:\mathbf{s}_j])\rangle\\ \overset{(a)}{\le}& \Theta(1)\cdot\frac{\lambda\eta_2}{nm}(1-\tilde{K}(j)\text{logit}^{(t)}_{\mathcal{I}(\mathbf{r}_i)}([\mathbf{o}\:\:\mathbf{s}_j])\\
		\overset{(b)}{=}& \mathcal{O} \left(\frac{\lambda\eta_2}{nm}\right),
	\end{aligned}
	\end{equation}
	and
	\begin{equation}
		\begin{aligned}
			&\frac{1}{m}\sum_{k=1}^{m}\langle \mathbf{w}^{(t+1)}_{\mathcal{I}(\mathbf{a}_j),k}, \Xi^{(t+1)}([\mathbf{o}\:\:\mathbf{s}_j\:\:\mathbf{r}_i])\rangle - \frac{1}{m}\sum_{k=1}^{m}\langle \mathbf{w}^{(t)}_{\mathcal{I}(\mathbf{a}_j),k}, \Xi^{(t)}([\mathbf{o}\:\:\mathbf{s}_j\:\:\mathbf{r}_i])\rangle\\
			\overset{(c)}{=}&\Theta(1)\cdot \left(\frac{\lambda\eta_2}{nm}(1-\text{logit}_{\mathcal{I}(\mathbf{a}_j)}^{(t)}([\mathbf{o}\:\:\mathbf{s}_j\:\:\mathbf{r}_i]))\langle \Xi^{(t)}([\mathbf{o}\:\:\mathbf{s}_j\:\:\mathbf{r}_i]), \Xi^{(t+1)}([\mathbf{o}\:\:\mathbf{s}_j\:\:\mathbf{r}_i])\rangle\right.\\
			&+ \sum_{k\neq i}\frac{\lambda\eta_2}{nm}(1-\text{logit}_{\mathcal{I}(\mathbf{a}_j)}^{(t)}([\mathbf{o}\:\:\mathbf{s}_j\:\:\mathbf{r}_k]))\langle \Xi^{(t)}([\mathbf{o}\:\:\mathbf{s}_j\:\:\mathbf{r}_k]), \Xi^{(t+1)}([\mathbf{o}\:\:\mathbf{s}_j\:\:\mathbf{r}_i])\rangle\\
			&- \left.\frac{\lambda\eta_2}{nm} \text{logit}^{(t)}_{\mathcal{I}(\mathbf{a}_j)}([\mathbf{o}\:\:\mathbf{s}_j])\langle\Xi^{(t)}([\mathbf{o}\:\:\mathbf{s}_j]), \Xi^{(t+1)}([\mathbf{o}\:\:\mathbf{s}_j\:\:\mathbf{r}_i])\rangle\right)\\
			=& \mathcal{O} \left(\frac{\lambda\eta_2 K}{nm}\right),
		\end{aligned}
	\end{equation}
    where $(a)$ is because of the gradient form in Lemma \ref{lemma: gradient}, the fifth statement and Lemma \ref{lemma:init}, $(b)$ is due to $1-\tilde{K}(j)\text{logit}^{(t)}_{\mathcal{I}(\mathbf{r}_i)}([\mathbf{o}\:\:\mathbf{s}_j] = \Theta(1)$, and $(c)$ is because of Lemmas \ref{lemma: gradient},  the fifth statement, Lemma \ref{lemma:init} and the fact that under Condition \ref{condition: condition}, $\Xi^{(t+1)}([\mathbf{o}\:\:\mathbf{s}_j\:\:\mathbf{r}_i])$ is close to $\Xi^{(t)}([\mathbf{o}\:\:\mathbf{s}_j\:\:\mathbf{r}_i])$, i.e.,
    \begin{align}\label{equ: s2s1diff}
        \left\|\Xi^{(t+1)}([\mathbf{o}\:\:\mathbf{s}_j\:\:\mathbf{r}_i])-\Xi^{(t)}([\mathbf{o}\:\:\mathbf{s}_j\:\:\mathbf{r}_i])\right\|_2 \le 4\left(\exp\left(\mathcal{O}\left(\frac{\eta_2\log(d)}{d}\right)\right)-1\right)= \mathcal{O}\left(\frac{\eta_2\log(d)}{d}\right).
    \end{align}
    Inequality (\ref{equ: s2s1diff}) is derived using the change of the attention score at iteration $t$, which is formally given in a later proof (\ref{equ: s1 update}).
    
	At iteration $t+1$, we have
        \begin{equation}
            \begin{aligned}
                \frac{1}{m}\sum_{k=1}^{m}\langle \mathbf{w}^{(t+1)}_{\mathcal{I}(\mathbf{r}_i),k},\Xi^{(t+1)}([\mathbf{o}\:\:\mathbf{s}_j])\rangle- \frac{1}{m}\sum_{k=1}^{m}\langle \mathbf{w}^{(T_1)}_{\mathcal{I}(\mathbf{r}_i),k}, \Xi^{(T_1)}([\mathbf{o}\:\:\mathbf{s}_j])\rangle = \mathcal{O}\left(\frac{\lambda\eta_2 }{nm}T_2\right)= \mathcal{O}\left(\!\frac{\log(d)}{\lambda \tilde{K}(j)}\!\right),
            \end{aligned}
        \end{equation}
    and
        \begin{equation}
	\begin{aligned}
		\frac{1}{m}\sum_{k=1}^{m}\langle \mathbf{w}^{(t+1)}_{\mathcal{I}(\mathbf{a}_j),k}, \Xi^{(t+1)}([\mathbf{o}\:\:\mathbf{s}_j\:\:\mathbf{r}_i])\rangle - \frac{1}{m}\sum_{k=1}^{m}\langle \mathbf{w}^{(T_1)}_{\mathcal{I}(\mathbf{a}_j),k}, \Xi^{(T_1)}([\mathbf{o}\:\:\mathbf{s}_j\:\:\mathbf{r}_i])\rangle =&\mathcal{O}\left(\frac{\lambda\eta_2 \tilde{K}(j)}{nm}T_2\right)\\
        =& \mathcal{O}\left(\frac{\log(d)}{\lambda}\right).
	\end{aligned}
        \end{equation}
	Therefore, we have
	\begin{align}
		1- \text{logit}_{\mathcal{I}(\mathbf{r}_i)}^{(t+1)}([\mathbf{o}\:\:\mathbf{s}_j\:\:\mathbf{r}_i]) = \Theta(1).
	\end{align}
Similarly, we have 
\begin{align}
	1- \tilde{K}(j) \text{logit}_{\mathcal{I}(\mathbf{r}_i)}^{(t+1)}([\mathbf{o}\:\:\mathbf{s}_j]) = \Theta(1).
\end{align}   
The first statement holds at iteration $t+1$.
    
   \textbf{Proof of the second statement:}
	We have
	\begin{align}
		\frac{\alpha^{(t)}_{2}([\mathbf{o}\:\:\mathbf{s}_j\:\:\mathbf{r}_i])}{\alpha^{(t)}_{3}([\mathbf{o}\:\:\mathbf{s}_j\:\:\mathbf{r}_i])} = \frac{\exp(\tilde{Z}^{(t)}_{\mathcal{I}(\mathbf{s}_j),\mathcal{I}(\mathbf{r}_i)}/\sqrt{d})}{\exp(\tilde{Z}^{(t)}_{\mathcal{I}(\mathbf{r}_i),\mathcal{I}(\mathbf{r}_i)}/\sqrt{d})} = \exp\left(\frac{\tilde{Z}^{(t)}_{\mathcal{I}(\mathbf{s}_j),\mathcal{I}(\mathbf{r}_i)} - \tilde{Z}^{(t)}_{\mathcal{I}(\mathbf{r}_i),\mathcal{I}(\mathbf{r}_i)}}{\sqrt{d}}\right).
	\end{align}
	The terms $\tilde{Z}^{(t)}_{\mathcal{I}(\mathbf{s}_j),\mathcal{I}(\mathbf{r}_i)}$ and $\tilde{Z}^{(t)}_{\mathcal{I}(\mathbf{r}_i),\mathcal{I}(\mathbf{r}_i)}$  satisfy
	\begin{equation}
	\begin{aligned}
		&|\tilde{Z}^{(t+1)}_{\mathcal{I}(\mathbf{s}_j),\mathcal{I}(\mathbf{r}_i)}- \tilde{Z}^{(t)}_{\mathcal{I}(\mathbf{s}_j),\mathcal{I}(\mathbf{r}_i)}|\\
		=& \frac{\eta_2\lambda}{n\sqrt{d}}\underbrace{\alpha_{2}^{(t)}([\mathbf{o}\:\:\mathbf{s}_j\:\:\mathbf{r}_i])}_{\le 1}\underbrace{\left|\left(\mathbf{e}_{\mathcal{I}(\mathbf{a}_j)}-\textbf{logit}^{(t)}([\mathbf{o}\:\:\mathbf{s}_{j}\:\:\mathbf{r}_i])\right)^\top\boldsymbol{\zeta}^{(t)}([\mathbf{o}\:\:\mathbf{s}_j\:\:\mathbf{r}_i])\left(\mathbf{s}_j - \tilde{\Xi}^{(t)}([\mathbf{o}\:\:\mathbf{s}_j\:\:\mathbf{r}_i])\right)\right|}_{=\mathcal{O}(\log(d)/\lambda)}\\
		\overset{(a)}{=}& \mathcal{O}\left(\frac{\eta_2\log(d) }{nd^{1/2}}\right),
	\end{aligned}
\end{equation}
and
\begin{equation}
	\begin{aligned}
		&|\tilde{Z}^{(t+1)}_{\mathcal{I}(\mathbf{r}_i),\mathcal{I}(\mathbf{r}_i)}- \tilde{Z}^{(t)}_{\mathcal{I}(\mathbf{r}_i),\mathcal{I}(\mathbf{r}_i)}|\\
		=& \frac{\eta_2\lambda}{n\sqrt{d}}\sum_{j\in\mathcal{D}_i}\underbrace{\alpha_{3}^{(t)}([\mathbf{o}\:\:\mathbf{s}_j\:\:\mathbf{r}_i])}_{\le 1}\underbrace{\left|\left(\mathbf{e}_{\mathcal{I}(\mathbf{a}_j)}-\textbf{logit}^{(t)}([\mathbf{o}\:\:\mathbf{s}_{j}\:\:\mathbf{r}_i])\right)^\top\boldsymbol{\zeta}^{(t)}([\mathbf{o}\:\:\mathbf{s}_j\:\:\mathbf{r}_i])\left(\mathbf{r}_i - \tilde{\Xi}^{(t)}([\mathbf{o}\:\:\mathbf{s}_j\:\:\mathbf{r}_i])\right)\right|}_{=\mathcal{O}(\log(d)/\lambda)}\\
		\overset{(b)}{=}& \mathcal{O}\left(\frac{\eta_2\log(d)}{d^{1/2}}\right).
	\end{aligned}
\end{equation}
where $(a)$ follows the conclusion from the third statement 

$\alpha_1^{(t)}([\mathbf{o}\;\;\mathbf{s}_j\;\;\mathbf{r}_i]) \le \frac{2}{(3d^{2})}$ and the conclusion from the first statement which implies that 
$\left|\left(\mathbf{e}_{\mathcal{I}(\mathbf{a}_j)}-\textbf{logit}^{(t)}([\mathbf{o}\:\:\mathbf{s}_{j}\:\:\mathbf{r}_i])\right)^\top\boldsymbol{\zeta}^{(t)}([\mathbf{o}\:\:\mathbf{s}_j\:\:\mathbf{r}_i])\right.$ $\left.\left(\mathbf{s}_j - \tilde{\Xi}^{(t)}([\mathbf{o}\:\:\mathbf{s}_j\:\:\mathbf{r}_i])\right)\right| = \mathcal{O}(\log(d)/\lambda)$.
$(b)$ follows the conclusion from the third statement $\alpha_1^{(t)}([\mathbf{o}\;\;\mathbf{s}_j\;\;\mathbf{r}_i]) \le \frac{2}{(3d^{2})}$ and the conclusion from the first statement which implies that 
$\left|\left(\mathbf{e}_{\mathcal{I}(\mathbf{a}_j)}-\textbf{logit}^{(t)}([\mathbf{o}\:\:\mathbf{s}_{j}\:\:\mathbf{r}_i])\right)^\top\boldsymbol{\zeta}^{(t)}([\mathbf{o}\:\:\mathbf{s}_j\:\:\mathbf{r}_i])\right.$ $\left.\left(\mathbf{r}_i - \tilde{\Xi}^{(t)}([\mathbf{o}\:\:\mathbf{s}_j\:\:\mathbf{r}_i])\right)\right| = \mathcal{O}(\log(d)/\lambda)$.

Therefore, we have
\begin{align}\label{equ: s1 update}
	\frac{\alpha^{(t+1)}_{2}([\mathbf{o}\:\:\mathbf{s}_j\:\:\mathbf{r}_i])}{\alpha^{(t+1)}_{3}([\mathbf{o}\:\:\mathbf{s}_j\:\:\mathbf{r}_i])} 	\frac{\alpha^{(t)}_{3}([\mathbf{o}\:\:\mathbf{s}_j\:\:\mathbf{r}_i])}{\alpha^{(t)}_{2}([\mathbf{o}\:\:\mathbf{s}_j\:\:\mathbf{r}_i])} = \exp\left(\mathcal{O}\left(\frac{\eta_2\log(d)}{d}\right)\right).
\end{align}
Recursively using (\ref{equ: s1 update}) from $T_1$ to $t$, we have
\begin{equation}
\begin{aligned}
	\frac{\alpha^{(t+1)}_{2}([\mathbf{o}\:\:\mathbf{s}_j\:\:\mathbf{r}_i])}{\alpha^{(t+1)}_{3}([\mathbf{o}\:\:\mathbf{s}_j\:\:\mathbf{r}_i])} 	\frac{\alpha^{(T_1)}_{3}([\mathbf{o}\:\:\mathbf{s}_j\:\:\mathbf{r}_i])}{\alpha^{(T_1)}_{2}([\mathbf{o}\:\:\mathbf{s}_j\:\:\mathbf{r}_i])}
    =& \exp\left(\mathcal{O}\left((t+1-T_1)\frac{\eta_2\log(d)}{d}\right)\right)\\
	=& \exp\left(\mathcal{O}\left((T_2-T_1)\frac{\eta_2\log(d)}{d}\right)\right)\\
	=& \exp\left(\mathcal{O}\left(\frac{\eta_2\log(d)}{d}\frac{nm\log(d)}{\lambda^2\eta_2 K}\right)\right)\\
	=& \exp\left(\mathcal{O}\left(\frac{nm\log^2(d)}{\lambda^2 dK}\right)\right).
\end{aligned}
\end{equation}

Under Condition \ref{condition: condition}, by $0.8\le\alpha^{(T_1)}_{2}([\mathbf{o}\:\:\mathbf{s}_j\:\:\mathbf{r}_i])/\alpha^{(T_1)}_{3}([\mathbf{o}\:\:\mathbf{s}_j\:\:\mathbf{r}_i]) \le 1.2$, we have
\begin{align}
	\frac{\alpha^{(t+1)}_{2}([\mathbf{o}\:\:\mathbf{s}_j\:\:\mathbf{r}_i])}{\alpha^{(t+1)}_{3}([\mathbf{o}\:\:\mathbf{s}_j\:\:\mathbf{r}_i])} \le 1.6.
\end{align}
Similarly, we have
\begin{align}
	\frac{\alpha^{(t+1)}_{2}([\mathbf{o}\:\:\mathbf{s}_j\:\:\mathbf{r}_i])}{\alpha^{(t+1)}_{3}([\mathbf{o}\:\:\mathbf{s}_j\:\:\mathbf{r}_i])} \ge 0.7.
\end{align}
Hence, the second statement holds at iteration $t+1$.

\textbf{Proof of the third statement:}
First, we prove the first conclusion in the third statement.

For $T_1\le t \le T_2$, the increment of the attention matrix satisfies
\begin{equation}
	\begin{aligned}
		&|\tilde{Z}^{(t+1)}_{\mathcal{I}(\mathbf{o}),\mathcal{I}(\mathbf{s}_j)}- \tilde{Z}^{(t)}_{\mathcal{I}(\mathbf{o}),\mathcal{I}(\mathbf{s}_j)}|\\
		\le& \frac{\eta_2\lambda}{n\sqrt{d}}\underbrace{\alpha_{1}^{(t)}([\mathbf{o},\mathbf{s}_j])}_{\le 1/d^{2}}\sum_{i\in\mathcal{B}_j}\underbrace{\left|\left(\mathbf{e}_{\mathcal{I}(\mathbf{r}_i)}-\textbf{logit}^{(t)}([\mathbf{o}\:\:\mathbf{s}_{j}])\right)^\top\boldsymbol{\zeta}^{(t)}([\mathbf{o}\:\:\mathbf{s}_j])\left(\mathbf{o} - \tilde{\Xi}^{(t)}([\mathbf{o}\:\:\mathbf{s}_j])\right)\right|
        }_{=\mathcal{O}(d\log(d)/\lambda)}\cdot\Theta(\sqrt{d})\\
		\overset{(a)}{=}& \mathcal{O}\left(\frac{\eta_2K\log(d)}{nd}\right), 
	\end{aligned}
\end{equation}
and 
	\begin{equation}
	\begin{aligned}
		&|\tilde{Z}^{(t+1)}_{\mathcal{I}(\mathbf{s}_j),\mathcal{I}(\mathbf{s}_j)}- \tilde{Z}^{(t)}_{\mathcal{I}(\mathbf{s}_j),\mathcal{I}(\mathbf{s}_j)}|\\
		=& \frac{\eta_2\lambda}{n\sqrt{d}}\underbrace{\alpha_{2}^{(t)}([\mathbf{o},\mathbf{s}_j])}_{=\Theta(1)}\sum_{i\in\mathcal{B}_j}\underbrace{\left|\left(\mathbf{e}_{\mathcal{I}(\mathbf{r}_i)}-\textbf{logit}^{(t)}([\mathbf{o}\:\:\mathbf{s}_{j}])\right)^\top\boldsymbol{\zeta}^{(t)}([\mathbf{o}\:\:\mathbf{s}_j])\left(\mathbf{s}_j - \tilde{\Xi}^{(t)}([\mathbf{o}\:\:\mathbf{s}_j])\right)\right|}_{=\mathcal{O}(\log(d)/\lambda)}\\
		\overset{(b)}{=}& \mathcal{O}\left(\frac{\eta_2K\log(d)}{nd^{1/2}}\right),
	\end{aligned}
\end{equation}
where $(a)$ follows the conclusion from the third statement $\alpha_1^{(t)}([\mathbf{o}\;\;\mathbf{s}_j]) \le \frac{2}{(3d^{2})}$ and the conclusions from the first and the fourth statements which imply that 
$\left|\left(\mathbf{e}_{\mathcal{I}(\mathbf{r}_i)}-\textbf{logit}^{(t)}([\mathbf{o}\:\:\mathbf{s}_{j}])\right)^\top\boldsymbol{\zeta}^{(t)}([\mathbf{o}\:\:\mathbf{s}_j])\right.$ $\left.\left(\mathbf{o} - \tilde{\Xi}^{(t)}([\mathbf{o}\:\:\mathbf{s}_j])\right)\right| = \mathcal{O}(d\log(d)/\lambda)$.
$(b)$ follows the conclusion from the third statement $\alpha_1^{(t)}([\mathbf{o}\;\;\mathbf{s}_j]) \le \frac{2}{(3d^{2})}$ and the conclusion from the first statement which implies that 

$\left|\left(\mathbf{e}_{\mathcal{I}(\mathbf{r}_i)}-\textbf{logit}^{(t)}([\mathbf{o}\:\:\mathbf{s}_{j}])\right)^\top\boldsymbol{\zeta}^{(t)}([\mathbf{o}\:\:\mathbf{s}_j])\left(\mathbf{s}_j - \tilde{\Xi}^{(t)}([\mathbf{o}\:\:\mathbf{s}_j])\right)\right| = \mathcal{O}(\log(d)/\lambda)$.

    Therefore, we have 
	\begin{align}
		\frac{\alpha^{(t+1)}_{1}([\mathbf{o}\:\:\mathbf{s}_j])}{\alpha^{(t+1)}_{2}([\mathbf{o}\:\:\mathbf{s}_j])}\frac{\alpha^{(t)}_{2}([\mathbf{o}\:\:\mathbf{s}_j])}{\alpha^{(t)}_{1}([\mathbf{o}\:\:\mathbf{s}_j])} =\exp\left(\mathcal{O}\left(\frac{ \eta_2K \log(d)}{nd}\right)\right).
	\end{align}
    As a result, we have
    \begin{align}
		\frac{\alpha^{(t+1)}_{1}([\mathbf{o}\:\:\mathbf{s}_j])}{\alpha^{(t+1)}_{2}([\mathbf{o}\:\:\mathbf{s}_j])}\frac{\alpha^{(T_1)}_{2}([\mathbf{o}\:\:\mathbf{s}_j])}{\alpha^{(T_1)}_{1}([\mathbf{o}\:\:\mathbf{s}_j])} =\exp\left(\mathcal{O}\left(\frac{ \eta_2K \log(d)}{nd}(T_2-T_1)\right)\right) = \exp\left(\mathcal{O}\left(\frac{m\log^2(d)
        }{d\lambda^2}\right)\right).
	\end{align}
Under Condition \ref{condition: condition}, using $ \alpha^{(T_1)}_{1}([\mathbf{o}\:\:\mathbf{s}_j])/\alpha^{(T_1)}_{2}([\mathbf{o}\:\:\mathbf{s}_j]) \le \frac{1}{(1.8d^{2})}$, we have
\begin{align}
    \frac{\alpha^{(t)}_{1}([\mathbf{o}\:\:\mathbf{s}_j])}{\alpha^{(t)}_{2}([\mathbf{o}\:\:\mathbf{s}_j])} \le \frac{2}{3d^{2}}.
\end{align}

Next, we prove the second conclusion in the third statement.
For $T_1\le t\le T_2$, we have
\begin{equation}
	\begin{aligned}
		&|\tilde{Z}^{(t+1)}_{\mathcal{I}(\mathbf{o}),\mathcal{I}(\mathbf{r}_i)}- \tilde{Z}^{(t)}_{\mathcal{I}(\mathbf{o}),\mathcal{I}(\mathbf{r}_i)}|\\
		=& \frac{\eta_2\lambda}{n\sqrt{d}}\sum_{j\in\mathcal{D}_i}\underbrace{\alpha_{1}^{(t)}([\mathbf{o}\:\:\mathbf{s}_{j}\:\:\mathbf{r}_i])}_{\le 1/d^{2}}\\
        &\cdot\underbrace{\left|\left(\mathbf{e}_{\mathcal{I}(\mathbf{a}_j)}-\textbf{logit}^{(t)}([\mathbf{o}\:\:\mathbf{s}_{j}\:\:\mathbf{r}_i])\right)^\top\boldsymbol{\zeta}^{(t)}([\mathbf{o}\:\:\mathbf{s}_j\:\:\mathbf{r}_i])\cdot\left(\mathbf{o} - \tilde{\Xi}^{(t)}([\mathbf{o}\:\:\mathbf{s}_j\:\:\mathbf{r}_i])\right)\right|}_{\Theta(d\log(d)/\lambda)}\cdot\Theta(\sqrt{d})\\
		\overset{(a)}{=}& \mathcal{O}\left(\frac{\eta_2\log(d)}{d}\right),
	\end{aligned}
\end{equation}
where $(a)$ follows the conclusion from the third statement $\alpha_1^{(t)}([\mathbf{o}\;\;\mathbf{s}_j\;\;\mathbf{r}_i]) \le \frac{2}{(3d^{2})}$ and the conclusions from the first and the fourth statements which imply that 
$\left|\left(\mathbf{e}_{\mathcal{I}(\mathbf{a}_j)}-\textbf{logit}^{(t)}([\mathbf{o}\:\:\mathbf{s}_{j}\:\:\mathbf{r}_i])\right)^\top\right.$ $\left.\boldsymbol{\zeta}^{(t)}([\mathbf{o}\:\:\mathbf{s}_j\:\:\mathbf{r}_i])\cdot\left(\mathbf{o} - \tilde{\Xi}^{(t)}([\mathbf{o}\:\:\mathbf{s}_j\:\:\mathbf{r}_i])\right)\right| = \mathcal{O}(d\log(d)/\lambda)$.

Therefore, we have 
	\begin{align}
		\frac{\alpha^{(t+1)}_{1}([\mathbf{o}\:\:\mathbf{s}_j\:\:\mathbf{r}_i])}{\alpha^{(t+1)}_{2}([\mathbf{o}\:\:\mathbf{s}_j\:\:\mathbf{r}_i])}\frac{\alpha^{(t)}_{2}([\mathbf{o}\:\:\mathbf{s}_j\:\:\mathbf{r}_i])}{\alpha^{(t)}_{1}([\mathbf{o}\:\:\mathbf{s}_j\:\:\mathbf{r}_i])} =\exp\left(\mathcal{O}\left(\frac{ \eta_2 \log(d)}{d}\right)\right).
	\end{align}
    As a result, we have
    \begin{equation}
    \begin{aligned}
		\frac{\alpha^{(t+1)}_{1}([\mathbf{o}\:\:\mathbf{s}_j\:\:\mathbf{r}_i])}{\alpha^{(t+1)}_{2}([\mathbf{o}\:\:\mathbf{s}_j\:\:\mathbf{r}_i])}\frac{\alpha^{(T_1)}_{2}([\mathbf{o}\:\:\mathbf{s}_j\:\:\mathbf{r}_i])}{\alpha^{(T_1)}_{1}([\mathbf{o}\:\:\mathbf{s}_j\:\:\mathbf{r}_i])} =&\exp\left(\mathcal{O}\left(\frac{ \eta_2\log(d)}{d}(T_2-T_1)\right)\right)\\ =&\exp\left(\mathcal{O}\left(\frac{mn\log^2(d)
        }{d\lambda^2K}\right)\right).
    \end{aligned}
    \end{equation}
Under Condition \ref{condition: condition}, using $ \alpha^{(T_1)}_{1}([\mathbf{o}\:\:\mathbf{s}_j\;\;\mathbf{r}_i])/\alpha^{(T_1)}_{2}([\mathbf{o}\:\:\mathbf{s}_j\;\;\mathbf{r}_i]) \le \frac{0.625}{d^{2}}$, we have
\begin{align}
    \frac{\alpha^{(t+1)}_{1}([\mathbf{o}\:\:\mathbf{s}_j\:\:\mathbf{r}_i])}{\alpha^{(t+1)}_{2}([\mathbf{o}\:\:\mathbf{s}_j\:\:\mathbf{r}_i])} \le \frac{0.8}{d^{2}}.
\end{align}
Therefore, we have
\begin{align}
    \alpha^{(t+1)}_{1}([\mathbf{o}\:\:\mathbf{s}_j]) \le \frac{2}{3d^{2}}, \alpha^{(t+1)}_{1}([\mathbf{o}\:\:\mathbf{s}_j\:\:\mathbf{r}_i]) \le \frac{2}{3d^{2}}.
\end{align}

The third statement holds for iteration $t+1$.

\textbf{Proof of the fourth statement:}
For $t'\in[T_1,T_2]$, by the gradient form in Lemma \ref{lemma: gradient}, we have
    \begin{equation}\label{equ: s1s3 4}
	\begin{aligned}
		&\frac{1}{m}\sum_{k=1}^{m}\langle \mathbf{w}_{\mathcal{I}(\mathbf{d}),k}^{(t'+1)}, 2\mathbf{o}\rangle \!-\! \frac{1}{m}\sum_{k=1}^{m}\langle \mathbf{w}_{\mathcal{I}(\mathbf{d}),k}^{(t')}, 2\mathbf{o}\rangle\\
        \ge& 
        -\! \frac{2\eta_2 \tilde{K}(j)}{m^2n}\Theta(1)\sum_{k=1}^m\sum_{j=1}^{N}\text{logit}_{\mathcal{I}(\mathbf{d})}^{(t')}([\mathbf{o}\:\:\mathbf{s}_j])\alpha_{1}^{(t')}([\mathbf{o}\:\:\mathbf{s}_j])\sigma'(\langle \mathbf{w}_{\mathcal{I}(\mathbf{d}),k}^{(t')}, \Xi^{(t')}([\mathbf{o}\:\:\mathbf{s}_j])\rangle)\lambda\left\|\mathbf{o}\right\|_2^2\\
		&-\frac{\eta_2}{m^2n}\Theta(1)\sum_{k=1}^m\sum_{j=1}^{N}\sum_{i\in\mathcal{B}_j} \text{logit}^{(t')}_{\mathcal{I}(\mathbf{d})}([\mathbf{o}\:\:\mathbf{s}_{j}\:\:\mathbf{r}_i])\alpha_{1}^{(t')}([\mathbf{o}\:\:\mathbf{s}_j\:\:\mathbf{r}_i])\sigma'(\langle \mathbf{w}_{\mathcal{I}(\mathbf{d}),k}^{(t')}, \Xi^{(t')}([\mathbf{o}\:\:\mathbf{s}_j\:\:\mathbf{r}_i])\rangle)\lambda\left\|\mathbf{o}\right\|_2^2\\
        \overset{(a)}{=}& -\mathcal{O}\left(\frac{\eta_2\lambda }{md^{3/4}}\right),
	\end{aligned}
    \end{equation}
where $(a)$ is obtained by the third statement, $\text{logit}_{\mathcal{I}(\mathbf{d})}^{(t')}([\mathbf{o}\:\:\mathbf{s}_j])\le 1$, $\text{logit}^{(t')}_{\mathcal{I}(\mathbf{d})}([\mathbf{o}\:\:\mathbf{s}_{j}\:\:\mathbf{r}_i])\le 1$, and $\sigma'(\cdot) \le 1$.

As a result, we have
\begin{equation}
\begin{aligned}
    \frac{1}{m}\sum_{k=1}^{m}\langle \mathbf{w}_{\mathcal{I}(\mathbf{d}),k}^{(t+1)}, 2\mathbf{o}\rangle - \frac{1}{m}\sum_{k=1}^{m}\langle \mathbf{w}_{\mathcal{I}(\mathbf{d}),k}^{(T_1)}, 2\mathbf{o}\rangle \ge A,
\end{aligned}
\end{equation}
where $A \le0$ and
\begin{equation}
\begin{aligned}
    |A| = \mathcal{O}\left(\frac{\eta_2\lambda }{md^{3/4}}(T_2-T_1)\right)
    = \mathcal{O}\left(\frac{\eta_2\lambda }{md^{3/4}}\frac{nm\log(d)}{\lambda^2\eta_2 K}\right)
    = \mathcal{O}\left(\frac{n\log(d)}{\lambda Kd^{3/4}}\right).
\end{aligned}
\end{equation}
As $\oldfrac{1}{m}\sum_{k=1}^{m}\langle \mathbf{w}_{\mathcal{I}(\mathbf{d}),k}^{(T_1)}, 2\mathbf{o}\rangle = \Theta(\frac{\eta_1\lambda d}{m})$, we have
\begin{align}
    \frac{1}{m}\sum_{k=1}^{m}\langle \mathbf{w}_{\mathcal{I}(\mathbf{d}),k}^{(t+1)}, 2\mathbf{o}\rangle = \Omega\left(\frac{\log(d)}{\lambda}\right).
\end{align}

Furthermore, for $t'\in[T_1,T_2]$, by the gradient form in Lemma \ref{lemma: gradient}, we also have
    \begin{equation}\label{equ: s1s3 5}
	\begin{aligned}
		&\frac{1}{m}\sum_{k=1}^{m}\langle \mathbf{w}_{\mathcal{I}(\mathbf{d}),k}^{(t'+1)}, 2\mathbf{o}\rangle \!-\! \frac{1}{m}\sum_{k=1}^{m}\langle \mathbf{w}_{\mathcal{I}(\mathbf{d}),k}^{(t')}, 2\mathbf{o}\rangle\\
        \le& \Theta(1)\cdot\frac{\eta_2}{n}\frac{n}{3m^2}\sum_{k=1}^m\sum_{i=1}^{n}(1-\text{logit}^{(t')}_{\mathcal{I}(\mathbf{d})}(2\mathbf{o}))\sigma'(\langle \mathbf{w}_{\mathcal{I}(\mathbf{d}),k}^{(t')}, 2\mathbf{o}\rangle) \lambda\left\|2\mathbf{o}\right\|_2^2\\
        \overset{(a)}{=}& \mathcal{O}\left(\frac{\eta_2\lambda }{md}\right),
	\end{aligned}
    \end{equation}
where $(a)$ holds because $1-\text{logit}^{(t')}_{\mathcal{I}(\mathbf{d})}(2\mathbf{o}) \le d^{-2}$ as $\frac{1}{m}\sum_{k=1}^{m}\langle \mathbf{w}_{\mathcal{I}(\mathbf{d}),k}^{(t')}, 2\mathbf{o}\rangle =\Omega(\log(d)/\lambda)$.

Hence, we have
\begin{align}
    \frac{1}{m}\sum_{k=1}^{m}\langle \mathbf{w}_{\mathcal{I}(\mathbf{d}),k}^{(t+1)}, 2\mathbf{o}\rangle = \mathcal{O}\left(\frac{d\log(d)}{\lambda}\right).
\end{align}
Thus, the fourth statement holds at iteration $t+1$.

    \textbf{Proof of the fifth statement:}
    After Stage 1 with $t\ge T_1$, for $l\in[m]$ satisfying $\langle \mathbf{w}^{(t)}_{\mathcal{I}(\mathbf{r}_k),l},\mathbf{s}_j \rangle > 0$, the update of $\langle \mathbf{w}^{(t)}_{\mathcal{I}(\mathbf{r}_k),l},\mathbf{s}_j \rangle$ satisfies
    \begin{equation}
	\begin{aligned}
		&\langle \mathbf{w}^{(t+1)}_{\mathcal{I}(\mathbf{r}_k),l},\mathbf{s}_j \rangle\\
        \!\ge& \langle \mathbf{w}^{(t)}_{\mathcal{I}(\mathbf{r}_k),l},\mathbf{s}_j \rangle \!+\! \lambda\frac{\eta_2}{nm}(1-\tilde{K}(j)\text{logit}^{(t)}_{\mathcal{I}(\mathbf{r}_k)}([\mathbf{o}\:\:\mathbf{s}_j])) - \lambda\frac{\eta_2}{nm}\sum_{i\in\mathcal{B}_j}\text{logit}^{(t)}_{\mathcal{I}(\mathbf{r}_k)}([\mathbf{o}\:\:\mathbf{s}_j\:\:\mathbf{r}_i]),
	\end{aligned}
    \end{equation}
	by the fact that $\langle\Xi^{(t)}([\mathbf{o}\:\:\mathbf{s}_j]),\mathbf{s}_j\rangle \ge \left\|\mathbf{s}_j\right\|_2^2=1$ and $\langle\Xi^{(t)}([\mathbf{o}\:\:\mathbf{s}_j\:\:\mathbf{r}_i]),\mathbf{s}_j\rangle \le \left\|\mathbf{s}_j\right\|_2^2=1$.
	Similarly, for all $l\in[m]$, the update of $\langle \mathbf{w}^{(t)}_{\mathcal{I}(\mathbf{r}_k),l},\mathbf{o} \rangle$ satisfies
\begin{equation}
	\begin{aligned}
		\langle \mathbf{w}^{(t+1)}_{\mathcal{I}(\mathbf{r}_k),l},\mathbf{o} \rangle
        \!\ge& \langle \mathbf{w}^{(t)}_{\mathcal{I}(\mathbf{r}_k),l},\mathbf{o} \rangle  - \mathcal{O}\left(\frac{\lambda\eta_2}{nmd}\right)\sum_{i\in\mathcal{B}_j}\text{logit}^{(t)}_{\mathcal{I}(\mathbf{r}_k)}([\mathbf{o}\:\:\mathbf{s}_j\:\:\mathbf{r}_i])\\
 &- \mathcal{O}\left(\frac{\lambda\eta_2}{nmd}\right)\sum_{j\in\mathcal{D}_k}\text{logit}_{\mathcal{I}(\mathbf{r}_k)}^{(t)}([\mathbf{o}\;\;\mathbf{s}_j])
 -\mathcal{O}\left(\frac{\lambda\eta_2}{m}\right)\text{logit}_{\mathcal{I}(\mathbf{r}_k)}^{(t)}(2\mathbf{o})\\
 =& \langle \mathbf{w}^{(t)}_{\mathcal{I}(\mathbf{r}_k),l},\mathbf{o} \rangle - \mathcal{O}\left(\frac{\lambda\eta_2K}{nmd}\right)-\mathcal{O}\left(\frac{\lambda\eta_2}{md}\right) - \mathcal{O}\left(\frac{\lambda\eta_2}{md^{3/2}}\right),
	\end{aligned}
    \end{equation}
As a result, by the third statement, we have
 \begin{equation}
	\begin{aligned}
		&\langle \mathbf{w}^{(t+1)}_{\mathcal{I}(\mathbf{r}_k),l},\Xi^{(t+1)}([\mathbf{o}\;\;\mathbf{s}_j]) \rangle -\langle \mathbf{w}^{(t)}_{\mathcal{I}(\mathbf{r}_k),l},\Xi^{(t)}([\mathbf{o}\;\;\mathbf{s}_j]) \rangle\\
        \!\ge&  0.9\lambda\frac{\eta_2}{nm}(1-\tilde{K}(j)\text{logit}^{(t)}_{\mathcal{I}(\mathbf{r}_k)}([\mathbf{o}\:\:\mathbf{s}_j])) - 1.1\lambda\frac{\eta_2}{nm}\sum_{i\in\mathcal{B}_j}\text{logit}^{(t)}_{\mathcal{I}(\mathbf{r}_k)}([\mathbf{o}\:\:\mathbf{s}_j\:\:\mathbf{r}_i])\\
        & - \mathcal{O}\left(\frac{\lambda\eta_2}{md^{3/2}}\right)- \mathcal{O}\left(\frac{\eta_2\log^2(d)}{\lambda d}\right).
	\end{aligned}
    \end{equation}
    Case 1: If $\oldfrac{1}{m}\sum_{l=1}^m\sigma(\langle \mathbf{w}^{(t)}_{\mathcal{I}(\mathbf{r}_k),l},\Xi^{(t)}([\mathbf{o}\;\;\mathbf{s}_j]) \rangle) \le 2\log(d)/(3\lambda)$, for all $k\in\mathcal{B}_j$, we have
    \begin{align}
        1-\tilde{K}(j)\text{logit}^{(t)}_{\mathcal{I}(\mathbf{r}_k)}([\mathbf{o}\:\:\mathbf{s}_j]) = \Theta(1),\text{logit}^{(t)}_{\mathcal{I}(\mathbf{r}_k)}([\mathbf{o}\:\:\mathbf{s}_j\:\:\mathbf{r}_i]) \le \frac{1}{\sqrt{d}},
    \end{align}
by Lemma \ref{lemma: mag_n}.
As a result, $\oldfrac{1}{m}\sum_{l=1}^m\sigma(\langle \mathbf{w}^{(t+1)}_{\mathcal{I}(\mathbf{r}_k),l},\Xi^{(t+1)}([\mathbf{o}\;\;\mathbf{s}_j]) \rangle) -\oldfrac{1}{m}\sum_{l=1}^m\sigma(\langle \mathbf{w}^{(t)}_{\mathcal{I}(\mathbf{r}_k),l},\Xi^{(t)}([\mathbf{o}\;\;\mathbf{s}_j]) \rangle) = \Theta(1)$ holds.

Case 2: If $\oldfrac{1}{m}\sum_{l=1}^m\sigma(\langle \mathbf{w}^{(t)}_{\mathcal{I}(\mathbf{r}_k),l},\Xi^{(t)}([\mathbf{o}\;\;\mathbf{s}_j]) \rangle) > 2\log(d)/(3\lambda)$, for all $l\in[m]$, we have
\begin{align}
    \langle \mathbf{w}^{(t+1)}_{\mathcal{I}(\mathbf{r}_k),l},\Xi^{(t+1)}([\mathbf{o}\;\;\mathbf{s}_j]) \rangle -\langle \mathbf{w}^{(t)}_{\mathcal{I}(\mathbf{r}_k),l},\Xi^{(t)}([\mathbf{o}\;\;\mathbf{s}_j]) \rangle \ge -\mathcal{O}\left(\frac{\eta_2\lambda}{nm}\right).
\end{align}

Then, we have
\begin{align}
    \frac{1}{m}\sum_{l=1}^m\sigma(\langle \mathbf{w}^{(t+1)}_{\mathcal{I}(\mathbf{r}_k),l},\Xi^{(t+1)}([\mathbf{o}\;\;\mathbf{s}_j]) \rangle) \ge \frac{\log(d)}{2\lambda}.
\end{align}

Therefore, the fifth statement holds for iteration $t+1$.
    
    \textbf{Proof of the sixth statement:}
We first prove that during $t\in[T_1, T_2-1]$, we have
    \begin{align}\label{equ: s2s61}
        &\frac{1}{m}\sum_{l=1}^m\sigma(\langle \mathbf{w}^{(t)}_{\mathcal{I}(\mathbf{a}_j),l},\mathbf{s}_j \rangle)\}
        \ge  0.
    \end{align}
We prove (\ref{equ: s2s61}) by considering the following cases.

Case 1: $\oldfrac{1}{m}\sum_{l=1}^m\sigma(\langle \mathbf{w}^{(t)}_{\mathcal{I}(\mathbf{a}_j),l},\mathbf{s}_j \rangle)\}\le \frac{\log(d)}{(3\lambda)}$.
For all $l\in\cup_{i\in\mathcal{B}_j}\mathcal{S}_{\mathbf{r},j,i}^{(t)}$, we have
\begin{equation}
\begin{aligned}
    \langle \mathbf{w}^{(t+1)}_{\mathcal{I}(\mathbf{a}_j),l},\mathbf{s}_j \rangle \ge& \langle \mathbf{w}^{(t)}_{\mathcal{I}(\mathbf{a}_j),l},\mathbf{s}_j \rangle + \frac{\lambda\eta_2}{nm}\sum_{i\in\mathcal{B}_j}\mathbb{I}(l\in\mathcal{S}_{\mathbf{r},j,i}^{(t)})(1-\text{logit}^{(t)}_{\mathcal{I}(\mathbf{a}_j)}([\mathbf{o}\;\;\mathbf{s}_j\;\;\mathbf{r}_i])) - \frac{\lambda\eta_2}{mn}\tilde{K}(j)\text{logit}^{(t)}([\mathbf{o}\;\;\mathbf{s}_j])\\
    \ge& \langle \mathbf{w}^{(t)}_{\mathcal{I}(\mathbf{a}_j),l},\mathbf{s}_j \rangle + \Theta\left(\frac{\eta_2\lambda}{nm}\right)-\mathcal{O}\left(\frac{\eta_2\lambda}{nm\sqrt{d}}\right)\\
    \ge& \langle \mathbf{w}^{(t)}_{\mathcal{I}(\mathbf{a}_j),l},\mathbf{s}_j \rangle.
\end{aligned}
\end{equation}

Case 2: $\oldfrac{1}{m}\sum_{l=1}^m\sigma(\langle \mathbf{w}^{(t)}_{\mathcal{I}(\mathbf{a}_j),l},\mathbf{s}_j \rangle)\}> \frac{\log(d)}{(3\lambda)}$.
We have
\begin{align}
    \langle \mathbf{w}^{(t+1)}_{\mathcal{I}(\mathbf{a}_j),l},\mathbf{s}_j \rangle \ge& \langle \mathbf{w}^{(t)}_{\mathcal{I}(\mathbf{a}_j),l},\mathbf{s}_j \rangle - \mathcal{O}\left(\frac{\eta_2\lambda K}{mn}\right)
    \ge \frac{\log(d)}{4\lambda}. 
\end{align}
As a result, we have
\begin{align}\label{equ: s2s62}
    \frac{1}{m}\sum_{l=1}^m\sigma(\langle \mathbf{w}^{(t)}_{\mathcal{I}(\mathbf{a}_j),l},\mathbf{s}_j \rangle) \ge 0.
\end{align}

Based on (\ref{equ: s2s62}), we now prove the conclusion in the sixth statement.
We have
\begin{equation}
	\begin{aligned}
		&\langle \mathbf{w}^{(t+1)}_{\mathcal{I}(\mathbf{a}_j),l},\Xi^{(t+1)}([\mathbf{o}\;\;\mathbf{s}_j\;\;\mathbf{r}_i]) \rangle
		\ge \langle \mathbf{w}^{(t)}_{\mathcal{I}(\mathbf{a}_j),l},\Xi^{(t)}([\mathbf{o}\;\;\mathbf{s}_j\;\;\mathbf{r}_i])\rangle\\
		&-1.1\lambda \frac{\eta^{(t)} \tilde{K}(j)}{nm}\text{logit}_{\mathcal{I}(\mathbf{a}_j)}^{(t)}([\mathbf{o}\:\:\mathbf{s}_j])\!+\!\lambda\frac{\Theta(1)\cdot\eta^{(t)}}{nm}\sum_{k\in \mathcal{B}_j}(1\!-\!\text{logit}^{(t)}_{\mathcal{I}(\mathbf{a}_j)}([\mathbf{o}\:\:\mathbf{s}_j\:\:\mathbf{r}_k]))\\
        &-2\lambda\sum_{k\neq j,k\in\mathcal{D}_i}\frac{\eta^{(t)}}{nm}\text{logit}^{(t)}_{\mathcal{I}(\mathbf{a}_j)}([\mathbf{o}\:\:\mathbf{s}_k\:\:\mathbf{r}_i])-\mathcal{O}\left(\frac{\lambda\eta^{(t)}}{md^{3/2}}\right) - \mathcal{O}\left(\frac{\eta^{(t)}\log^2(d)}{\lambda d}\right),
	\end{aligned}
\end{equation}

We consider two cases.

Case 1: $\oldfrac{1}{m}\sum_{l=1}^m\sigma\left(\langle \mathbf{w}_{\mathcal{I}(\mathbf{a}_j),l}^{(t)},\Xi^{(t)}([\mathbf{o}\;\;\mathbf{s}_j\;\;\mathbf{r}_i])\rangle\right) \le \frac{\log(d)}{6\lambda}$.
By (\ref{equ: s2s62}), we also have 
\begin{align}
    \frac{1}{m}\sum_{l=1}^m\sigma\left(\langle \mathbf{w}_{\mathcal{I}(\mathbf{a}_j),l}^{(t)},\mathbf{s}_j\rangle\right) \le \frac{\log(d)}{2\lambda}, \frac{1}{m}\sum_{l=1}^m\sigma\left(\langle \mathbf{w}_{\mathcal{I}(\mathbf{a}_j),l}^{(t)},\mathbf{r}_i\rangle\right) \le \frac{\log(d)}{6\lambda}.
\end{align}
A direct result is
\begin{equation}
\begin{aligned}
    \langle \mathbf{w}^{(t+1)}_{\mathcal{I}(\mathbf{a}_j),l},\Xi^{(t+1)}([\mathbf{o}\;\;\mathbf{s}_j\;\;\mathbf{r}_i]) \rangle
		\ge& \langle \mathbf{w}^{(t)}_{\mathcal{I}(\mathbf{a}_j),l},\Xi^{(t)}([\mathbf{o}\;\;\mathbf{s}_j\;\;\mathbf{r}_i])\rangle + \Theta\left(\frac{\lambda\eta_2}{mn}\right)-\mathcal{O}\left(\frac{\eta_2\lambda}{nm\sqrt{d}}\right)\\
        \ge& \langle \mathbf{w}^{(t)}_{\mathcal{I}(\mathbf{a}_j),l},\Xi^{(t)}([\mathbf{o}\;\;\mathbf{s}_j\;\;\mathbf{r}_i])\rangle.
\end{aligned}
\end{equation}

Case 2: $\oldfrac{1}{m}\sum_{l=1}^m\sigma\left(\langle \mathbf{w}_{\mathcal{I}(\mathbf{a}_j),l}^{(t)},\Xi^{(t)}([\mathbf{o}\;\;\mathbf{s}_j\;\;\mathbf{r}_i])\rangle\right) \ge \frac{\log(d)}{6\lambda}$.
In this case, we have
\begin{equation}
\begin{aligned}
    \langle \mathbf{w}^{(t+1)}_{\mathcal{I}(\mathbf{a}_j),l},\Xi^{(t+1)}([\mathbf{o}\;\;\mathbf{s}_j\;\;\mathbf{r}_i]) \rangle
		\ge& \langle \mathbf{w}^{(t)}_{\mathcal{I}(\mathbf{a}_j),l},\Xi^{(t)}([\mathbf{o}\;\;\mathbf{s}_j\;\;\mathbf{r}_i])\rangle -\mathcal{O}\left(\frac{\eta_2\lambda}{nm}\right)\\
        \ge& \frac{\log(d)}{12\lambda}.
\end{aligned}
\end{equation}
Thus, the sixth statement holds.

\textbf{Proof of the seventh statement:}
First, prove the first part.
  For iteration $t\in[0,T_p-1]$, we have
    \begin{align}
        \langle\mathbf{w}^{(t+1)}_{\mathcal{I}(\mathbf{a}_j),l},\mathbf{o}\rangle - \langle\mathbf{w}^{(t)}_{\mathcal{I}(\mathbf{a}_j),l},\mathbf{o}\rangle \le \frac{\eta^{(t)}\lambda}{nm}\sum_{i\in\mathcal{B}_j}\alpha_1^{(t)}([\mathbf{o}\;\;\mathbf{s}_j\;\;\mathbf{r}_i])(1-\text{logit}^{(t)}_{\mathcal{I}(\mathbf{a}_j)})([\mathbf{o}\;\;\mathbf{s}_j\;\;\mathbf{r}_i])\left\|\mathbf{o}\right\|_2^2
    \end{align}
    After pre-training Stage 1, for all $j\in[N]$ and $l\in[m]$, we have
    \begin{align}
         \langle\mathbf{w}^{(T_1)}_{\mathcal{I}(\mathbf{a}_j),l},\mathbf{o}\rangle - \langle\mathbf{w}^{(0)}_{\mathcal{I}(\mathbf{a}_j),l},\mathbf{o}\rangle \le \mathcal{O}\left(\frac{\eta_1\lambda K d}{nm}\right).
    \end{align}
    Then, at pre-training Stage 2 with iteration $t$, we have
    \begin{align}
         \langle\mathbf{w}^{(t+1)}_{\mathcal{I}(\mathbf{a}_j),l},\mathbf{o}\rangle - \langle\mathbf{w}^{(t)}_{\mathcal{I}(\mathbf{a}_j),l},\mathbf{o}\rangle \le \mathcal{O}\left(\frac{\eta_2\lambda K}{nmd}\right),
    \end{align}
    due to the attention score bound in the third statement.
    Then, we have
    \begin{equation}
    \begin{aligned}
        \langle\mathbf{w}^{(t+1)}_{\mathcal{I}(\mathbf{a}_j),l},\mathbf{o}\rangle - \langle\mathbf{w}^{(0)}_{\mathcal{I}(\mathbf{a}_j),l},\mathbf{o}\rangle \le& \mathcal{O}\left(\frac{\eta_1\lambda K d}{nm}\right)+\mathcal{O}\left(\frac{\eta_2\lambda K}{nmd}T_2\right)\\
        =& \mathcal{O}\left(\frac{\eta_1\lambda Kd}{nm} + \frac{\log(d)\lambda}{d}\right)\\
        =&\mathcal{O}\left(\frac{\log(d)}{n\lambda}\right).
    \end{aligned}
    \end{equation}

Next, we prove the second part.
For iteration $t\in[0,T_p-1]$, we have
    \begin{align}
        \langle\mathbf{w}^{(t+1)}_{\mathcal{I}(\mathbf{r}_i),l},\mathbf{o}\rangle - \langle\mathbf{w}^{(t)}_{\mathcal{I}(\mathbf{r}_i),l},\mathbf{o}\rangle \le \frac{\eta^{(t)}\lambda}{nm}\sum_{j\in\mathcal{D}_i}\alpha_1^{(t)}([\mathbf{o}\;\;\mathbf{s}_j])(1-\text{logit}^{(t)}_{\mathcal{I}(\mathbf{r}_i)}([\mathbf{o}\;\;\mathbf{s}_j]))\left\|\mathbf{o}\right\|_2^2
    \end{align}
 During pre-training Stage 1, we have
    \begin{align}
         \langle\mathbf{w}^{(T_1)}_{\mathcal{I}(\mathbf{a}_j),l},\mathbf{o}\rangle - \langle\mathbf{w}^{(0)}_{\mathcal{I}(\mathbf{a}_j),l},\mathbf{o}\rangle \le \mathcal{O}\left(\frac{\eta_1\lambda d}{m}\right),
    \end{align}
Then, at pre-training Stage 2 with iteration $t$, we have
    \begin{align}
         \langle\mathbf{w}^{(t+1)}_{\mathcal{I}(\mathbf{r}_i),l},\mathbf{o}\rangle - \langle\mathbf{w}^{(t)}_{\mathcal{I}(\mathbf{r}_i),l},\mathbf{o}\rangle \le \mathcal{O}\left(\frac{\eta_2\lambda}{md}\right),
    \end{align}
Then, we have
    \begin{equation}
    \begin{aligned}
        \langle\mathbf{w}^{(t+1)}_{\mathcal{I}(\mathbf{r}_i),l},\mathbf{o}\rangle - \langle\mathbf{w}^{(0)}_{\mathcal{I}(\mathbf{r}_i),l},\mathbf{o}\rangle \le& \mathcal{O}\left(\frac{\eta_1\lambda K d}{nm}\right)+\mathcal{O}\left(\frac{\eta_2\lambda}{md}T_2\right)\\
        =& \mathcal{O}\left(\frac{\eta_1\lambda d}{m} + n\log(d)/(\lambda d)\right)\\
        =&\mathcal{O}\left(\log(d)/(\lambda)\right).
    \end{aligned}
    \end{equation}
Thus the seventh statement holds at iteration $t+1$.

This finishes the proof of the induction.
    \end{proof}

\subsection{Proof of pre-training Stage 3}
In this subsection, we prove the convergence of the model pre-training.
We first derive the following key properties during Stage 3.
\begin{lemma}\label{lemma: pt_stage3}
	Under Condition \ref{condition: condition}, during pre-training Stage 3 ($T_2 < t \le T_p$), the following holds:
    \begin{enumerate}
        \item For all $j\in[N]$ and $i\in\mathcal{B}_j$, we have \begin{align} 		
		0.5\le \frac{\alpha^{(t)}_{2}([\mathbf{o}\:\:\mathbf{s}_j\:\:\mathbf{r}_i])}{\alpha^{(t)}_{3}([\mathbf{o}\:\:\mathbf{s}_j\:\:\mathbf{r}_i])} \le 2.
	\end{align}
        \item For all $j\in[N]$ and $i\in\mathcal{B}_j$, we have
        \begin{align}
        \alpha_1^{(t)}([\mathbf{o}\:\:\mathbf{s}_j]) \le \frac{1}{d^{2}}, \alpha_1^{(t)}([\mathbf{o}\:\:\mathbf{s}_j\:\:\mathbf{r}_i]) \le \frac{1}{d^{2}}.
    \end{align}
    \item For each $j\in[N]$ and all $k\in \mathcal{B}_j$, $l \in \mathcal{S}_{\text{s},k,j}^{(0)}$, we have $\langle\mathbf{w}_{\mathcal{I}(\mathbf{r}_k),l}^{(t)}, \Xi^{(t)}([\mathbf{o}\;\;\mathbf{s}_j]))\rangle > 0$.
    \item For each $j\in[N]$ and all $k\in\mathcal{B}_j$, $l\in\mathcal{S}_{\text{r},j,k}^{(0)}$. we have $\langle \mathbf{w}^{(t)}_{\mathcal{I}(\mathbf{a}_j),l}, \Xi^{(t)}([\mathbf{o}\:\:\mathbf{s}_j\:\:\mathbf{r}_k]) \rangle > 0$.
    \item For each $j\in[N]$ and all $k\in \mathcal{B}_j$, $l \in \mathcal{S}_{\text{s},k,j}^{(0)}$, we have $\langle\mathbf{w}_{\mathcal{I}(\mathbf{r}_k),l}^{(t)}, \mathbf{s}_j\rangle = \Omega(\log(d)/\lambda)$.
    \item For all $j\in[N], i\in\mathcal{R}$ and $l\in[m]$, we have 
    \begin{align}
        \langle\mathbf{w}^{(t)}_{\mathcal{I}(\mathbf{r}_i),l},\mathbf{o}\rangle \le \tilde{\mathcal{O}}\left( \frac{N+|\mathcal{R}|}{\lambda }\right),
        \langle\mathbf{w}^{(t)}_{\mathcal{I}(\mathbf{a}_j),l},\mathbf{o}\rangle \le \tilde{\mathcal{O}}\left( \frac{N+|\mathcal{R}|}{\lambda n}\right).
    \end{align}
    \item For any $j\in[N]$ and $i\in\mathcal{B}_j$, we have     \begin{align}
        \langle \mathbf{w}^{(t)}_{\mathcal{I}(\mathbf{r}_i),l},\Xi^{(t)}([\mathbf{o}\;\;\mathbf{s}_j]) \rangle =\tilde{\mathcal{O}}\left(\frac{N+|\mathcal{R}|}{\lambda}\right),
        \langle \mathbf{w}^{(t)}_{\mathcal{I}(\mathbf{a}_j),l},\Xi^{(t)}([\mathbf{o}\;\;\mathbf{s}_j\;\;\mathbf{r}_i]) \rangle = \tilde{\mathcal{O}}\left(\frac{N+|\mathcal{R}|}{\lambda}\right).
    \end{align}
    \end{enumerate}
\end{lemma}
\begin{proof}
    In this proof, we establish each of the seven statements in Lemma \ref{lemma: pt_stage3} by induction.
    It is easy to show that the conclusions hold at $T_2$ by Lemma \ref{lemma: pt_stage2}.
    Suppose that the conclusions hold at iteration $T_2\le t\le T_p-1$. 
    We now prove that the conclusions hold at iteration $t+1$ sequentially.    
    
    \textbf{Proof of the first statement:}
    We have
	\begin{align}
		\frac{\alpha^{(t)}_{2}([\mathbf{o}\:\:\mathbf{s}_j\:\:\mathbf{r}_i])}{\alpha^{(t)}_{3}([\mathbf{o}\:\:\mathbf{s}_j\:\:\mathbf{r}_i])} = \frac{\exp(\tilde{Z}^{(t)}_{\mathcal{I}(\mathbf{s}_j),\mathcal{I}(\mathbf{r}_i)}/\sqrt{d})}{\exp(\tilde{Z}^{(t)}_{\mathcal{I}(\mathbf{r}_i),\mathcal{I}(\mathbf{r}_i)}/\sqrt{d})} = \exp\left(\frac{\tilde{Z}^{(t)}_{\mathcal{I}(\mathbf{s}_j),\mathcal{I}(\mathbf{r}_i)} - \tilde{Z}^{(t)}_{\mathcal{I}(\mathbf{r}_i),\mathcal{I}(\mathbf{r}_i)}}{\sqrt{d}}\right).
	\end{align}
    For the term $\tilde{Z}^{(t)}_{\mathcal{I}(\mathbf{s}_j),\mathcal{I}(\mathbf{r}_i)} - \tilde{Z}^{(t)}_{\mathcal{I}(\mathbf{r}_i),\mathcal{I}(\mathbf{r}_i)}$, we bound it by bounding the increments of $\tilde{Z}^{(t)}_{\mathcal{I}(\mathbf{s}_j),\mathcal{I}(\mathbf{r}_i)}$ and $\tilde{Z}^{(t)}_{\mathcal{I}(\mathbf{r}_i),\mathcal{I}(\mathbf{r}_i)}$.
    We have
	\begin{equation}
		\begin{aligned}
			&|\tilde{Z}^{(t+1)}_{\mathcal{I}(\mathbf{s}_j),\mathcal{I}(\mathbf{r}_i)}- \tilde{Z}^{(t)}_{\mathcal{I}(\mathbf{s}_j),\mathcal{I}(\mathbf{r}_i)}|\\
			=& \frac{\eta_3\lambda}{n\sqrt{d}}\underbrace{\alpha_{2}^{(t)}([\mathbf{o}\:\:\mathbf{s}_j\:\:\mathbf{r}_i])}_{\le 1}\underbrace{\left(\mathbf{e}_{\mathcal{I}(\mathbf{a}_j)}-\textbf{logit}^{(t)}([\mathbf{o}\:\:\mathbf{s}_{j}\:\:\mathbf{r}_i])\right)^\top\boldsymbol{\zeta}^{(t)}([\mathbf{o}\:\:\mathbf{s}_j\:\:\mathbf{r}_i])\left(\mathbf{s}_j - \tilde{\Xi}^{(t)}([\mathbf{o}\:\:\mathbf{s}_j\:\:\mathbf{r}_i])\right)}_{=\tilde{\mathcal{O}}\left( \frac{N+|\mathcal{R}|}{\lambda}\right)}\\
			=& \mathcal{O}\left(\frac{\eta_3(N+|\mathcal{R}|) }{nd^{1/2}}\right),
		\end{aligned}
	\end{equation}
	and
	\begin{equation}
		\begin{aligned}
			&|\tilde{Z}^{(t+1)}_{\mathcal{I}(\mathbf{r}_i),\mathcal{I}(\mathbf{r}_i)}- \tilde{Z}^{(t)}_{\mathcal{I}(\mathbf{r}_i),\mathcal{I}(\mathbf{r}_i)}|\\
			=& \frac{\eta_3\lambda}{n\sqrt{d}}\sum_{j\in\mathcal{D}_i}\underbrace{\alpha_{3}^{(t)}([\mathbf{o}\:\:\mathbf{s}_j\:\:\mathbf{r}_i])}_{\le 1}\underbrace{\left(\mathbf{e}_{\mathcal{I}(\mathbf{a}_j)}-\textbf{logit}^{(t)}([\mathbf{o}\:\:\mathbf{s}_{j}\:\:\mathbf{r}_i])\right)^\top\boldsymbol{\zeta}^{(t)}([\mathbf{o}\:\:\mathbf{s}_j\:\:\mathbf{r}_i])\left(\mathbf{r}_i - \tilde{\Xi}^{(t)}([\mathbf{o}\:\:\mathbf{s}_j\:\:\mathbf{r}_i])\right)}_{=\tilde{\mathcal{O}}\left( \frac{N+|\mathcal{R}|}{\lambda}\right)}\\
			=& \mathcal{O}\left(\frac{\eta_3(N+|\mathcal{R}|)}{d^{1/2}}\right).
		\end{aligned}
	\end{equation}
	Therefore, we have
	\begin{align}\label{equ: s3 update}
		\frac{\alpha^{(t+1)}_{2}([\mathbf{o}\:\:\mathbf{s}_j\:\:\mathbf{r}_i])}{\alpha^{(t+1)}_{3}([\mathbf{o}\:\:\mathbf{s}_j\:\:\mathbf{r}_i)}\frac{\alpha^{(t)}_{3}([\mathbf{o}\:\:\mathbf{s}_j\:\:\mathbf{r}_i])}{\alpha^{(t)}_{2}([\mathbf{o}\:\:\mathbf{s}_j\:\:\mathbf{r}_i])} = \exp\left(\mathcal{O}\left(\frac{\eta_3(N+|\mathcal{R}|)}{d}\right)\right).
	\end{align}
	Using (\ref{equ: s3 update}), for $t+1$, we have
	\begin{equation}
	\begin{aligned}
		&\frac{\alpha^{(t+1)}_{2}([\mathbf{o}\:\:\mathbf{s}_j\:\:\mathbf{r}_i])}{\alpha^{(t+1)}_{3}([\mathbf{o}\:\:\mathbf{s}_j\:\:\mathbf{r}_i])}\frac{\alpha^{(T_2)}_{3}([\mathbf{o}\:\:\mathbf{s}_j\:\:\mathbf{r}_i])}{\alpha^{(T_2)}_{2}([\mathbf{o}\:\:\mathbf{s}_j\:\:\mathbf{r}_i])}\\
		=& \exp\left(\mathcal{O}\left(\frac{\eta_3(N+|\mathcal{R}|)}{d}(T_p-T_2)\right)\right)\\
		=& \exp\left(\tilde{\mathcal{O}}\left(\frac{(N+|\mathcal{R}|)}{d}\frac{(mN(N+|\mathcal{R}|))}{\lambda^2}\right)\right).
	\end{aligned}
	\end{equation}
	Under Condition \ref{condition: condition}, we have
	\begin{align}
		\frac{\alpha^{(t+1)}_{2}([\mathbf{o}\:\:\mathbf{s}_j\:\:\mathbf{r}_i])}{\alpha^{(t+1)}_{3}([\mathbf{o}\:\:\mathbf{s}_j\:\:\mathbf{r}_i])} \le 2.
	\end{align}
	Similarly, we have
	\begin{align}
		\frac{\alpha^{(t+1)}_{2}([\mathbf{o}\:\:\mathbf{s}_j\:\:\mathbf{r}_i])}{\alpha^{(t+1)}_{3}([\mathbf{o}\:\:\mathbf{s}_j\:\:\mathbf{r}_i])} \ge 0.5.
	\end{align}
Thus, the first statement holds at iteration $t+1$.

\textbf{Proof of the second statement:}
    For $T_2\le t \le T_p$, the increment of the attention matrix satisfies
\begin{equation}
	\begin{aligned}
		&|\tilde{Z}^{(t+1)}_{\mathcal{I}(\mathbf{o}),\mathcal{I}(\mathbf{s}_j)}- \tilde{Z}^{(t)}_{\mathcal{I}(\mathbf{o}),\mathcal{I}(\mathbf{s}_j)}|\\
		\le& \frac{\eta_3\lambda}{n\sqrt{d}}\underbrace{\alpha_{1}^{(t)}([\mathbf{o},\mathbf{s}_j])}_{\le 1/d^{2}}\sum_{i\in\mathcal{B}_j}\underbrace{\left|\left(\mathbf{e}_{\mathcal{I}(\mathbf{r}_i)}-\textbf{logit}^{(t)}([\mathbf{o}\:\:\mathbf{s}_{j}])\right)^\top\boldsymbol{\zeta}^{(t)}([\mathbf{o}\:\:\mathbf{s}_j])\left(\mathbf{o} - \tilde{\Xi}^{(t)}([\mathbf{o}\:\:\mathbf{s}_j])\right)\right|
        }_{=\tilde{\mathcal{O}}(\frac{N(N+|\mathcal{R}|)}{\lambda)}}\cdot\Theta(\sqrt{d})\\
		=& \mathcal{O}\left(\frac{\eta_3(N+|\mathcal{R}|)}{d^{3/2}}\right), 
	\end{aligned}
\end{equation}
and 
	\begin{equation}
	\begin{aligned}
		&|\tilde{Z}^{(t+1)}_{\mathcal{I}(\mathbf{s}_j),\mathcal{I}(\mathbf{s}_j)}- \tilde{Z}^{(t)}_{\mathcal{I}(\mathbf{s}_j),\mathcal{I}(\mathbf{s}_j)}|\\
		=& \frac{\eta_3\lambda}{n\sqrt{d}}\underbrace{\alpha_{2}^{(t)}([\mathbf{o},\mathbf{s}_j])}_{=\Theta(1)}\sum_{i\in\mathcal{B}_j}\underbrace{\left|\left(\mathbf{e}_{\mathcal{I}(\mathbf{r}_i)}-\textbf{logit}^{(t)}([\mathbf{o}\:\:\mathbf{s}_{j}])\right)^\top\boldsymbol{\zeta}^{(t)}([\mathbf{o}\:\:\mathbf{s}_j])\left(\mathbf{s}_j - \tilde{\Xi}^{(t)}([\mathbf{o}\:\:\mathbf{s}_j])\right)\right|}_{=\tilde{\mathcal{O}}((N+|\mathcal{R}|)/\lambda)}\\
		=& \mathcal{O}\left(\frac{\eta_3K(N+|\mathcal{R}|)}{nd^{1/2}}\right).
	\end{aligned}
\end{equation}
    So, for all $T_2 < t\le T_p$, we have 
	\begin{align}
		\frac{\alpha^{(t+1)}_{1}([\mathbf{o}\:\:\mathbf{s}_j])}{\alpha^{(t+1)}_{2}([\mathbf{o}\:\:\mathbf{s}_j])}\frac{\alpha^{(t)}_{2}([\mathbf{o}\:\:\mathbf{s}_j])}{\alpha^{(t)}_{1}([\mathbf{o}\:\:\mathbf{s}_j])} =\exp\left(\mathcal{O}\left(\frac{ \eta_3K(N+|\mathcal{R}|)}{nd}\right)\right).
	\end{align}
    As a result, we have
    \begin{equation}
    \begin{aligned}
		&\frac{\alpha^{(t+1)}_{1}([\mathbf{o}\:\:\mathbf{s}_j])}{\alpha^{(t+1)}_{2}([\mathbf{o}\:\:\mathbf{s}_j])}\frac{\alpha^{(T_2)}_{2}([\mathbf{o}\:\:\mathbf{s}_j])}{\alpha^{(T_2)}_{1}([\mathbf{o}\:\:\mathbf{s}_j])} =\exp\left(\mathcal{O}\left(\frac{ \eta_3K(N+|\mathcal{R}|)}{nd}(T_p-T_2)\right)\right)\\
        =& \exp\left(\mathcal{O}\left(\frac{ \eta_3K(N+|\mathcal{R}|)}{nd}\frac{(mN(N+|\mathcal{R}|)}{\lambda^2}\right)\right),
	\end{aligned}
    \end{equation}
resulting in
\begin{align}
    \frac{\alpha^{(t+1)}_{1}([\mathbf{o}\:\:\mathbf{s}_j])}{\alpha^{(t+1)}_{2}([\mathbf{o}\:\:\mathbf{s}_j])} \le \frac{1}{d^{2}}.
\end{align}
Moreover, we have 
\begin{equation}
\begin{aligned}
		&|\tilde{Z}^{(t+1)}_{\mathcal{I}(\mathbf{o}),\mathcal{I}(\mathbf{r}_i)}- \tilde{Z}^{(t)}_{\mathcal{I}(\mathbf{o}),\mathcal{I}(\mathbf{r}_i)}|\\
		=& \frac{\eta_3\lambda}{n\sqrt{d}}\sum_{j\in\mathcal{D}_i}\underbrace{\alpha_{1}^{(t)}([\mathbf{o}\:\:\mathbf{s}_{j}\:\:\mathbf{r}_i])}_{\le 1/d^{2}}\\
        &\cdot\underbrace{\left|\left(\mathbf{e}_{\mathcal{I}(\mathbf{a}_j)}-\textbf{logit}^{(t)}([\mathbf{o}\:\:\mathbf{s}_{j}\:\:\mathbf{r}_i])\right)^\top\boldsymbol{\zeta}^{(t)}([\mathbf{o}\:\:\mathbf{s}_j\:\:\mathbf{r}_i])\cdot\left(\mathbf{o} - \tilde{\Xi}^{(t)}([\mathbf{o}\:\:\mathbf{s}_j\:\:\mathbf{r}_i])\right)\right|}_{=\tilde{\mathcal{O}}(\frac{(N+|\mathcal{R}|)}{(\lambda d)})}\cdot\Theta(\sqrt{d})\\
		=& \mathcal{O}\left(\frac{\eta_3(N+|\mathcal{R}|)}{nd^{3/2}}\right).
	\end{aligned}
\end{equation}
So, for all $T_2 \le t\le T_p$, we have 
	\begin{align}
		\frac{\alpha^{(t+1)}_{1}([\mathbf{o}\:\:\mathbf{s}_j\:\:\mathbf{r}_i])}{\alpha^{(t+1)}_{2}([\mathbf{o}\:\:\mathbf{s}_j\:\:\mathbf{r}_i])}\frac{\alpha^{(t)}_{2}([\mathbf{o}\:\:\mathbf{s}_j\:\:\mathbf{r}_i])}{\alpha^{(t)}_{1}([\mathbf{o}\:\:\mathbf{s}_j\:\:\mathbf{r}_i])} =\exp\left(\mathcal{O}\left(\frac{\eta_3(N+|\mathcal{R}|))}{nd}\right)\right).
	\end{align}
    As a result, we have
    \begin{equation}
    \begin{aligned}
		&\frac{\alpha^{(t+1)}_{1}([\mathbf{o}\:\:\mathbf{s}_j\:\:\mathbf{r}_i])}{\alpha^{(t+1)}_{2}([\mathbf{o}\:\:\mathbf{s}_j\:\:\mathbf{r}_i])}\frac{\alpha^{(T_2)}_{2}([\mathbf{o}\:\:\mathbf{s}_j\:\:\mathbf{r}_i])}{\alpha^{(T_2)}_{1}([\mathbf{o}\:\:\mathbf{s}_j\:\:\mathbf{r}_i])}\\
        =&\exp\left(\mathcal{O}\left(\frac{\eta_3(N+|\mathcal{R}|))}{nd}(T_p-T_2)\right)\right)\\ =&\exp\left(\tilde{\mathcal{O}}\left(\frac{N+|\mathcal{R}|}{nd}\frac{mN(N+|\mathcal{R}|)}{\lambda^2}\right)\right),
    \end{aligned}
    \end{equation}
resulting in
\begin{align}
    \frac{\alpha^{(t+1)}_{1}([\mathbf{o}\:\:\mathbf{s}_j\:\:\mathbf{r}_i])}{\alpha^{(t+1)}_{2}([\mathbf{o}\:\:\mathbf{s}_j\:\:\mathbf{r}_i])} \le \frac{1}{d^{2}}.
\end{align}
Therefore, we have
\begin{align}
    \alpha^{(t+1)}_{1}([\mathbf{o}\:\:\mathbf{s}_j]) \le \frac{1}{d^{2}}, \alpha^{(t+1)}_{1}([\mathbf{o}\:\:\mathbf{s}_j\:\:\mathbf{r}_i]) \le \frac{1}{d^{2}}.
\end{align}
The second statement holds at iteration $t+1$.

\textbf{Proof of the third statement:}

 After Stage 1 with $t\ge T_2$, for $l\in[m]$ satisfying $\langle \mathbf{w}^{(t)}_{\mathcal{I}(\mathbf{r}_k),l},\mathbf{s}_j \rangle > 0$, the update of $\langle \mathbf{w}^{(t)}_{\mathcal{I}(\mathbf{r}_k),l},\mathbf{s}_j \rangle$ satisfies
    \begin{equation}
	\begin{aligned}
		&\langle \mathbf{w}^{(t+1)}_{\mathcal{I}(\mathbf{r}_k),l},\mathbf{s}_j \rangle\\
        \!\ge& \langle \mathbf{w}^{(t)}_{\mathcal{I}(\mathbf{r}_k),l},\mathbf{s}_j \rangle \!+\! \lambda\frac{\eta_3}{nm}(1-\tilde{K}(j)\text{logit}^{(t)}_{\mathcal{I}(\mathbf{r}_k)}([\mathbf{o}\:\:\mathbf{s}_j])) - \lambda\frac{\eta_3}{nm}\sum_{i\in\mathcal{B}_j}\text{logit}^{(t)}_{\mathcal{I}(\mathbf{r}_k)}([\mathbf{o}\:\:\mathbf{s}_j\:\:\mathbf{r}_i]),
	\end{aligned}
    \end{equation}
	by the fact that $\langle\Xi^{(t)}([\mathbf{o}\:\:\mathbf{s}_j]),\mathbf{s}_j\rangle \ge \left\|\mathbf{s}_j\right\|_2^2=1$ and $\langle\Xi^{(t)}([\mathbf{o}\:\:\mathbf{s}_j\:\:\mathbf{r}_i]),\mathbf{s}_j\rangle \le \left\|\mathbf{s}_j\right\|_2^2=1$.
	Similarly, for all $l\in[m]$, the update of $\langle \mathbf{w}^{(t)}_{\mathcal{I}(\mathbf{r}_k),l},\mathbf{o} \rangle$ satisfies
\begin{equation}
	\begin{aligned}
		\langle \mathbf{w}^{(t+1)}_{\mathcal{I}(\mathbf{r}_k),l},\mathbf{o} \rangle
        \!\ge& \langle \mathbf{w}^{(t)}_{\mathcal{I}(\mathbf{r}_k),l},\mathbf{o} \rangle  - \mathcal{O}\left(\frac{\lambda\eta_3}{nmd}\right)\sum_{i\in\mathcal{B}_j}\text{logit}^{(t)}_{\mathcal{I}(\mathbf{r}_k)}([\mathbf{o}\:\:\mathbf{s}_j\:\:\mathbf{r}_i])\\
 &- \mathcal{O}\left(\frac{\lambda\eta_3}{nmd}\right)\sum_{j\in\mathcal{D}_k}\text{logit}_{\mathcal{I}(\mathbf{r}_k)}^{(t)}([\mathbf{o}\;\;\mathbf{s}_j])
 -\mathcal{O}\left(\frac{\lambda\eta_3}{m}\right)\text{logit}_{\mathcal{I}(\mathbf{r}_k)}^{(t)}(2\mathbf{o})\\
 =& \langle \mathbf{w}^{(t)}_{\mathcal{I}(\mathbf{r}_k),l},\mathbf{o} \rangle - \mathcal{O}\left(\frac{\lambda\eta_3K}{nmd}\right)-\mathcal{O}\left(\frac{\lambda\eta_3K}{md}\right) - \mathcal{O}\left(\frac{\lambda\eta_3K}{md^{3/2}}\right),
	\end{aligned}
    \end{equation}
As a result, by the sixth statement, we have
 \begin{equation}
	\begin{aligned}
		&\langle \mathbf{w}^{(t+1)}_{\mathcal{I}(\mathbf{r}_k),l},\Xi^{(t+1)}([\mathbf{o}\;\;\mathbf{s}_j]) \rangle -\langle \mathbf{w}^{(t)}_{\mathcal{I}(\mathbf{r}_k),l},\Xi^{(t)}([\mathbf{o}\;\;\mathbf{s}_j]) \rangle\\
        \!\ge&  0.9\lambda\frac{\eta_3}{nm}(1-\tilde{K}(j)\text{logit}^{(t)}_{\mathcal{I}(\mathbf{r}_k)}([\mathbf{o}\:\:\mathbf{s}_j])) - 1.1\lambda\frac{\eta_3}{nm}\sum_{i\in\mathcal{B}_j}\text{logit}^{(t)}_{\mathcal{I}(\mathbf{r}_k)}([\mathbf{o}\:\:\mathbf{s}_j\:\:\mathbf{r}_i])\\
        & - \mathcal{O}\left(\frac{\lambda\eta_3}{md^{3/2}}\right)- \mathcal{O}\left(\frac{\eta_3\log^2(d)}{\lambda d}\right).
	\end{aligned}
    \end{equation}
    Case 1: If $\oldfrac{1}{m}\sum_{l=1}^m\sigma(\langle \mathbf{w}^{(t)}_{\mathcal{I}(\mathbf{r}_k),l},\Xi^{(t)}([\mathbf{o}\;\;\mathbf{s}_j]) \rangle) \le 2\log(d)/(3\lambda)$, for all $i\in\mathcal{B}_j$, we have
    \begin{align}
        1-\tilde{K}(j)\text{logit}^{(t)}_{\mathcal{I}(\mathbf{r}_k)}([\mathbf{o}\:\:\mathbf{s}_j]) = \Theta(1),\text{logit}^{(t)}_{\mathcal{I}(\mathbf{r}_k)}([\mathbf{o}\:\:\mathbf{s}_j\:\:\mathbf{r}_i]) \le \frac{1}{\sqrt{d}},
    \end{align}
by Lemma \ref{lemma: mag_n}.
As a result, $\oldfrac{1}{m}\sum_{l=1}^m\sigma(\langle \mathbf{w}^{(t+1)}_{\mathcal{I}(\mathbf{r}_k),l},\Xi^{(t+1)}([\mathbf{o}\;\;\mathbf{s}_j]) \rangle) -\oldfrac{1}{m}\sum_{l=1}^m\sigma(\langle \mathbf{w}^{(t)}_{\mathcal{I}(\mathbf{r}_k),l},\Xi^{(t)}([\mathbf{o}\;\;\mathbf{s}_j]) \rangle) = \Theta(1)$ holds.

Case 2: If $\oldfrac{1}{m}\sum_{l=1}^m\sigma(\langle \mathbf{w}^{(t)}_{\mathcal{I}(\mathbf{r}_k),l},\Xi^{(t)}([\mathbf{o}\;\;\mathbf{s}_j]) \rangle) > 2\log(d)/(3\lambda)$, for all $l\in[m]$, we have
\begin{align}
    \langle \mathbf{w}^{(t+1)}_{\mathcal{I}(\mathbf{r}_k),l},\Xi^{(t+1)}([\mathbf{o}\;\;\mathbf{s}_j]) \rangle -\langle \mathbf{w}^{(t)}_{\mathcal{I}(\mathbf{r}_k),l},\Xi^{(t)}([\mathbf{o}\;\;\mathbf{s}_j]) \rangle \ge -\mathcal{O}\left(\frac{\eta_3\lambda}{nm}\right).
\end{align}

Then, we have
\begin{align}
    \frac{1}{m}\sum_{l=1}^m\sigma(\langle \mathbf{w}^{(t+1)}_{\mathcal{I}(\mathbf{r}_k),l},\Xi^{(t+1)}([\mathbf{o}\;\;\mathbf{s}_j]) \rangle) \ge \frac{\log(d)}{2\lambda}.
\end{align}

Therefore, the third statement holds for iteration $t+1$.

\textbf{Proof of the fourth statement:}
Based on the fifth statement, we have
\begin{equation}
	\begin{aligned}
		&\langle \mathbf{w}^{(t+1)}_{\mathcal{I}(\mathbf{a}_j),l},\Xi^{(t+1)}([\mathbf{o}\;\;\mathbf{s}_j\;\;\mathbf{r}_i]) \rangle
		\ge \langle \mathbf{w}^{(t)}_{\mathcal{I}(\mathbf{a}_j),l},\Xi^{(t)}([\mathbf{o}\;\;\mathbf{s}_j\;\;\mathbf{r}_i])\rangle\\
		&-1.1\lambda \frac{\eta^{(t)} \tilde{K}(j)}{nm}\text{logit}_{\mathcal{I}(\mathbf{a}_j)}^{(t)}([\mathbf{o}\:\:\mathbf{s}_j])\!+\!\lambda\frac{\Theta(1)\cdot\eta^{(t)}}{nm}\sum_{k\in \mathcal{B}_j}(1\!-\!\text{logit}^{(t)}_{\mathcal{I}(\mathbf{a}_j)}([\mathbf{o}\:\:\mathbf{s}_j\:\:\mathbf{r}_k]))\\
        &-2\lambda\sum_{k\neq j,k\in\mathcal{D}_i}\frac{\eta^{(t)}}{nm}\text{logit}^{(t)}_{\mathcal{I}(\mathbf{a}_j)}([\mathbf{o}\:\:\mathbf{s}_k\:\:\mathbf{r}_i])-\mathcal{O}\left(\frac{\lambda\eta^{(t)}}{md^{3/2}}\right) - \mathcal{O}\left(\frac{\eta^{(t)}\log^2(d)}{\lambda d}\right),
	\end{aligned}
\end{equation}

We consider two cases.

Case 1: $\oldfrac{1}{m}\sum_{l=1}^m\sigma\left(\langle \mathbf{w}_{\mathcal{I}(\mathbf{a}_j),l}^{(t)},\Xi^{(t)}([\mathbf{o}\;\;\mathbf{s}_j\;\;\mathbf{r}_i])\rangle\right) \le \frac{\log(d)}{6\lambda}$.
By the fifth statement, we also have 
\begin{align}
    \frac{1}{m}\sum_{l=1}^m\sigma\left(\langle \mathbf{w}_{\mathcal{I}(\mathbf{a}_j),l}^{(t)},\mathbf{s}_j\rangle\right) \le \frac{\log(d)}{2\lambda}, \frac{1}{m}\sum_{l=1}^m\sigma\left(\langle \mathbf{w}_{\mathcal{I}(\mathbf{a}_j),l}^{(t)},\mathbf{r}_i\rangle\right) \le \frac{\log(d)}{6\lambda}.
\end{align}
A direct result is
\begin{equation}
\begin{aligned}
    \langle \mathbf{w}^{(t+1)}_{\mathcal{I}(\mathbf{a}_j),l},\Xi^{(t+1)}([\mathbf{o}\;\;\mathbf{s}_j\;\;\mathbf{r}_i]) \rangle
		\ge& \langle \mathbf{w}^{(t)}_{\mathcal{I}(\mathbf{a}_j),l},\Xi^{(t)}([\mathbf{o}\;\;\mathbf{s}_j\;\;\mathbf{r}_i])\rangle + \Theta\left(\frac{\lambda\eta_3}{mn}\right)-\mathcal{O}\left(\frac{\eta_3\lambda}{nm\sqrt{d}}\right)\\
        \ge& \langle \mathbf{w}^{(t)}_{\mathcal{I}(\mathbf{a}_j),l},\Xi^{(t)}([\mathbf{o}\;\;\mathbf{s}_j\;\;\mathbf{r}_i])\rangle.
\end{aligned}
\end{equation}

Case 2: $\oldfrac{1}{m}\sum_{l=1}^m\sigma\left(\langle \mathbf{w}_{\mathcal{I}(\mathbf{a}_j),l}^{(t)},\Xi^{(t)}([\mathbf{o}\;\;\mathbf{s}_j\;\;\mathbf{r}_i])\rangle\right) \ge \frac{\log(d)}{6\lambda}$.
In this case, we have
\begin{equation}
\begin{aligned}
    \langle \mathbf{w}^{(t+1)}_{\mathcal{I}(\mathbf{a}_j),l},\Xi^{(t+1)}([\mathbf{o}\;\;\mathbf{s}_j\;\;\mathbf{r}_i]) \rangle
		\ge& \langle \mathbf{w}^{(t)}_{\mathcal{I}(\mathbf{a}_j),l},\Xi^{(t)}([\mathbf{o}\;\;\mathbf{s}_j\;\;\mathbf{r}_i])\rangle -\mathcal{O}\left(\frac{\eta_3\lambda}{nm}\right)\\
        \ge& \frac{\log(d)}{12\lambda}.
\end{aligned}
\end{equation}
Thus, the fourth statement holds for iteration $t+1$.

\textbf{Proof of the fifth statement:}
With the derivation of the fifth and the sixth statements in Lemma \ref{lemma: pt_stage2}, after $T_2$ iterations, we have $\frac{1}{m}\sum_{l=1}^m\sigma(\langle \mathbf{w}^{(T_2)}_{\mathcal{I}(\mathbf{r}_k),l},\Xi^{(T_2)}([\mathbf{o}\;\;\mathbf{s}_j]) \rangle) \ge \frac{\log(d)}{2\lambda}$ and $\frac{1}{m}\sum_{l=1}^m\langle \mathbf{w}^{(T_2)}_{\mathcal{I}(\mathbf{a}_j),l},\mathbf{s}_j \rangle \ge \frac{\log(d)}{4\lambda}$.
If there exists an iteration $t\in[T_2,T_p]$ that $\frac{1}{m}\sum_{l=1}^m\langle \mathbf{w}^{(t)}_{\mathcal{I}(\mathbf{a}_j),l}, \mathbf{s}_j\rangle<\log(d)/3\lambda$, for all $k\in \mathcal{B}_j$, $l \in \mathcal{S}_{\text{s},k,j}^{(0)}$ and for $t'>t$, we have
\begin{equation}
\begin{aligned}
\langle \mathbf{w}^{(t')}_{\mathcal{I}(\mathbf{a}_j),l},\mathbf{s}_j\rangle \ge& \log(d)/3\lambda - \mathcal{O}\left(\frac{\lambda}{nm}\frac{1}{d^3}(T_p-T_2)\right)\\
=& \log(d)/3\lambda - \mathcal{O}\left(\frac{\lambda\eta_3}{nm}\frac{1}{d}(Nm(N+|\mathcal{R}|))\right)\\
\ge& \log(d)/6\lambda.
\end{aligned}
\end{equation}
Thus, the conclusion holds for iteration $t+1$.

\textbf{Proof of the sixth statement:}
First, we prove the first part.
  For iteration $t\in[0,T_p-1]$, we have
    \begin{align}
        \langle\mathbf{w}^{(t+1)}_{\mathcal{I}(\mathbf{a}_j),l},\mathbf{o}\rangle - \langle\mathbf{w}^{(t)}_{\mathcal{I}(\mathbf{a}_j),l},\mathbf{o}\rangle \le \frac{\eta^{(t)}\lambda}{nm}\sum_{i\in\mathcal{B}_j}\alpha_1^{(t)}([\mathbf{o}\;\;\mathbf{s}_j\;\;\mathbf{r}_i])(1-\text{logit}^{(t)}_{\mathcal{I}(\mathbf{a}_j)}([\mathbf{o}\;\;\mathbf{s}_j\;\;\mathbf{r}_i]))\left\|\mathbf{o}\right\|_2^2
    \end{align}
    At pre-training Stage 3 with iteration $t\in[T_2,T_p]$, we have
    \begin{align}
         \langle\mathbf{w}^{(t+1)}_{\mathcal{I}(\mathbf{a}_j),l},\mathbf{o}\rangle - \langle\mathbf{w}^{(t)}_{\mathcal{I}(\mathbf{a}_j),l},\mathbf{o}\rangle \le \mathcal{O}\left(\frac{\eta_3\lambda K}{nmd}\right),
    \end{align}
    Then, we have
    \begin{equation}
    \begin{aligned}
        \langle\mathbf{w}^{(t+1)}_{\mathcal{I}(\mathbf{a}_j),l},\mathbf{o}\rangle\le& \langle\mathbf{w}^{(T_2)}_{\mathcal{I}(\mathbf{a}_j),l},\mathbf{o}\rangle+\mathcal{O}\left(\frac{\eta_3\lambda K}{nmd}(T_p-T_2)\right)\\
        =&\mathcal{O}\left(\frac{\log(d)}{\lambda n}\right) + \tilde{O}\left(\frac{N+|\mathcal{R}|}{\lambda d}\right)\\
        =& \tilde{O}\left(\frac{N+|\mathcal{R}|}{\lambda n}\right).
    \end{aligned}
    \end{equation}

Next, we prove the second part.
For iteration $t\in[0,T_p-1]$, we have
    \begin{align}
        \langle\mathbf{w}^{(t+1)}_{\mathcal{I}(\mathbf{r}_i),l},\mathbf{o}\rangle - \langle\mathbf{w}^{(t)}_{\mathcal{I}(\mathbf{r}_i),l},\mathbf{o}\rangle \le \frac{\eta^{(t)}\lambda}{nm}\sum_{j\in\mathcal{D}_i}\alpha_1^{(t)}([\mathbf{o}\;\;\mathbf{s}_j])(1-\text{logit}^{(t)}_{\mathcal{I}(\mathbf{r}_i)}([\mathbf{o}\;\;\mathbf{s}_j]))\left\|\mathbf{o}\right\|_2^2
    \end{align}
 During pre-training Stage 3 with iteration $t$, we have
    \begin{align}
         \langle\mathbf{w}^{(t+1)}_{\mathcal{I}(\mathbf{r}_i),l},\mathbf{o}\rangle - \langle\mathbf{w}^{(t)}_{\mathcal{I}(\mathbf{r}_i),l},\mathbf{o}\rangle \le \mathcal{O}\left(\frac{\eta_3\lambda}{md}\right),
    \end{align}
Then, we have
    \begin{equation}
    \begin{aligned}
        \langle\mathbf{w}^{(t+1)}_{\mathcal{I}(\mathbf{r}_i),l},\mathbf{o}\rangle \le& \langle\mathbf{w}^{(T_2)}_{\mathcal{I}(\mathbf{r}_i),l},\mathbf{o}\rangle + \mathcal{O}\left(\frac{\eta_3\lambda}{md}(T_p-T_2)\right)\\
        =&\tilde{\mathcal{O}}\left( \frac{N+|\mathcal{R}|}{\lambda }\right).
    \end{aligned}
    \end{equation}
Thus the conclusion holds for iteration $t+1$.

    \textbf{Proof of the seventh statement:}
 Based on the gradient form in Lemma \ref{lemma: gradient}, in Stage 3 $t\in[T_2,T_p-1]$, for all $l\in[m]$ and $i\in\mathcal{B}_j$, we have
    \begin{align}
        \langle \mathbf{w}^{(t+1)}_{\mathcal{I}(\mathbf{r}_i),l},\mathbf{s}_j \rangle - \langle \mathbf{w}^{(t)}_{\mathcal{I}(\mathbf{r}_i),l},\mathbf{s}_j \rangle \le 2\frac{\lambda \eta_3}{nm}(1-\tilde{K}(j)\text{logit}_{\mathcal{I}(\mathbf{r}_i)}([\mathbf{o}\:\:\mathbf{s}_j])),
    \end{align}
    \begin{align}
        \langle \mathbf{w}^{(t+1)}_{\mathcal{I}(\mathbf{a}_j),l},\mathbf{s}_j \rangle - \langle \mathbf{w}^{(t)}_{\mathcal{I}(\mathbf{a}_j),l},\mathbf{s}_j \rangle \le 2\sum_{i\in\mathcal{B}_j}\frac{\lambda \eta_3}{nm}(1-\text{logit}_{\mathcal{I}(\mathbf{a}_j)}([\mathbf{o}\:\:\mathbf{s}_j\:\:\mathbf{r}_i])).
    \end{align}
    At the end of stage 2, we have 
    \begin{align}
        \langle \mathbf{w}^{(T_2)}_{\mathcal{I}(\mathbf{r}_i),l},\mathbf{s}_j \rangle = \mathcal{O}(\frac{\log(d)}{\lambda}), \langle \mathbf{w}^{(T_2)}_{\mathcal{I}(\mathbf{a}_j),l},\mathbf{s}_j \rangle = \mathcal{O}(\frac{\log(d)}{\lambda}).
    \end{align}

    As $T_p - T_2 = \tilde{\Theta}((mN(N+|\mathcal{R}|))/(\lambda^2\eta_3))$, for all $t+1$, we have
    \begin{align}
        \langle \mathbf{w}^{(t+1)}_{\mathcal{I}(\mathbf{r}_i),l},\mathbf{s}_j \rangle =& \tilde{\mathcal{O}}\left(\frac{N+|\mathcal{R}|}{\lambda}\right),\\
        \langle \mathbf{w}^{(t+1)}_{\mathcal{I}(\mathbf{a}_j),l},\mathbf{s}_j \rangle =& \tilde{\mathcal{O}}\left(\frac{N+|\mathcal{R}|}{\lambda}\right),\\
        \langle \mathbf{w}^{(t+1)}_{\mathcal{I}(\mathbf{a}_j),l},\mathbf{r}_i \rangle =& \tilde{\mathcal{O}}\left(\frac{N+|\mathcal{R}|}{\lambda}\right),
    \end{align}
    As a result, we have
    \begin{align}
        \langle \mathbf{w}^{(t+1)}_{\mathcal{I}(\mathbf{r}_i),l},\Xi^{(t+1)}([\mathbf{o}\;\;\mathbf{s}_j]) \rangle =& \tilde{\mathcal{O}}\left(\frac{N+|\mathcal{R}|}{\lambda}\right),\\
        \langle \mathbf{w}^{(t+1)}_{\mathcal{I}(\mathbf{a}_j),l},\Xi^{(t+1)}([\mathbf{o}\;\;\mathbf{s}_j\;\;\mathbf{r}_i]) \rangle =& \tilde{\mathcal{O}}\left(\frac{N+|\mathcal{R}|}{\lambda}\right).
    \end{align}
Thus, the conclusion holds at iteration $t+1$.

    This completes the proof of the induction.
\end{proof}
Combined with Lemma \ref{lemma:init}, we can characterize the activation patterns during \ac{pt}.
\begin{lemma}\label{lemma: pattern pt}
    For each iteration $t\in [T_p]$, for all $j\in[N]$ and $k\in\mathcal{B}_j$, with probability $1-\delta$, the followings hold:
    \begin{align}
        |\mathcal{S}_{\text{s},k,j}^{(t)}| \ge 0.4m
        \quad\text{and}\quad
        |\mathcal{S}_{\text{r},j,k}^{(t)}| \ge 0.4m.
    \end{align}
\end{lemma}

Finally, we prove the convergence of pre-training in Stage 3.
\begin{lemma}
    Under condition 1, for any constant $\kappa$, there exist an iteration $0<t_p<T_p = \tilde{\Theta}(\frac{mN(N+|\mathcal{R}|)}{(\lambda^2\eta_3)}) + \frac{nm\log(d)}{(\lambda^2\eta_2K)})$ such that
    \begin{align}
        \mathcal{L}_{\mathcal{P}}(\mathbf{W}^{(t_p)},\mathbf{Z}^{(t_p)}) \le 0.001 + (1+\kappa) H.
    \end{align}
\end{lemma}
\begin{proof}
Let $\mathbf{W}^*$ be
\begin{equation}\label{equ: w*}
	\begin{aligned}
		\mathbf{w}^*_{j,k} =& \mathbf{w}^{(0)}_{j,k} + \frac{6}{\lambda}\log\left(\frac{d\epsilon}{(1-\epsilon)}\right)\sum_{i=1}^{N}\left(\sum_{l\in\mathcal{B}_i}\mathbf{r}_l - 6\sum_{l\neq i}\mathbf{s}_l \right)\mathbb{I}(j = \mathcal{I}(\mathbf{a}_i))\\
		&+ \frac{5}{\lambda}\log\left(\frac{d\epsilon}{(1-\epsilon)}\right)\sum_{l=1}^{K}\left(\sum_{i\in \mathcal{D}_l}\mathbf{s}_i - \sum_{l=1}^{K}\mathbf{r}_l\right)\mathbb{I}(j = \mathcal{I}(\mathbf{r}_l))\\
		&+ \frac{15}{2\sqrt{d}\lambda}\log\left(\frac{d\epsilon}{(1-\epsilon)}\right)\left(\mathbf{o}-\sum_{i=1}^{N}\mathbf{s}_i - \sum_{i=1}^{K}\mathbf{r}_i\right)\mathbb{I}\left(j=\mathcal{I}(\mathbf{d})\right),
	\end{aligned}
\end{equation}
where $\epsilon = 0.0001$.
With $\mathbf{W}^*$, we have
\begin{equation}\label{equ: descent}
	\begin{aligned}
		\|\mathbf{W}^{(t+1)}-\mathbf{W}^*\|_F^2 =& \|\mathbf{W}^{(t)}-\mathbf{W}^*\|_F^2 - \underbrace{2\eta_3\langle \nabla_{\mathbf{W}^{(t)}} \mathcal{L}_\mathcal{P}(\mathbf{W}^{(t)},\mathbf{Z}^{(t)}), \mathbf{W}^{(t)}-\mathbf{W}^*\rangle}_{A}\\
		&+ \eta_3^2\|\nabla_{\mathbf{W}^{(t)}} \mathcal{L}_\mathcal{P}(\mathbf{W}^{(t)},\mathbf{Z}^{(t)})\|_F^2.
	\end{aligned}
\end{equation}
For the term $\langle \nabla_{\mathbf{W}^{(t)}} \mathcal{L}_\mathcal{P}(\mathbf{W}^{(t)},\mathbf{Z}^{(t)}), \mathbf{W}^{(t)}-\mathbf{W}^*\rangle$, by the homogeneity of the ReLU function and the convexity of the cross-entropy loss, we have
\begin{equation}\label{equ: A}
	\begin{aligned}
		A=&2\eta_3\left\langle \nabla_{\mathbf{W}^{(t)}} \mathcal{L}_\mathcal{P}(\mathbf{W}^{(t)},\mathbf{Z}^{(t)}), \mathbf{W}^{(t)}-\mathbf{W}^*\right\rangle\\
		=& \frac{2\eta_3}{n}\sum_{(\mathbf{X},y)\in \mathcal{P}}  \left\langle\nabla_{\mathbf{W}^{(t)}} \mathcal{L}(\mathbf{W}^{(t)},\mathbf{Z}^{(t)},\mathbf{X},y), \mathbf{W}^{(t)}-\mathbf{W}^*\right\rangle\\
		=& \frac{2\eta_3}{n} \sum_{(\mathbf{X},y)\in \mathcal{P}} \sum_{k\in[d]}\left\langle \frac{\partial \mathcal{L}(\mathbf{W}^{(t)},\mathbf{Z}^{(t)},\mathbf{X},y)}{\partial x^\text{output}_k(\mathbf{W}^{(t)},\mathbf{Z}^{(t)}, \mathbf{X})} \nabla_{\mathbf{W}^{(t)}}x^\text{output}_k(\mathbf{W}^{(t)},\mathbf{Z}^{(t)}, \mathbf{X}), \mathbf{W}^{(t)} - \mathbf{W}^*\right\rangle\\
		=& \frac{2\eta_3}{n} \sum_{(\mathbf{X},y)\in \mathcal{P}}\sum_{k\in[d]}\frac{\partial \mathcal{L}(\mathbf{W}^{(t)},\mathbf{Z}^{(t)},\mathbf{X},y)}{\partial x^\text{output}_k(\mathbf{W}^{(t)},\mathbf{Z}^{(t)}, \mathbf{X})}\left\langle\nabla_{\mathbf{W}^{(t)}} x^\text{output}_k(\mathbf{W}^{(t)},\mathbf{Z}^{(t)},\mathbf{X}), \mathbf{W}^{(t)}\right\rangle\\
		&- \frac{2\eta_3}{n} \sum_{(\mathbf{X},y)\in \mathcal{P}}\sum_{k\in [d]}\frac{\partial \mathcal{L}(\mathbf{W}^{(t)},\mathbf{Z}^{(t)},\mathbf{X},y)}{\partial x^\text{output}_k(\mathbf{W}^{(t)},\mathbf{Z}^{(t)}, \mathbf{X})}\left\langle\nabla_{\mathbf{W}^{(t)}} x^\text{output}_k(\mathbf{W}^{(t)},\mathbf{Z}^{(t)},\mathbf{X}), \mathbf{W}^*\right\rangle\\
		=& \frac{2\eta_3}{n}\sum_{(\mathbf{X},y)\in \mathcal{P}}\sum_{k\in [d]} \frac{\partial \mathcal{L}(\mathbf{W}^{(t)},\mathbf{Z}^{(t)},\mathbf{X},y)}{\partial x^\text{output}_k(\mathbf{W}^{(t)},\mathbf{Z}^{(t)}, \mathbf{X})} \\
        &\left(x^\text{output}_k(\mathbf{W}^{(t)},\mathbf{Z}^{(t)},\mathbf{X}) - \underbrace{\left\langle \nabla_{\mathbf{W}^{(t)}} x^\text{output}_k(\mathbf{W}^{(t)},\mathbf{Z}^{(t)}, \mathbf{X}) , \mathbf{W}^* \right\rangle}_{B}\right),
	\end{aligned}
\end{equation}
Before deriving the bound of the term $B$, we first recall that by Lemma \ref{lemma: pattern pt}, in each iteration, there is $\Theta(m)$ amount of neurons are activated and by Lemma \ref{lemma: pt_stage3} statements 1 and 2, the attention scores of tokens $\mathbf{s}_j$ and $\mathbf{r}$ are at least 0.3 for all $j\in[N]$ and $i\in\mathcal{B}_j$.
Additionally, by the form of the gradient, the gradient of data $(\mathbf{X},y)$ for the MLP given by
\begin{align}\label{equ: grad_mlp}
    \nabla_{\mathbf{w}^{(t)}_{i,l}} x^\text{output}_k(\mathbf{W}^{(t)}, \mathbf{Z}^{(t)},\mathbf{X}) = 
    \begin{cases}
        -\oldfrac{\lambda}{m}\sigma'(\langle\mathbf{w}^{(t)}_{i,l},\mathbf{x}_a(\mathbf{Z}^{(t)},\mathbf{X}) \rangle)\mathbf{x}_a(\mathbf{Z}^{(t)},\mathbf{X}), &\text{ if }i = y,\\
        0 &\text{ if }i\neq y.
    \end{cases}
\end{align}

For the term $B$, we consider three cases.

Case 1: $(\mathbf{X},y)$ is in the form of $([\mathbf{o}\:\:\mathbf{s}_j\:\:\mathbf{r}_i],\mathcal{I}(\mathbf{a}))$.
By the form of the gradient for the MLP in (\ref{equ: grad_mlp}), the construction of $\mathbf{W}^*$ in (\ref{equ: w*}) and the number of activated neurons is at least $0.4m$, we have
\begin{align}\label{equ: s3c1}
	B \begin{cases}
		\le-1.2\log(\frac{\epsilon (d-1)}{(1-\epsilon)}), & \text{if } k = \mathcal{I}(\mathbf{a}_j),\\
		=\tilde{\mathcal{O}}(\sigma_0), & \text{if } k\neq \mathcal{I}(\mathbf{a}_j).
	\end{cases}
\end{align}

Case 2: $(\mathbf{X},y)$ is in the form of $([\mathbf{o}\:\:\mathbf{s}_j],\mathcal{I}(\mathbf{r}_i))$, we have
\begin{align}\label{equ: s3c2}
	B \begin{cases}
		=\tilde{\mathcal{O}}(\sigma_0), & \text{if } k = \mathcal{I}(\mathbf{a}_j),\\
		\le -1.2\log(\frac{\epsilon (d-1)}{(1-\epsilon)}), & \text{if } k= \mathcal{I}(\mathbf{r}_i), \forall i\in \mathcal{D}_j,\\
		=\tilde{\mathcal{O}}(\sigma_0),&\text{if } k\neq \mathcal{I}(\mathbf{a}_j), \mathcal{I}(\mathbf{r}_i), \forall i\in  \mathcal{D}_j.
	\end{cases}
\end{align}
Case 3: If $(\mathbf{X},y)$ is in the form of 
$(\mathbf{o},\mathcal{I}(\mathbf{d}))$, we have
\begin{align}\label{equ: s3c3}
	B \begin{cases}
		\le-1.2\sqrt{d}\log(\frac{\epsilon (d-1))}{(1-\epsilon)}), & \text{if } k = \mathcal{I}(\mathbf{d}),\\
		=\tilde{\mathcal{O}}(\sqrt{d}\sigma_{0}), & \text{if } k\neq \mathcal{I}(\mathbf{d}).
	\end{cases}
\end{align}

Combining the three cases (\ref{equ: s3c1}), (\ref{equ: s3c2}), (\ref{equ: s3c3}) and (\ref{equ: A}), we have
\begin{align}\label{equ: A2}
	A \ge 2\eta_3 \mathcal{L}_\mathcal{P}(\mathbf{W}^{(t)},\mathbf{Z}^{(t)})-2\eta_3 \epsilon-\frac{2\eta_3}{3}\frac{K}{K+1}\log(K).
\end{align}

Additionally, for the term $\eta_3^2\|\nabla_{\mathbf{W}^{(t)}} \mathcal{L}_\mathcal{P}(\mathbf{W}^{(t)},\mathbf{Z}^{(t)})\|_F^2$ in (\ref{equ: descent}), we have
\begin{equation}\label{equ: B}
	\begin{aligned}
		\eta_3^2\|\nabla_{\mathbf{W}^{(t)}} \mathcal{L}_\mathcal{P}(\mathbf{W}^{(t)},\mathbf{Z}^{(t)})\|_F^2
		\le&\eta_3^2\! \left[\! \frac{1}{n} \sum_{i=1}^{n}\sum_{k\in[d]}\left|\frac{\partial \mathcal{L}(\mathbf{W}^{(t)},\mathbf{Z}^{(t)},\mathbf{x},y)}{\partial x^\text{output}_k(\mathbf{W}^{(t)},\mathbf{Z}^{(t)},\mathbf{X})}\right| \!\left\| \nabla x^\text{output}_k(\mathbf{W}^{(t)},\mathbf{Z}^{(t)},\mathbf{X})\right\|_F\!\right]^2\\
		\le& \eta_3^2 4^2\cdot 4\left[\frac{1}{n}\sum_{i=1}^{n} \left(1-\text{logit}_y^{(t)}\left(\mathbf{X}\right)\right)\right]^2\\
		\le& 64\eta_3^2 \mathcal{L}_\mathcal{P}(\mathbf{W}^{(t)},\mathbf{Z}^{(t)}).
	\end{aligned}
\end{equation}
Substituting (\ref{equ: A2}) and (\ref{equ: B}) into (\ref{equ: descent}), we have
\begin{equation}
	\begin{aligned}
		\|\mathbf{W}^{(t+1)}-\mathbf{W}^*\|_F^2 \le&  \|\mathbf{W}^{(t)}-\mathbf{W}^*\|_F^2 - 2\eta_3\left(\mathcal{L}_\mathcal{P}(\mathbf{W}^{(t)},\mathbf{Z}^{(t)}) - \epsilon - \frac{K}{K+1}\log(K)/3\right)\\
		&+ 64\eta_3^2\mathcal{L}_\mathcal{P}(\mathbf{W}^{(t)},\mathbf{Z}^{(t)}).
	\end{aligned}
\end{equation}
Telescoping over $t$ and take $\eta_3 \le \frac{\kappa}{64} (\kappa\le 1/2)$, we have
\begin{equation}\label{equ: s3aux1}
	\begin{aligned}
		\|\mathbf{W}^{(T_p)}-\mathbf{W}^*\|_F^2 \le& \|\mathbf{W}^{(T_2+1)}-\mathbf{W}^*\|_F^2 - (2-\kappa)\eta_3\sum_{t=T_2+1}^{T_p}\mathcal{L}_\mathcal{P}(\mathbf{W}^{(t)},\mathbf{Z}^{(t)})+2\eta_3 (T_p-T_2)\epsilon\\
		&+ 2\eta_3 (T_p-T_2)\underbrace{\frac{K}{K+1}\frac{\log(K)}{3}}_{H}.
	\end{aligned}
\end{equation}
Rearranging (\ref{equ: s3aux1}) , we have
\begin{align}
	\frac{1}{T_p-T_2}\sum_{t=T_2+1}^{T_p}\mathcal{L}_\mathcal{P}(\mathbf{W}^{(t)},\mathbf{Z}^{(t)}) \le \frac{\|\mathbf{W}^{(T_2)}-\mathbf{W}^*\|_F^2}{(2-\kappa)\eta_3 (T_p-T_2)} + 2(1+\kappa) \epsilon + (1+\kappa)H.
\end{align}

For the term $\|\mathbf{W}^{(T_2)}-\mathbf{W}^*\|_F^2$, we have
\begin{equation}
\begin{aligned}
	\|\mathbf{W}^{(T_2)}-\mathbf{W}^*\|_F^2 \le&  \|\mathbf{W}^{(T_2)}-\mathbf{W}^{(0)}\|_F^2 + \|\mathbf{W}^{(0)}-\mathbf{W}^*\|_F^2\\
	=& \tilde{\mathcal{O}}(m/\lambda^2) +  \tilde{\mathcal{O}}(mN(N+|\mathcal{R}|)/\lambda^2)\\
    =& \tilde{\mathcal{O}}(mN(N+|\mathcal{R}|)/\lambda^2).
\end{aligned}
\end{equation}
Setting $T_p - T_2 = \tilde{\Theta}(\frac{mN(N+|\mathcal{R}|)}{(\lambda^2\eta_3)})$ completes the proof.
\end{proof}

\section{Proofs of full fine-tuning}\label{app: proof ft}
In this part, we provide the proofs of the results for full fine-tuning.
We show some key properties of the activation patterns, attention scores during full \ac{ft} in \ref{appendix: activation_pattern} and \ref{appendix: attention_score}, and prove the results for full \ac{ft} in statements of Theorem \ref{theorem: main} in \ref{appendix: proof statement 2} and \ref{appendix: proof statement 3}.
For each \ac{ft} iteration $t_f\in[0,T_f]$, we represent the iteration index with $T_p+t_f$.
\subsection{Attention scores during full FT}\label{appendix: attention_score}
\begin{lemma}\label{lemma: attn_fft}
    Under Condition \ref{condition: condition}, during full FT, the following holds:
    \begin{align}
       \frac{1}{2}\le \frac{\alpha_{1}^{(T_p+0)}([\mathbf{s}_j\:\:\mathbf{p}])}{\alpha_2^{(T_p+0)}([\mathbf{s}_j\:\:\mathbf{p}])} \le 2.
    \end{align}
\end{lemma}
\begin{proof}
    As the attention matrix $\mathbf{Z}$ has never been updated during PT with respect to the query of $\mathbf{p}$, 
    \begin{align}
        \frac{\alpha_{1}^{(T_p+0)}([\mathbf{s}_j\:\:\mathbf{p}])}{\alpha_2^{(T_p+0)}([\mathbf{s}_j\:\:\mathbf{p}])} = \frac{\alpha_{1}^{(0)}([\mathbf{s}_j\:\:\mathbf{p}])}{\alpha_2^{(0)}([\mathbf{s}_j\:\:\mathbf{p]}])} = \frac{\exp(\tilde{\mathcal{O}}(\sigma_0))}{\exp(\tilde{\mathcal{O}}(\sigma_0))},
    \end{align}
    by Lemma \ref{lemma: attn_init_order}.
    As a result, we have 
    \begin{align}
         0.95 \le\frac{\alpha_{1}^{(T_p+0)}([\mathbf{s}_j\:\:\mathbf{p}])}{\alpha_2^{(T_p+0)}([\mathbf{s}_j\:\:\mathbf{p}])} \le 1.05.
    \end{align}
    This finishes the proof.
\end{proof}

\subsection{Key property of pre-trained transformers on the fine-tuning data}
We highlight the following key property of the transformer after pre-training.
\begin{lemma}\label{lemma: s mag}
	With probability $1-\delta$, the following holds:
    \begin{itemize}
    \item For all $j\in[N]$, we have
	\begin{align}
		\frac{1}{m}\sum_{l=1}^{m}\sigma(\langle \mathbf{w}^{(T_p+0)}_{\mathcal{I}(\mathbf{a}_j),l},\Xi^{(T_p+0)}[\mathbf{s}_j\;\;\mathbf{p}]\rangle) =\Omega\left(\frac{\log(d)}{\lambda}\right).
	\end{align}
        \item For all $j\in[N]$ and $i\in\mathcal{B}_j$, we have 
        \begin{align}
            \max_{i\in\mathcal{B}_j}\!\frac{1}{m}\!\sum_{l=1}^m \!\sigma(\langle \mathbf{w}^{(T_p+0)}_{\mathcal{I}(\mathbf{r}_i),l},\Xi^{(T_p+0)}([\mathbf{s}_j\;\;\!\mathbf{p}])\rangle) \!-\! \frac{1}{m}\!\sum_{l=1}^m \!\sigma(\langle \mathbf{w}^{(T_p+0)}_{\mathcal{I}(\mathbf{a}_j),l},\Xi^{(T_p+0)}([\mathbf{s}_j\;\;\!\mathbf{p}])\rangle) \!=\! \mathcal{O}\!\left(\!\frac{\log(d)}{\lambda}\!\right).
        \end{align}
    \end{itemize}
\end{lemma}
    
\begin{proof}
    By the fifth statement of Lemma \ref{lemma: pt_stage3} and the fact that for all $l\in[m]$,
    \begin{align}\label{equ: aux_p}
         \langle \mathbf{w}^{(T_p+0)}_{\mathcal{I}(\mathbf{a}_j),l},\mathbf{p}\rangle = \tilde{\mathcal{O}}(\sigma_0).      
    \end{align}
    the first statement holds.
    
    Next, prove the second statement.
    We first derive the increment properties for iteration $t'$ satisfying 
    \begin{align}\label{equ: pt_aux}
        \frac{1}{m}\sum_{l=1}^m \sigma(\langle \mathbf{w}^{(t')}_{\mathcal{I}(\mathbf{r}_i),l},\mathbf{s}_j\rangle) - \frac{1}{m}\sum_{l=1}^m \sigma(\langle \mathbf{w}^{(t')}_{\mathcal{I}(\mathbf{a}_j),l},\mathbf{s}_j\rangle) \ge \frac{\log(d)}{\lambda}.
    \end{align}
   For all $i\in\mathcal{B}_j$ and arbitrary $l_i\in[m]$, we have
    \begin{align}
        \max_{i\in\mathcal{B}_j} \left(\langle \mathbf{w}_{\mathcal{I}(\mathbf{r}_i),l_i}^{(t'+1)},\mathbf{s}_j \rangle -  \langle \mathbf{w}_{\mathcal{I}(\mathbf{r}_i),l_i}^{(t')},\mathbf{s}_j \rangle\right) \le 2\sum_{i\in\mathcal{B}_j}(1-\tilde{K}(j)\text{logit}_{\mathcal{I}(\mathbf{r}_i)}^{(t')}([\mathbf{o}\;\;\mathbf{s}_j])) \le 2\tilde{K}(j)\frac{\eta^{(t')}\lambda}{mnd}.
    \end{align}
    We also have
    \begin{equation}
    \begin{aligned}
        \langle \mathbf{w}_{\mathcal{I}(\mathbf{a}_j),l}^{(t'+1)},\mathbf{s}_j \rangle -  \langle \mathbf{w}_{\mathcal{I}(\mathbf{a}_j),l}^{(t')},\mathbf{s}_j \rangle \ge - 2\tilde{K}(j)\frac{\eta^{(t')}\lambda}{mn}\text{logit}_{\mathcal{I}(\mathbf{a}_j)}^{(t')}([\mathbf{o}\;\;\mathbf{s}_j]) \ge - 2\tilde{K}(j) \frac{\eta^{(t')}\lambda}{mnd},
    \end{aligned}
    \end{equation}
    for all $l\in[m]$.
    Therefore, we have
    \begin{equation}
    \begin{aligned}
    &\frac{1}{m}\sum_{i\in\mathcal{B}_j}\sum_{l=1}^m\left(\langle \mathbf{w}_{\mathcal{I}(\mathbf{r}_i),l_i}^{(T_p)},\mathbf{s}_j \rangle - \langle \mathbf{w}_{\mathcal{I}(\mathbf{r}_i),l_i}^{(t')},\mathbf{s}_j \rangle\right) - \frac{1}{m}\sum_{l=1}^m\left(\langle \mathbf{w}_{\mathcal{I}(\mathbf{a}_j),l}^{(T_p)},\mathbf{s}_j \rangle -  \langle \mathbf{w}_{\mathcal{I}(\mathbf{a}_j),l}^{(t')},\mathbf{s}_j \rangle\right)\\
    \le&   2\tilde{K}(j) \frac{\eta_2\lambda}{mnd}T_2 - 2\tilde{K}(j) \frac{\eta_3\lambda}{mnd}(T_p-T_2)\\
    =& \mathcal{O}\left(\frac{\log(d)}{\lambda}\right).
    \end{aligned}
    \end{equation}
    Combing with (\ref{equ: aux_p}) and (\ref{equ: pt_aux}), we can conclude that
        \begin{align}
            \max_{i\in\mathcal{B}_j}\!\frac{1}{m}\!\sum_{l=1}^m \!\sigma(\langle \mathbf{w}^{(T_p+0)}_{\mathcal{I}(\mathbf{r}_i),l},\Xi^{(T_p+0)}([\mathbf{s}_j\;\;\!\mathbf{p}])\rangle) \!-\! \frac{1}{m}\!\sum_{l=1}^m \!\sigma(\langle \mathbf{w}^{(T_p+0)}_{\mathcal{I}(\mathbf{a}_j),l},\Xi^{(T_p+0)}([\mathbf{s}_j\;\;\!\mathbf{p}])\rangle) \!=\! \mathcal{O}\!\left(\!\frac{\log(d)}{\lambda}\!\right).
        \end{align}
        This finishes the proof.
\end{proof}

\subsection{Useful lemmas in full \ac{ft}}\label{appendix: activation_pattern}
In this subsection, we characterize the feature learning during full \ac{ft}.
We first define the following set for indexing the subjects in $\mathcal{Q}_s$:
\begin{align}
    \mathcal{J} := \{j\in[N] :([\mathbf{s}_j\:\:\mathbf{p}],\mathcal{I}(\mathbf{a}_j))\in\mathcal{Q}_s\}.
\end{align}
\begin{lemma}\label{lemma: mag_n1}
	During full \ac{ft}, $t_f\in[T_f]$, the following holds:
    \begin{itemize}
        \item For all $k\notin \{\mathcal{I}(\mathbf{a}_j)\}_{j\in\mathcal{J}}$ and $l\in[m]$,
        \begin{align}
		\langle\mathbf{w}^{(T_p+t_f)}_{k,l}, \mathbf{p}\rangle =\tilde{\mathcal{O}}\left(\sigma_0\right).
	\end{align}
    \item For all $j \in [N]$, $k\neq j$ and $l\in[m]$, we have
        \begin{align}
		\langle\mathbf{w}^{(T_p+t_f)}_{\mathcal{I}(\mathbf{a}_j),l}, \mathbf{s}_k\rangle =\tilde{\mathcal{O}}\left(\sigma_0\right).
	\end{align}
    \end{itemize}
\end{lemma}
\begin{proof}
We prove the statements sequentially.

\textbf{Proof of the first statement:}
For all $k\notin \{\mathcal{I}(\mathbf{a}_j)\}_{j\in\mathcal{J}}$ and $l\in[m]$, we have
\begin{equation}
\begin{aligned}
    &\langle\mathbf{w}^{(T_p+t_f+1)}_{k,l},\mathbf{p}\rangle - \langle\mathbf{w}^{(T_p+t_f)}_{k,l},\mathbf{p}\rangle\\
    =& 
        -\sum_{q\in\mathcal{J}}\Theta(1)\frac{\eta_f\lambda}{mn}\mathbb{I}(\langle\mathbf{w}^{(T_p+t_f)}_{k,l},\Xi^{(T_p+t_f)}([\mathbf{s}_q\;\;\mathbf{p}])\rangle>0)\text{logit}_{k}^{(T_p+t_f)}([\mathbf{s}_q\;\;\mathbf{p}])\\
        \le& 0.
\end{aligned}
\end{equation}
Therefore, we have
\begin{equation}
    \begin{aligned}
         \langle\mathbf{w}^{(T_p+t_f)}_{k,l},\mathbf{p}\rangle \le  \langle\mathbf{w}^{(T_p+0)}_{k,l},\mathbf{p}\rangle = \tilde{\mathcal{O}}(\sigma_0).
    \end{aligned}
    \end{equation}
This completes the proof.

\textbf{Proof of the second statement:}
 During fine-tuning $t_f\in[0,T_f-1]$, for all $l\in[m]$, $j\in[N]$, and $k\neq j$, we have
    \begin{equation}
    \begin{aligned}
        &\langle\mathbf{w}^{(T_p+t_f+1)}_{\mathcal{I}(\mathbf{a}_j),l},\mathbf{s}_k\rangle - \langle\mathbf{w}^{(T_p+t_f)}_{\mathcal{I}(\mathbf{a}_j),l},\mathbf{s}_k\rangle\\
 =& 
        -\Theta(1)\frac{\eta_f\lambda}{mn}\mathbb{I}(\langle\mathbf{w}^{(T_p+t_f)}_{\mathcal{I}(\mathbf{a}_j),l},\Xi^{(T_p+t_f)}([\mathbf{s}_k\;\;\mathbf{p}])\rangle>0)\text{logit}_{\mathcal{I}(\mathbf{a}_j)}^{(T_p+t_f)}([\mathbf{s}_k\;\;\mathbf{p}])\\
        \le& 0.
    \end{aligned}
    \end{equation}
    Therefore, by Lemma \ref{lemma: mag_n}, for all $t_f \in [T_f]$, $l\in[m]$, and $j\in[N]$, $k\neq j$, we have
\begin{equation}
    \begin{aligned}
         \langle\mathbf{w}^{(T_p+t_f)}_{\mathcal{I}(\mathbf{a}_j),l},\mathbf{s}_k\rangle \le  \langle\mathbf{w}^{(T_p+0)}_{\mathcal{I}(\mathbf{a}_j),l},\mathbf{s}_k\rangle = \tilde{\mathcal{O}}(\sigma_0).
    \end{aligned}
    \end{equation}
This completes the proof.
\end{proof}
\begin{lemma}
    For all $l\in\mathcal{S}^{(0)}_{\text{s},i,j}$, we have
    \begin{align}
        \langle \mathbf{w}^{(T_p+0)}_{\mathcal{I}(\mathbf{r}_k),l},\Xi^{(T_p+0)}([\mathbf{s}_j\;\;\mathbf{p}]) \rangle = \Omega\left(\frac{\log(d)}{\lambda}\right),
    \end{align}
    and for all $l\in\cup_{i\in\mathcal{B}_j}\mathcal{S}^{(0)}_{\text{r},j,i}$,
    \begin{align}
        \langle \mathbf{w}^{(T_p+0)}_{\mathcal{I}(\mathbf{a}_j),l},\Xi^{(T_p+0)}([\mathbf{s}_j\;\;\mathbf{p}]) \rangle = \Omega\left(\frac{\log(d)}{\lambda}\right),
    \end{align}
\end{lemma}
\begin{proof}
    By Lemma \ref{lemma: s mag}, at the end of pre-training, for all $l\in\mathcal{S}^{(0)}_{\text{s},i,j}$, we have
    \begin{align}
        \langle \mathbf{w}^{(T_p+0)}_{\mathcal{I}(\mathbf{r}_k),l},\mathbf{s}_j \rangle = \Omega\left(\frac{\log(d)}{\lambda}\right),
    \end{align}
    and 
     \begin{align}
        \langle \mathbf{w}^{(T_p+0)}_{\mathcal{I}(\mathbf{r}_k),l},\mathbf{p} \rangle =  \tilde{O}(\sigma_0).
    \end{align}
    Then, by Lemma \ref{lemma: attn_fft} we have 
    \begin{align}
        \langle \mathbf{w}^{(T_p+0)}_{\mathcal{I}(\mathbf{r}_k),l},\Xi^{(T_p+0)}[\mathbf{s}_j\;\;\mathbf{p}] \rangle = \Omega\left(\frac{\log(d)}{\lambda}\right),
    \end{align}
    Additionally, under Condition \ref{condition: condition}, for all $l\in\cup_{i\in\mathcal{B}_j}\mathcal{S}^{(0)}_{\text{r},j,i}$,
    \begin{align}
        \langle \mathbf{w}^{(T_p+0)}_{\mathcal{I}(\mathbf{a}_j),l},\mathbf{s}_j \rangle =\Omega\left(\frac{\log(d)}{\lambda}\right),
    \end{align}
    and 
    \begin{align}
        \langle \mathbf{w}^{(T_p+0)}_{\mathcal{I}(\mathbf{a}_j),l},\mathbf{p} \rangle = \tilde{O}(\sigma_0),
    \end{align}
    As a result, we have 
    \begin{align}
        \langle \mathbf{w}^{(T_p+0)}_{\mathcal{I}(\mathbf{a}_j),l},\Xi^{(T_p+0)}([\mathbf{s}_j\;\;\mathbf{p}]) \rangle = \Omega\left(\frac{\log(d)}{\lambda}\right).
    \end{align}
    This completes the proof.
\end{proof}

\begin{lemma}\label{lemma: mag2}
	For all $j\in [N]$ and $l\in\cup_{i\in\mathcal{B}_j}\mathcal{S}_{\text{r},i,j}^{(0)}$, during full \ac{ft}, $t_f\in[T_f]$, the following holds:
	\begin{align}
		\langle\mathbf{w}^{(T_p+t_f)}_{\mathcal{I}(\mathbf{a}_j),l}, \mathbf{s}_j\rangle =\Omega\left(\frac{\log(d)}{\lambda}\right).
	\end{align}
\end{lemma}
\begin{proof}
	For all $t_f\in[0,T_f-1]$, $l\in[m]$ and $j\in [N]$, the update of $\mathbf{w}^{(T_p+t_f)}_{\mathcal{I}(\mathbf{a}_j),l}$ satisfies
    \begin{equation}
	\begin{aligned}
		&\langle\mathbf{w}^{(T_p+t_f+1)}_{\mathcal{I}(\mathbf{a}_j),l},\mathbf{s}_j\rangle - \langle\mathbf{w}^{(T_p+t_f)}_{\mathcal{I}(\mathbf{a}_j),l},\mathbf{s}_j\rangle\\
        =& \Theta(1)\frac{\lambda\eta_f}{\beta N_f}\mathbb{I}(\langle\mathbf{w}^{(T_p+t_f)}_{\mathcal{I}(\mathbf{a}_j),l}, \Xi^{(T_f+t_f)}([\mathbf{s}_j\;\;\mathbf{p}])\rangle >0)\mathbb{I}(j\in\mathcal{J})(1-\text{logit}^{(T_p+t_f)}_{\mathcal{I}(\mathbf{a}_j)}([\mathbf{s}_j\:\:\mathbf{p}]))
        \\
        \ge& 0.
	\end{aligned}
    \end{equation}
Then, for all $t_f\in[0,T_f-1]$, $l\in[m]$ and $j\in [N]$, we have
\begin{align}\label{equ: ft_aux1}
    \langle\mathbf{w}^{(T_p+t_f)}_{\mathcal{I}(\mathbf{a}_j),l},\mathbf{s}_j\rangle \ge \langle\mathbf{w}^{(T_p+0)}_{\mathcal{I}(\mathbf{a}_j),l},\mathbf{s}_j\rangle.
\end{align}
Based on Lemma \ref{lemma: pt_stage3}, we have
	\begin{align}
		\langle\mathbf{w}^{(T_p+t_f)}_{\mathcal{I}(\mathbf{a}_j),l}, \mathbf{s}_j\rangle =\Omega\left(\frac{\log(d)}{\lambda}\right),
	\end{align}
for all $j\in [N]$ and $l\in\cup_{i\in\mathcal{B}_j}\mathcal{S}_{\text{r},i,j}^{(0)}$.

This completes the proof.
\end{proof}

\begin{lemma}\label{lemma: act}
   When $\beta N_f \ge C_1|\mathcal{R}|$, under Condition \ref{condition: condition}, with probability $1-\delta$, for all $l\in[m]$ and $i\in\mathcal{R}$, we have
    \begin{align}
        \mathbb{P}\left[ \sum_{j\in\mathcal{D}_i} \mathbb{I}(l\in \mathcal{S}^{(0)}_{\text{s},i,j})\ge 0.15\frac{\beta N_f K}{|\mathcal{R}|}\right] \ge 1-|\mathcal{R}|\exp\left(-\frac{N_f\beta K^2}{2|\mathcal{R}|^2}\right).
    \end{align}
\end{lemma}
\begin{proof}
    By Lemma \ref{lemma: select_num}, with probability at least $1-|\mathcal{R}|\exp\left(-\frac{N_f\beta K^2}{(2|\mathcal{R}|^2)}\right)$, each $i\in\mathcal{R}$ satisfies
    \begin{align}
        \sum_{j\in\mathcal{J}} \mathbb{I}(i\in\mathcal{B}_j) \ge 0.5 \frac{\beta N_f K}{|\mathcal{R}|}.
    \end{align}
    Then, by Lemma \ref{lemma: cover1}, under Condition \ref{condition: condition}, with probability $1-\delta$, for each $l\in[m]$ and $i\in\mathcal{R}$, we have
    \begin{align}
        \sum_{j\in\mathcal{D}_i} \mathbb{I}(l\in \mathcal{S}^{(0)}_{\text{s},i,j}) \ge 0.15 \frac{\beta N_f K}{|\mathcal{R}|}.
    \end{align}
    Combining with the activation patterns finishes the proof.
\end{proof}

\subsection{Proof of the result for full \ac{ft} in Statement 2 of Theorem \ref{theorem: main}}\label{appendix: proof statement 2}
We denote 
\begin{align}
    \mathcal{J}_k := \{j\in[N] :([\mathbf{s}_j\:\:\mathbf{p}],\mathcal{I}(\mathbf{a}_j))\in\mathcal{Q}_s, k\in\mathcal{B}_j\}.
\end{align}
After pre-training, the first iteration of full \ac{ft} satisfies
\begin{align}
	\sum_{j\in\mathcal{J}_i}\text{logit}_{\mathcal{I}(\mathbf{r}_i)}^{(T_p+0)}([\mathbf{s}_j\:\:\mathbf{p}]) = \Theta\left(\frac{N_f\beta}{K}\cdot \frac{K}{|\mathcal{R}|}\right) = \Theta\left(\frac{N_f\beta}{|\mathcal{R}|}\right),
\end{align}
and 
\begin{align}
	1-\text{logit}^{(T_p+0)}_{\mathcal{I}(\mathbf{a}_j)}([\mathbf{s}_j\:\:\mathbf{p}]) = \Theta(1).
\end{align}

By Lemmas \ref{lemma: select_num} and \ref{lemma: act}, with probability $1-|\mathcal{R}|\exp\left(-\frac{N_f\beta K^2}{(2|\mathcal{R}|^2)}\right)$, for all $i\in\mathcal
{R}$ and $k\in[m]$, the model updates satisfy
\begin{equation}
\begin{aligned}
	&\langle\mathbf{w}^{(T_p+1)}_{\mathcal{I}(\mathbf{r}_i),k}, \mathbf{p}\rangle - \langle\mathbf{w}^{(T_p+0)}_{\mathcal{I}(\mathbf{r}_i),k},\mathbf{p}\rangle\\
    =&-\Theta\left(\frac{\lambda\eta_f}{N_f\beta m}\right)\sum_{j\in\mathcal{J}}\text{logit}^{(T_p+0)}_{\mathcal{I}(\mathbf{r}_i)}([\mathbf{s}_j\;\;\mathbf{p}])  \mathbb{I}(\langle\mathbf{w}^{(T_p+0)}_{\mathcal{I}(\mathbf{r}_i),k},\Xi^{(T_p+0)}([\mathbf{s}_j\;\;\mathbf{p}])  \rangle > 0)\\
    =& -\Theta\left(\frac{\eta_f\lambda}{mN_f\beta}\frac{N_f\beta K}{|\mathcal{R}|}\frac{1}{K}\right)\\
    =& -\Theta\left(\frac{\lambda\eta_f }{|\mathcal{R}|m}\right),
\end{aligned}
\end{equation}
and for all $i\in\mathcal{J}$,
\begin{equation}
\begin{aligned}
    &\frac{1}{m}\sum_{k=1}^{m}\langle\mathbf{w}^{(T_p+1)}_{\mathcal{I}(\mathbf{a}_i),k}, \mathbf{p}\rangle - \frac{1}{m}\sum_{k=1}^{m}\langle\mathbf{w}^{(T_p+0)}_{\mathcal{I}(\mathbf{a}_i),k},\mathbf{p}\rangle\\
    =& \Theta\left(\frac{\lambda\eta_f}{N_f\beta m}\right)(1-\text{logit}_{\mathcal{I}(\mathbf{a}_i)}^{(T_p+0)}([\mathbf{s}_i\;\;\mathbf{p}]))\\
    =& \Theta\left(\frac{\lambda\eta_f}{N_f\beta m}\right).
\end{aligned}
\end{equation}
Similarly, we have
\begin{align}
	\frac{1}{m}\sum_{k=1}^{m}\langle\mathbf{w}^{(T_p+1)}_{\mathcal{I}(\mathbf{a}_i),k}, \mathbf{s}_i\rangle - \frac{1}{m}\sum_{k=1}^{m}\langle\mathbf{w}^{(T_p+0)}_{\mathcal{I}(\mathbf{a}_i),k},\mathbf{s}_i\rangle = \Theta\left(\frac{\lambda\eta_f}{N_f\beta m}\right).
\end{align}
When $\eta_f = \Theta(\frac{m|\mathcal{R}|\log(d)}{\lambda^2})$, we have that
\begin{align}
	\frac{1}{m}\sum_{k=1}^{m}\langle\mathbf{w}^{(T_p+1)}_{\mathcal{I}(\mathbf{r}_i),k}, \mathbf{p}\rangle - \frac{1}{m}\sum_{k=1}^{m}\langle\mathbf{w}^{(T_p+0)}_{\mathcal{I}(\mathbf{r}_i),k},\mathbf{p}\rangle = -\Theta\left(\frac{\eta_f\lambda}{|\mathcal{R}|m}\right) = -\tilde{\Theta}\left(\frac{\log(d)}{\lambda}\right),
\end{align}
and
\begin{align}
	\frac{1}{m}\sum_{k=1}^{m}\langle\mathbf{w}^{(T_p+1)}_{\mathcal{I}(\mathbf{a}_i),k}, \mathbf{p}\rangle - \frac{1}{m}\sum_{k=1}^{m}\langle\mathbf{w}^{(T_p+0)}_{\mathcal{I}(\mathbf{a}_i),k},\mathbf{p}\rangle = \Theta\left(\frac{\lambda\eta_f}{N_f\beta m}\right) = \Theta\left(\frac{|\mathcal{R}|\log(d)}{\lambda N_f\beta}\right).
\end{align}
Consequently, when  $\beta N_f> C_1 |\mathcal{R}|$, after one iteration, we already have
\begin{align}
	\max_{k\in\mathcal{R}} x^\text{output}_{\mathcal{I}(\mathbf{r}_k)}(\mathbf{W}^{(T_p+1)},\mathbf{Z}^{(T_p+1)},[\mathbf{s}_j\:\:\mathbf{p}]) < x^\text{output}_{\mathcal{I}(\mathbf{a}_j)}(\mathbf{W}^{(T_p+1)},\mathbf{Z}^{(T_p+1)},[\mathbf{s}_j\:\:\mathbf{p}]),
\end{align}
\begin{align}
	\max_{k\in[N]\backslash\{j\}}x^\text{output}_{\mathcal{I}(\mathbf{a}_k)}(\mathbf{W}^{(T_p+1)},\mathbf{Z}^{(T_p+1)},[\mathbf{s}_j\:\:\mathbf{p}]) < x^\text{output}_{\mathcal{I}(\mathbf{a}_j)}(\mathbf{W}^{(T_p+1)},\mathbf{Z}^{(T_p+1)},[\mathbf{s}_j\:\:\mathbf{p}]),
\end{align}
and 
\begin{align}
	\max_{l\notin\{\mathcal{I}(\mathbf{a}_j)\}_{j\in[N]}\cup\{\mathcal{I}(\mathbf{r}_i)\}_{i\in \mathcal{R}}}x^\text{output}_{l}(\mathbf{W}^{(T_p+1)},\mathbf{Z}^{(T_p+1)},[\mathbf{s}_j\:\:\mathbf{p}]) < x^\text{output}_{\mathcal{I}(\mathbf{a}_j)}(\mathbf{W}^{(T_p+1)},\mathbf{Z}^{(T_p+1)},[\mathbf{s}_j\:\:\mathbf{p}]),
\end{align}

Within $T_f = \mathcal{O}(d^{-0.01}\frac{N_f\beta}{|\mathcal{R}|})$, we have
\begin{align}
    	\frac{1}{m}\sum_{k=1}^{m}\sigma(\langle\mathbf{w}^{(T_p+T_f)}_{\mathcal{I}(\mathbf{a}_i),k}, \mathbf{p}\rangle) - \frac{1}{m}\sum_{k=1}^{m}\sigma(\langle\mathbf{w}^{(T_p+T_f)}_{\mathcal{I}(\mathbf{a}_i),k},\mathbf{p}\rangle) = \mathcal{O}\left(\frac{|\mathcal{R}|\log(d)}{\lambda N_f\beta}T_f\right) = \mathcal{O}\left(\frac{d^{-0.01}\log(d)}{\lambda}\right).
\end{align}

Hence, by Lemmas \ref{lemma: mag_n1} and \ref{lemma: mag2}, within $T_f = \mathcal{O}(d^{-0.01}\frac{N_f\beta}{|\mathcal{R}|})$ the outputs of the transformer also satisfy 
\begin{align}
	\max_{k\in\mathcal{R}} x^\text{output}_{\mathcal{I}(\mathbf{r}_k)}(\mathbf{W}^{(T_p+t_f)},\mathbf{Z}^{(T_p+t_f)},[\mathbf{s}_j\:\:\mathbf{p}]) < x^\text{output}_{\mathcal{I}(\mathbf{a}_j)}(\mathbf{W}^{(T_p+t_f)},\mathbf{Z}^{(T_p+t_f)},[\mathbf{s}_j\:\:\mathbf{p}]),
\end{align}
\begin{align}
	\max_{k\in[N]\backslash\{j\}}x^\text{output}_{\mathcal{I}(\mathbf{a}_k)}(\mathbf{W}^{(T_p+t_f)},\mathbf{Z}^{(T_p+t_f)},[\mathbf{s}_j\:\:\mathbf{p}]) < x^\text{output}_{\mathcal{I}(\mathbf{a}_j)}(\mathbf{W}^{(T_p+t_f)},\mathbf{Z}^{(T_p+t_f)},[\mathbf{s}_j\:\:\mathbf{p}]),
\end{align}
and 
\begin{align}
	\max_{l\notin\{\mathcal{I}(\mathbf{a}_j)\}_{j\in[N]}\cup\{\mathcal{I}(\mathbf{r}_i)\}_{i\in \mathcal{R}}}x^\text{output}_{l}(\mathbf{W}^{(T_p+t_f)},\mathbf{Z}^{(T_p+t_f)},[\mathbf{s}_j\:\:\mathbf{p}]) < x^\text{output}_{\mathcal{I}(\mathbf{a}_j)}(\mathbf{W}^{(T_p+t_f)},\mathbf{Z}^{(T_p+t_f)},[\mathbf{s}_j\:\:\mathbf{p}]),
\end{align}
for all $t_f\in[1,T_f]$.

This finishes the proof and concludes that
\begin{align}
	\mathcal{L}_e(\mathbf{W}^{(T_p+T_f)},\mathbf{Z}^{(T_p+T_f)}) \le |\mathcal{R}|\exp\left(-\frac{N_f\beta K^2}{2|\mathcal{R}|^2}\right).
\end{align}

\subsection{Proof of the result for full \ac{ft} in Statement 3 of Theorem \ref{theorem: main}}\label{appendix: proof statement 3}
When $\beta N_fK/|\mathcal{R}| \le C_3$ with $0<C_3<1$, there are more than $(1-C_3)|\mathcal{R}|$ relation elements, denoted as set $\mathcal{C}$, not been covered in the \ac{ft} training set.
For the elements not covered, $j\in\mathcal{C}$, we have
\begin{align}
	\frac{1}{m}\sum_{k=1}^{m}\sigma(\langle\mathbf{w}^{(T_p+t_f)}_{\mathcal{I}(\mathbf{r}_j),k}, \mathbf{p}\rangle)\le \frac{1}{m}\sum_{k=1}^{m}\sigma(\langle\mathbf{w}^{(T_p+0)}_{\mathcal{I}(\mathbf{r}_j),k}, \mathbf{p}\rangle)= \tilde{\mathcal{O}}\left(\frac{\sigma_0}{\sqrt{m}}\right).
\end{align}

For all $i\in\mathcal{C}$ and $\mathcal{D}_i$ and any $t_f$, we have
\begin{align}
    \frac{1}{m}\sum_{l=1}^m\sigma(\langle\mathbf{w}^{(T_p+t_f)}_{\mathcal{I}(\mathbf{r}_i),l},[\mathbf{s}_j\;\;\mathbf{p}]\rangle) = \Theta(1),
\end{align}
resulting in 
\begin{align}
     \frac{1}{m}\sum_{l=1}^m\sigma(\langle\mathbf{w}^{(T_p+t_f)}_{\mathcal{I}(\mathbf{r}_i),l},[\mathbf{s}_j\;\;\mathbf{p}]\rangle) >  \frac{1}{m}\sum_{l=1}^m\sigma(\langle\mathbf{w}^{(T_p+t_f)}_{\mathcal{I}(\mathbf{a}_j),l},[\mathbf{s}_j\;\;\mathbf{p}]\rangle).
\end{align}

For the remaining $(1-\beta)N_f+N_r$ data $([\mathbf{s}_j\;\;\mathbf{p}],\mathcal{I}(\mathbf{a}_j))\in\mathcal{Q}\backslash \mathcal{Q}_s$, we have
\begin{align}
	\mathbb{P}[j\in\mathcal{D}_i, \forall i\in \mathcal{C}] \ge 1- C_3^{\tilde{K}(j)}.
\end{align}

Therefore, we have
\begin{align}
	\mathcal{L}_e(\mathbf{W}^{(T_p+T_f)},\mathbf{Z}^{(T_p+T_f)}) \ge 1-C_3.
\end{align}
This completes the proof.

\section{Proofs of low-rank fine-tuning}\label{app: low-rank FT}
We consider a low-rank fine-tuning approach, where the model update is based on the best rank-1 approximation of the full gradient.
Specifically, in \ac{ft} iteration $t_f$, the model update is
\begin{align}\nonumber
	&\mathbf{Z}^{(T_p+t_f+1)} = \mathbf{Z}^{(T_p+t_f)} - \frac{\eta_f}{n}\text{RA}_1\left(\sum_{(\mathbf{X},y)\in\mathcal{Q}_s} \nabla_{\mathbf{Z}^{(T_p+t_f)}} \mathcal{L}(\mathbf{W}^{(T_p+t_f)},\mathbf{Z}^{(T_p+t_f)},\mathbf{X},y)\right),\\\nonumber
	&\mathbf{W}^{(T_p+t_f+1)} = \mathbf{W}^{(T_p+t_f)} - \frac{\eta_f}{n}\text{RA}_1\left(\sum_{(\mathbf{X},y)\in\mathcal{Q}_s} \nabla_{\mathbf{W}^{(T_p+t_f)}} \mathcal{L}(\mathbf{W}^{(T_p+t_f)},\mathbf{Z}^{(T_p+t_f)},\mathbf{X},y)\right),
\end{align}
where $\text{RA}_1(\mathbf{B})$ denotes the best rank-1 approximation of $\mathbf{B}$.

For a matrix $\mathbf{A}$, its best rank-1 approximation $\tilde{\mathbf{A}}$ is found by solving
\begin{align}
    \tilde{\mathbf{A}} = \mathbf{a}\mathbf{b}^\top,
\end{align}
where
\begin{align}
    \mathbf{a},\mathbf{b} = \arg\min \left\|\mathbf{A} - \mathbf{a}\mathbf{b}^\top\right\|_F.
\end{align}
Notably, by the Eckart–Young–Mirsky theorem, 
\begin{align}
    \tilde{\mathbf{A}} = \sigma_1\mathbf{u}_1\mathbf{v}_1^\top,
\end{align}
where $\sigma_1$ is the largest eigenvalue of $\mathbf{A}$ and $\mathbf{u}_1,\mathbf{v}_1$ are the corresponding left and right singular vectors.

\subsection{Useful lemmas}
First, we present two lemmas that characterize some properties of the best rank-1 approximation.
Denote the full gradient at \ac{ft} iteration $t_f$ as 
\begin{align}
	\mathbf{G}^{(t_f)} = \begin{bmatrix}
	\nabla_{\mathbf{w}^{(T_p+t_f)}_{1,1}} \mathcal{L}_\mathcal{F}(\mathbf{W}^{(T_p+t_f)},\mathbf{Z}^{(T_p+t_f)})\;\;\cdots\;\; \nabla_{\mathbf{w}^{(T_p+t_f)}_{d,m}} \mathcal{L}_\mathcal{F}(\mathbf{W}^{(T_p+t_f)},\mathbf{Z}^{(T_p+t_f)})
	\end{bmatrix}^\top.
\end{align}

\begin{lemma}\label{lemma: sft grad 1}

The best rank-1 approximation of the full gradient $\mathbf{G}^{(t_f)}$ can be decomposed as
\begin{align}
		&\mathbf{G}^{(t_f)} = \begin{bmatrix}
		    \mathbf{v}_p\:\: \mathbf{v}_1\:\cdots\:\mathbf{v}_N
	\end{bmatrix}
        \begin{bmatrix}
            \mathbf{p}\:\:\mathbf{s}_1\:\cdots\:\mathbf{s}_N
        \end{bmatrix}^\top.
\end{align}
The matrix $\mathbf{G}^{(t_f)}$ satisfies
\begin{align}
    \left\|\text{RA}_1(\mathbf{G}^{(t_f)}) - \mathbf{v}_p\mathbf{p}^\top\right\|_F =\mathcal{O}\left(\frac{\lambda}{\sqrt{mN_f\beta}}\right).
\end{align}
\end{lemma}
\begin{proof}
	By the Eckart–Young–Mirsky theorem, we derive the best rank-1 approximation of the full gradient $\mathbf{G}^{(t_f)}$ based on SVD.
	The gradient can be represented as follows. For clarity, we add $|$ to seprate the elements in a vector in (\ref{equ: grad1}) and (\ref{equ: grad2}).
    For all $i\in\mathcal{R}$ and $l\in[m]$, we rearrange the gradient as follows.
	\begin{equation}\label{equ: grad1}
	\begin{aligned}
		&\nabla_{\mathbf{w}^{(T_p+t_f)}_{\mathcal{I}(\mathbf{r}_i),l}} \mathcal{L}_\mathcal{F}(\mathbf{W}^{(T_p+t_f)},\mathbf{Z}^{(T_p+t_f)}) \\
		=& \frac{\lambda}{mN_f\beta}\begin{bmatrix}
			\mathbf{p}\: |\:\mathbf{s}_1\:|\cdots|\:\mathbf{s}_N
		\end{bmatrix}\\
        &\left[
			-\underbrace{\sum_{([\mathbf{s}_j\:\:\mathbf{p}],\mathcal{I}(\mathbf{a}_j))\in\mathcal{Q}_s,j\in\mathcal{J}_i}\mathbb{I}(\langle\mathbf{w}^{(T_p+t_f)}_{\mathcal{I}(\mathbf{r}_i),l},\Xi^{(T_p+t_f)}([\mathbf{s}_j\:\:\mathbf{p}])\rangle)\text{logit}^{(T_p+t_f)}_{\mathcal{I}(\mathbf{r}_i)}([\mathbf{s}_j\:\:\mathbf{p}])}_{u_1}\;|\;\right.\\
            &-\mathbb{I}(1\in\mathcal{J})\mathbb{I}(\langle\mathbf{w}^{(T_p+t_f)}_{\mathcal{I}(\mathbf{r}_i),l},\Xi^{(T_p+t_f)}([\mathbf{s}_1\:\:\mathbf{p}])\rangle)\text{logit}^{(T_p+t_f)}_{\mathcal{I}(\mathbf{r}_i)}([\mathbf{s}_1\:\:\mathbf{p}])\;|\;\cdots\;|\;\\
 &\left.-\mathbb{I}(N\in\mathcal{J})\mathbb{I}(\langle\mathbf{w}^{(T_p+t_f)}_{\mathcal{I}(\mathbf{r}_i),l},\Xi^{(T_p+t_f)}([\mathbf{s}_N\:\:\mathbf{p}])\rangle)\text{logit}^{(T_p+t_f)}_{\mathcal{I}(\mathbf{r}_i)}([\mathbf{s}_N\:\:\mathbf{p}])
		\right]^\top.
	\end{aligned}
	\end{equation}
Here, it is worth noting that the norm of the components corresponding to $\mathbf{s}_j$ for all $j\in\mathcal{Q}_s$ is bounded by
\begin{align}
    &\left\|\nabla_{\mathbf{w}^{(T_p+t_f)}_{\mathcal{I}(\mathbf{r}_i),l}} \mathcal{L}_\mathcal{F}(\mathbf{W}^{(T_p+t_f)},\mathbf{Z}^{(T_p+t_f)}) +u_1\mathbf{p} \right\|_2
    = \mathcal{O}\left(\frac{\lambda}{m\sqrt{N_f\beta}}\right).
\end{align}
    
    Then, for each $j\in \mathcal{J}$ and $l\in[m]$, we rearrange the gradient as follows.
    \begin{equation}\label{equ: grad2}
	\begin{aligned}
&\nabla_{\mathbf{w}^{(T_p+t_f)}_{\mathcal{I}(\mathbf{a}_j),l}} \mathcal{L}_\mathcal{F}(\mathbf{W}^{(T_p+t_f)},\mathbf{Z}^{(T_p+t_f)})\\
    =& \frac{\lambda}{mN_f\beta}\begin{bmatrix}
		\mathbf{p}\:\: \mathbf{s}_1\:\cdots\:\mathbf{s}_N
	\end{bmatrix} \left(\left[
		\mathbb{I}(\langle\mathbf{w}^{(T_p+t_f)}_{\mathcal{I}(\mathbf{a}_j),l},\Xi^{(T_p+t_f)}([\mathbf{s}_j\:\:\mathbf{p}])\rangle)(1-\text{logit}^{(T_p+t_f)}_{\mathcal{I}(\mathbf{a}_j)}([\mathbf{s}_j\:\:\mathbf{p}])
        )\right.\right.\\
        &\left.-\sum_{i\neq j,i\in\mathcal{J}}\mathbb{I}(\langle\mathbf{w}^{(T_p+t_f)}_{\mathcal{I}(\mathbf{a}_j),l},\Xi^{(T_p+t_f)}([\mathbf{s}_i\:\:\mathbf{p}])\rangle)\text{logit}^{(T_p+t_f)}_{\mathcal{I}(\mathbf{a}_j)}([\mathbf{s}_i\:\:\mathbf{p}])
        )\;|\right.\\
        &-\mathbb{I}(1\in\mathcal{J})\mathbb{I}(\langle\mathbf{w}^{(T_p+t_f)}_{\mathcal{I}(\mathbf{a}_j),l},\Xi^{(T_p+t_f)}([\mathbf{s}_1\:\:\mathbf{p}])\rangle)\text{logit}^{(T_p+t_f)}_{\mathcal{I}(\mathbf{a}_j)}([\mathbf{s}_1\:\:\mathbf{p}])\;|\;\cdots\;|\;  \\
        &\mathbb{I}(\langle\mathbf{w}^{(T_p+t_f)}_{\mathcal{I}(\mathbf{a}_j),l},\Xi^{(T_p+t_f)}([\mathbf{s}_j\:\:\mathbf{p}])\rangle)(1-\text{logit}^{(T_p+t_f)}_{\mathcal{I}(\mathbf{a}_j)}([\mathbf{s}_j\:\:\mathbf{p}]))\;|\;\cdots\;|\;\\
        &\left.-\mathbb{I}(N\in\mathcal{J})\mathbb{I}(\langle\mathbf{w}^{(T_p+t_f)}_{\mathcal{I}(\mathbf{a}_j),l},\Xi^{(T_p+t_f)}([\mathbf{s}_N\:\:\mathbf{p}])\rangle)\text{logit}^{(T_p+t_f)}_{\mathcal{I}(\mathbf{a}_j)}([\mathbf{s}_N\:\:\mathbf{p}])
	\right]\\
    =& \frac{\lambda}{mN_f\beta}\begin{bmatrix}
		\mathbf{p}\:\: \mathbf{s}_1\:\cdots\:\mathbf{s}_N
	\end{bmatrix} \left(\left[
		\mathbb{I}(\langle\mathbf{w}^{(T_p+t_f)}_{\mathcal{I}(\mathbf{a}_j),l},\Xi^{(T_p+t_f)}([\mathbf{s}_j\:\:\mathbf{p}])\rangle)(1-\text{logit}^{(T_p+t_f)}_{\mathcal{I}(\mathbf{a}_j)}([\mathbf{s}_j\:\:\mathbf{p}])
        )\right.\right.\\
        &\left.-\sum_{i\neq j, i\in\mathcal{J}}\mathbb{I}(\langle\mathbf{w}^{(T_p+t_f)}_{\mathcal{I}(\mathbf{a}_j),l},\Xi^{(T_p+t_f)}([\mathbf{s}_i\:\:\mathbf{p}])\rangle)\text{logit}^{(T_p+t_f)}_{\mathcal{I}(\mathbf{a}_j)}([\mathbf{s}_i\:\:\mathbf{p}])
        )\;|\; 0\;|\;\cdots\;|\; \right. \\
        &\left.\left.\mathbb{I}(\langle\mathbf{w}^{(T_p+t_f)}_{\mathcal{I}(\mathbf{a}_j),l},\Xi^{(T_p+t_f)}([\mathbf{s}_j\:\:\mathbf{p}])\rangle)(1-\text{logit}^{(T_p+t_f)}_{\mathcal{I}(\mathbf{a}_j)}([\mathbf{s}_j\:\:\mathbf{p}]))\;|\;\cdots\;|\;0)
	\right] + \boldsymbol{\epsilon}_{\mathcal{I}(\mathbf{a}_j),l}\right),
\end{aligned}
\end{equation}
where $\boldsymbol{\epsilon}_{\mathcal{I}(\mathbf{a}_j),l}$ is the vector containing the components except those for $\mathbf{p}$ and $\mathbf{s}_j$.
It is worth noting that the norm of $\boldsymbol{\epsilon}_{\mathcal{I}(\mathbf{a}_j),l}$ satisfies
\begin{align}
    \left\|\boldsymbol{\epsilon}_{\mathcal{I}(\mathbf{a}_j),l}\right\|_2 = \mathcal{O}\left(\frac{\lambda}{m\sqrt{N_f\beta} d}\right),
\end{align}
by Lemma \ref{lemma: mag_n1}.

For all $k\notin \{\mathcal{I}(\mathbf{a}_j)\}_{j\in\mathcal{J}}\cup\{\mathcal{I}(\mathbf{r}_i)\}_{i\in\mathcal{R}}$ and $k\in[d]$, $l\in[m]$, we have
\begin{equation}
\begin{aligned}
    &\nabla_{\mathbf{w}^{(T_p+t_f)}_{k,l}} \mathcal{L}_\mathcal{F}(\mathbf{W}^{(T_p+t_f)},\mathbf{Z}^{(T_p+t_f)})\\
    =& \frac{\lambda}{mN_f\beta}\begin{bmatrix}
		\mathbf{p}\:\: \mathbf{s}_1\:\cdots\:\mathbf{s}_N
	\end{bmatrix}\\
        &\left.-\sum_{i\in\mathcal{J}}\mathbb{I}(\langle\mathbf{w}^{(T_p+t_f)}_{k,l},\Xi^{(T_p+t_f)}([\mathbf{s}_i\:\:\mathbf{p}])\rangle)\text{logit}^{(T_p+t_f)}_{k}([\mathbf{s}_i\:\:\mathbf{p}])
        )\;|\right.\\
        &-\mathbb{I}(1\in\mathcal{J})\mathbb{I}(\langle\mathbf{w}^{(T_p+t_f)}_{k,l},\Xi^{(T_p+t_f)}([\mathbf{s}_1\:\:\mathbf{p}])\rangle)\text{logit}^{(T_p+t_f)}_{k}([\mathbf{s}_1\:\:\mathbf{p}])\;|\;\cdots\;|\;  \\
        &\left.-\mathbb{I}(N\in\mathcal{J})\mathbb{I}(\langle\mathbf{w}^{(T_p+t_f)}_{k,l},\Xi^{(T_p+t_f)}([\mathbf{s}_N\:\:\mathbf{p}])\rangle)\text{logit}^{(T_p+t_f)}_{k}([\mathbf{s}_N\:\:\mathbf{p}])
	\right]\\
    =&\boldsymbol{\epsilon}_{k,l}.
\end{aligned}
\end{equation}
The norm of $\nabla_{\mathbf{w}^{(T_p+t_f)}_{k,l}} \mathcal{L}_\mathcal{F}(\mathbf{W}^{(T_p+t_f)},\mathbf{Z}^{(T_p+t_f)})$ satisfies
\begin{align}
    \left\|\nabla_{\mathbf{w}^{(T_p+t_f)}_{k,l}} \mathcal{L}_\mathcal{F}(\mathbf{W}^{(T_p+t_f)},\mathbf{Z}^{(T_p+t_f)})\right\|_2 = \mathcal{O}\left(\frac{\lambda}{m\sqrt{N_f\beta}d}\right).
\end{align}

	Compactly, we decompose the full gradient as
	\begin{align}
		&\mathbf{G}^{(t_f)} = \mathbf{v}_p\mathbf{p}^\top + \mathbf{v}_{1}\mathbf{s}_1^\top + \cdots + \mathbf{v}_{N}\mathbf{s}_N^\top + \boldsymbol{\epsilon},
	\end{align}
    where $\boldsymbol{\epsilon}\in\mathbb{R}^{md\times d}$ is a concatenation of $\boldsymbol{\epsilon}_{k,l}$ for all $k\in[d]$ and $l\in[m]$ with $\boldsymbol{\epsilon}_{\mathcal{I}(\mathbf{r}_i),l} = \mathbf{0}$.
	As $\text{RA}_1(\mathbf{G}^{(t_f)})$ is the best rank-1 approximation of $\mathbf{G}^{(t_f)}$, we have that
	\begin{align}
		\left\|\mathbf{G}^{(t_f)} - \text{RA}_1(\mathbf{G}^{(t_f)})\right\|_F \le \left\|\mathbf{G}^{(t_f)} - \mathbf{v}_p\mathbf{p}^\top\right\|_F,
	\end{align}
    by definition.
	By the triangle inequality, we have
	\begin{equation}
	\begin{aligned}
		\left\|\text{RA}_1(\mathbf{G}^{(t_f)}) - \mathbf{v}_p\mathbf{p}^\top\right\|_F \le& \left\|\mathbf{G}^{(t_f)} - \text{RA}_1(\mathbf{G}^{(t_f)})\right\|_F + \left\|\mathbf{G}^{(t_f)} - \mathbf{v}_p\mathbf{p}^\top\right\|_F\\
		\le& 2\left\|\mathbf{G}^{(t_f)} - \mathbf{v}_p\mathbf{p}^\top\right\|_F.
	\end{aligned}
	\end{equation}
	The term $\left\|\mathbf{G}^{(t_f)} - \mathbf{v}_p\mathbf{p}^\top\right\|_F$ satisfies
\begin{equation}
	\begin{aligned}
		\left\|\mathbf{G}^{(t_f)} - \mathbf{v}_p\mathbf{p}^\top\right\|_F = \left\|\sum_{i=1}^{N}\mathbf{v}_i\mathbf{s}_i \right\|_F+ \left\|\boldsymbol{\epsilon}\right\|_F =& \mathcal{O}\left(\frac{\lambda}{\sqrt{N_fm\beta}}\right) + \mathcal{O}\left(\frac{\lambda}{\sqrt{mN_f\beta d}}\right)\\
        =& \mathcal{O}\left(\frac{\lambda}{\sqrt{N_fm\beta}}\right).
\end{aligned}
\end{equation}
	This completes the proof.
\end{proof}

\begin{lemma}\label{lemma: sft grad 2}
	For all $i\in[md]$, the full gradient $\mathbf{G}^{(t_f)}$ satisfies
	\begin{align}
		\left\|\mathbf{e}_i^\top\text{RA}_1(\mathbf{G}^{(t_f)})\right\|_2 \le \left\|\mathbf{e}_i^\top\mathbf{G}^{(t_f)}\right\|_2.
	\end{align}
\end{lemma}
\begin{proof}
	Suppose that the SVD decomposition of $\mathbf{G}^{(t_f)}$ as $\mathbf{U}\boldsymbol{\Sigma}\mathbf{V}^\top$.
	By the Eckart–Young–Mirsky theorem, we have that
	\begin{align}
		\text{RA}_1\left(\mathbf{G}^{(t_f)}\right) = \sigma_1\mathbf{u}_1\mathbf{v}_1^\top,
	\end{align}
        where $\sigma_1\ge \sigma_2 \ge \cdots \ge \sigma_{d}$ are the singular values of $\mathbf{G}^{(t_f)}$.
	Then, we have
	\begin{align}
		\mathbf{e}_i^\top\mathbf{G}^{(t_f)} = \mathbf{e}_i^\top\sum_{j=1}^{d} \sigma_j\mathbf{u}_j\mathbf{v}_j^\top, \mathbf{e}_i^\top\text{RA}_1(\mathbf{G}^{(t_f)}) = \mathbf{e}_i^\top\sigma_1\mathbf{u}_1\mathbf{v}_1^\top.
	\end{align}
    Taking the L2 norms, we have 	
	\begin{align}
		\left\|\mathbf{e}_i^\top\mathbf{G}^{(t_f)}\right\|_2^2 = \left\|\mathbf{e}_i^\top\sum_{j=1}^{d} \sigma_j\mathbf{u}_j\mathbf{v}_j^\top\right\|_2^2 = \left\|\sum_{j=1}^{d} \sigma_ju_{i,j}\mathbf{v}_j^\top\right\|_2^2 = \sum_{j=1}^{d}\left(\sigma_ju_{i,j}\right)^2,
	\end{align}
	and
	\begin{align}
		\left\|\mathbf{e}_i^\top\text{RA}_1(\mathbf{G}^{(t_f)})\right\|_2^2 = \left\|\mathbf{e}_i^\top \sigma_1\mathbf{u}_1\mathbf{v}_1^\top\right\|_2^2 = \left\| \sigma_1u_{i,1}\mathbf{v}_1^\top\right\|_2^2  = (\sigma_1^2u_{i,1})^2 \le \left\|\mathbf{e}_i^\top\mathbf{G}^{(t_f)}\right\|_2^2.
	\end{align}
This completes the proof.
\end{proof}
Next, we present a lemma that characterizes the property of $\nabla_{\mathbf{Z}^{(t)}}\mathcal{L}_{\mathcal{Q}_s}(\mathbf{W}^{(T_p+t_f)},\mathbf{Z}^{(T_p+t_f)})$

\begin{lemma}\label{lemma: low rank ft grad}
    The gradient $\nabla_{\mathbf{Z}^{(T_p+t_f)}}\mathcal{L}_{\mathcal{Q}_s}(\mathbf{W}^{(T_p+t_f)},\mathbf{Z}^{(T_p+t_f)})$ satisfies
    \begin{align}
        \left|\mathbf{b}_1^\top\text{RA}_1\!\left(\nabla_{\mathbf{Z}^{(T_p+t_f)}}\mathcal{L}_{\mathcal{Q}_s}(\mathbf{W}^{(T_p+t_f)},\mathbf{Z}^{(T_p+t_f)})\right)\!\mathbf{b}_2\right| \!\le\! \left|\mathbf{b}_1^\top\nabla_{\mathbf{Z}^{(T_p+t_f)}}\mathcal{L}_{\mathcal{Q}_s}(\mathbf{W}^{(T_p+t_f)},\mathbf{Z}^{(T_p+t_f)}) \mathbf{b}_2\right|,
    \end{align}
    for any $\mathbf{b}_1,\mathbf{b}_2\in\{\mathbf{s}_j:{[\mathbf{s}_j\:\:\mathbf{p}]}\in\mathcal{Q}_s\}\cup\{\mathbf{p}\}$.
\end{lemma}
\begin{proof}
Suppose that $\mathbf{M}$ is a unitary matrix, and each $i^{th}$ column corresponds to the normalized embedding whose index is $i$.
By the Eckart–Young–Mirsky theorem, we have that
	\begin{align}
		\text{RA}_1\left(\nabla_{\mathbf{Z}^{(T_p+t_f)}}\mathcal{L}_{\mathcal{Q}_s}(\mathbf{W}^{(T_p+t_f)},\mathbf{Z}^{(T_p+t_f)})\right) = \sigma_1\mathbf{u}_1\mathbf{v}_1^\top,
	\end{align}
        where $\sigma_1\ge \sigma_2\ge \cdots \ge \sigma_d$ are the singular values of $\nabla_{\mathbf{Z}^{(T_p+t_f)}}\mathcal{L}_{\mathcal{Q}_s}(\mathbf{W}^{(T_p+t_f)},\mathbf{Z}^{(T_p+t_f)})$.
    We have 
    \begin{align}
        \mathbf{M}^\top \nabla_{\mathbf{Z}^{(T_p+t_f)}}\mathcal{L}_{\mathcal{Q}_s}(\mathbf{W}^{(T_p+t_f)},\mathbf{Z}^{(T_p+t_f)})\mathbf{M} = \sum_{i=1}^d \sigma_i \mathbf{M}^\top\mathbf{u}_i\mathbf{v}_i^\top\mathbf{M},
    \end{align}
    and
    \begin{align}
        \mathbf{M}^\top \text{RA}_1\left(\nabla_{\mathbf{Z}^{(T_p+t_f)}}\mathcal{L}_{\mathcal{Q}_s}(\mathbf{W}^{(T_p+t_f)},\mathbf{Z}^{(T_p+t_f)})\right)\mathbf{M} = \sigma_1\mathbf{M}^\top\mathbf{u}_1\mathbf{v}_1^\top\mathbf{M}.
    \end{align}
    Then, we have
    \begin{equation}
    \begin{aligned}
		&\left\|\mathbf{e}_i^\top\mathbf{M}^\top \nabla_{\mathbf{Z}^{(T_p+t_f)}}\mathcal{L}_{\mathcal{Q}_s}(\mathbf{W}^{(T_p+t_f)},\mathbf{Z}^{(T_p+t_f)})\mathbf{M}\right\|_2^2\\
        =& \left\|\mathbf{e}_i^\top\sum_{j=1}^{d} \sigma_j\mathbf{M}^\top\mathbf{u}_j\mathbf{v}_j^\top\mathbf{M}\right\|_2^2\\
        =& \left\|\sum_{j=1}^{d} \sigma_j\mathbf{e}_i^\top\mathbf{M}^\top\mathbf{u}_{j}\mathbf{v}_j^\top\mathbf{M}\right\|_2^2\\ =& \sum_{j=1}^{d}\left(\sigma_j\mathbf{e}_i^\top\mathbf{M}^\top\mathbf{u}_{j}\right)^2,
	\end{aligned}
    \end{equation}
	and
    \begin{equation}
	\begin{aligned}
		&\left\|\mathbf{e}_i^\top\mathbf{M}^\top \text{RA}_1\left(\nabla_{\mathbf{Z}^{(T_p+t_f)}}\mathcal{L}_{\mathcal{Q}_s}(\mathbf{W}^{(T_p+t_f)},\mathbf{Z}^{(T_p+t_f)})\right)\mathbf{M}\right\|_2^2\\ 
        =& \left\|\mathbf{e}_i^\top\sigma_1\mathbf{M}^\top\mathbf{u}_1\mathbf{v}_1^\top\mathbf{M}\right\|_2^2\\
        =& \left\|\mathbf{e}_i^\top\mathbf{M}^\top \nabla_{\mathbf{Z}^{(T_p+t_f)}}\mathcal{L}_{\mathcal{Q}_s}(\mathbf{W}^{(T_p+t_f)},\mathbf{Z}^{(T_p+t_f)})\mathbf{M}\right\|_2^2.
	\end{aligned}
    \end{equation}
    During fine-tuning, only $(\mathbf{M}^\top\mathbf{Z}^{(T_p+t_f)}\mathbf{M})_{\mathcal{I}(\mathbf{s}_j),\mathcal{I}(\mathbf{p})}$, for all $[\mathbf{s}_j\:\:\mathbf{p}]\in\mathcal{Q}_s$ and $(\mathbf{M}^\top\mathbf{Z}^{(T_p+t_f)}\mathbf{M})_{\mathcal{I}(\mathbf{p}),\mathcal{I}(\mathbf{p})}$ are updated. 
    So, we have
    \begin{equation}
    \begin{aligned}
        &\left\|\mathbf{e}_{\mathcal{I}(\mathbf{s}_j)}^\top\mathbf{M}^\top \nabla_{\mathbf{Z}^{(T_p+t_f)}}\mathcal{L}_{\mathcal{Q}_s}(\mathbf{W}^{(T_p+t_f)},\mathbf{Z}^{(T_p+t_f)})\mathbf{M}\right\|_2 = \left|\mathbf{s}_j^\top \nabla_{\mathbf{Z}^{(T_p+t_f)}}\mathcal{L}_{\mathcal{Q}_s}(\mathbf{W}^{(T_p+t_f)},\mathbf{Z}^{(T_p+t_f)})\mathbf{p}\right|,\\
        &\left\|\mathbf{e}^\top_{\mathcal{I}(\mathbf{p})}\mathbf{M}^\top \nabla_{\mathbf{Z}^{(T_p+t_f)}}\mathcal{L}_{\mathcal{Q}_s}(\mathbf{W}^{(T_p+t_f)},\mathbf{Z}^{(T_p+t_f)})\mathbf{M}\right\|_2 = \left|\mathbf{p}^\top \nabla_{\mathbf{Z}^{(T_p+t_f)}}\mathcal{L}_{\mathcal{Q}_s}(\mathbf{W}^{(T_p+t_f)},\mathbf{Z}^{(T_p+t_f)})\mathbf{p}\right|, \\
        &\left\|\mathbf{e}^\top_{\mathcal{I}(\mathbf{s}_j)}\mathbf{M}^\top \text{RA}_1\left(\nabla_{\mathbf{Z}^{(T_p+t_f)}}\mathcal{L}_{\mathcal{Q}_s}(\mathbf{W}^{(T_p+t_f)},\mathbf{Z}^{(T_p+t_f)})\right)\mathbf{M}\right\|_2\\
        =& \left|\mathbf{s}_j^\top \text{RA}_1\left(\nabla_{\mathbf{Z}^{(T_p+t_f)}}\mathcal{L}_{\mathcal{Q}_s}(\mathbf{W}^{(T_p+t_f)},\mathbf{Z}^{(T_p+t_f)})\right)\mathbf{p}\right|,\\
        &\left\|\mathbf{e}^\top_{\mathcal{I}(\mathbf{p})}\mathbf{M}^\top \text{RA}_1\left(\nabla_{\mathbf{Z}^{(T_p+t_f)}}\mathcal{L}_{\mathcal{Q}_s}(\mathbf{W}^{(T_p+t_f)},\mathbf{Z}^{(T_p+t_f)})\right)\mathbf{M}\right\|_2\\
=& \left|\mathbf{p}^\top \text{RA}_1\left(\nabla_{\mathbf{Z}^{(T_p+t_f)}}\mathcal{L}_{\mathcal{Q}_s}(\mathbf{W}^{(T_p+t_f)},\mathbf{Z}^{(T_p+t_f)})\right)\mathbf{p}\right|.
    \end{aligned}
    \end{equation}
    This completes the proof.
\end{proof}

\subsection{Features in low-rank \ac{ft}}\label{appendix: activation_pattern}
In this subsection, we characterize the features during low-rank \ac{ft}.
\begin{lemma}\label{lemma: mag_n2}
	For all $j\in [N]$, $k\neq j$, during low-rank \ac{ft}, $t_f\in[T_f]$, the following holds:
	\begin{align}
		\frac{1}{m}\sum_{l=1}^m\langle\mathbf{w}^{(T_p+t_f)}_{\mathcal{I}(\mathbf{a}_j),l}, \mathbf{s}_k\rangle =\mathcal{O}\left(\frac{\log(d)}{\lambda d^{0.01}}\right).
	\end{align}
\end{lemma}
\begin{proof}
 During low-rank fine-tuning with $t\in[0,T_f-1]$, for all $l\in[m]$, $j\in[N]$, and $k\neq j$, we have
    \begin{equation}
    \begin{aligned}
        \langle\mathbf{w}^{(T_p+t+1)}_{\mathcal{I}(\mathbf{a}_j),l},\mathbf{s}_k\rangle - \langle\mathbf{w}^{(T_p+t)}_{\mathcal{I}(\mathbf{a}_j),l},\mathbf{s}_k\rangle
\le& \mathcal{O}\left(\frac{\lambda\eta_f}{mN_f\beta}\right),
    \end{aligned}
    \end{equation}
    by Lemmas \ref{lemma: sft grad 2} and \ref{lemma: low rank ft grad}.
    Therefore, by Lemma \ref{lemma: mag_n}, for all $t_f \in [T_f]$ with $T_f = \mathcal{O}(d^{-0.01}N_f\beta/|\mathcal{R}|)$ and for all $l\in[m]$, and $j\in[N]$, $k\neq j$, we have
\begin{equation}
    \begin{aligned}
         \langle\mathbf{w}^{(T_p+t_f)}_{\mathcal{I}(\mathbf{a}_j),l},\mathbf{s}_k\rangle \le  \langle\mathbf{w}^{(T_p+0)}_{\mathcal{I}(\mathbf{a}_j),l},\mathbf{s}_k\rangle = \mathcal{O}\left(T_f\frac{\lambda\eta_f}{mN_f\beta}\right) = \mathcal{O}\left(\frac{\log(d)}{\lambda d^{0.01}}\right).
    \end{aligned}
    \end{equation}
This completes the proof.
\end{proof}
\begin{lemma}\label{lemma: mag3}
	Under Condition \ref{condition: condition}, for all $j\in [N]$, $l\in \cup_{i\in\mathcal{B}_j}\mathcal{S}_{\text{r},j,i}^{(0)}$, during low-rank \ac{ft}, $t_f\in[T_f]$, the following holds:
	\begin{align}
		\langle\mathbf{w}^{(T_p+t_f)}_{\mathcal{I}(\mathbf{a}_j),l}, \mathbf{s}_j\rangle =\Omega\left(\frac{\log(d)}{\lambda}\right).
	\end{align}
\end{lemma}
\begin{proof}
	According to Lemma \ref{lemma: pt_stage3}, for all $j\in [N]$ and $l\in \cup_{i\in\mathcal{B}_j}\mathcal{S}_{\text{r},j,i}^{(0)}$, we have
	\begin{align}
		\langle\mathbf{w}^{(T_p+0)}_{\mathcal{I}(\mathbf{a}_j),l}, \mathbf{s}_j\rangle =\Omega(\log(d)/\lambda).
	\end{align} 
	By Lemma \ref{lemma: sft grad 2}, for all $l\in[m]$, the update of $\mathbf{w}^{(T_p+t_f)}_{\mathcal{I}(\mathbf{a}_j),l}$ satisfies
	\begin{align}
		\langle\mathbf{w}^{(T_p+t_f+1)}_{\mathcal{I}(\mathbf{a}_j),l},\mathbf{s}_j\rangle - \langle\mathbf{w}^{(T_p+t_f)}_{\mathcal{I}(\mathbf{a}_j),l},\mathbf{s}_j\rangle \ge -\mathcal{O}\left(\frac{\lambda\eta_f}{mN_f\beta}\right).
	\end{align}
    For $t_f \le T_f = \mathcal{O}(N_f\beta d^{-0.01}/|\mathcal{R}|)$ and $\eta_f = \Theta(m|\mathcal{R}|\log(d)/\lambda^2)$, we have
    \begin{align}
		\langle\mathbf{w}^{(T_p+t_f)}_{\mathcal{I}(\mathbf{a}_j),l},\mathbf{s}_j\rangle - \langle\mathbf{w}^{(T_p+0)}_{\mathcal{I}(\mathbf{a}_j),l},\mathbf{s}_j\rangle \ge -\mathcal{O}(d^{-0.01}\log(d)/\lambda).
	\end{align}
    Hence, for all $l\in \cup_{i\in\mathcal{B}_j}\mathcal{S}_{\text{r},j,i}^{(0)}$, we have
    \begin{align}
		\langle\mathbf{w}^{(T_p+t_f)}_{\mathcal{I}(\mathbf{a}_j),l}, \mathbf{s}_j\rangle = \Omega(\log(d)/\lambda).
	\end{align}
Averaging over all $l\in[m]$ completes the proof.
\end{proof}

\subsection{Attention scores during low-rank FT}
With the same proof as Lemma \ref{lemma: attn_fft}, we have the following lemma.
\begin{lemma}
    Under Condition \ref{condition: condition}, during low-rank FT, the following holds:
    \begin{align}
       \frac{1}{2}\le \frac{\alpha_{1}^{(T_p+0)}([\mathbf{s}_j\:\:\mathbf{p}])}{\alpha_2^{(T_p+0)}([\mathbf{s}_j\:\:\mathbf{p}])} \le 2.
    \end{align}
\end{lemma}

\subsection{Proof of the result for low-rank \ac{ft} in Statement 2 of Theorem \ref{theorem: main}}
Recall that at the first FT iteration, the pre-trained model satisfies
\begin{equation}
\begin{aligned}
    \sum_{j\in\mathcal{J}_i}\text{logit}_{\mathcal{I}(\mathbf{r}_i)}^{(T_p+0)}([\mathbf{s}_j\:\:\mathbf{p}]) = \Theta\left(\frac{N_f\beta}{K}\cdot \frac{K}{|\mathcal{R}|}\right) = \Theta\left(\frac{N_f\beta}{|\mathcal{R}|}\right),
\end{aligned}
\end{equation}
and
\begin{align}
    1-\text{logit}^{(T_p+0)}_{\mathcal{I}(\mathbf{a}_j)}([\mathbf{s}_j\:\:\mathbf{p}]) = \Theta(1).
\end{align}
By Lemma \ref{lemma: sft grad 1} and Lemma \ref{lemma: select_num}, with probability $1-|\mathcal{R}|\exp\left(-\frac{N_f\beta K^2}{(2|\mathcal{R}|^2)}\right)$, for all $i\in\mathcal
{R}$, the updates with low-rank FT satisfy
\begin{equation}
\begin{aligned}
	\frac{1}{m}\sum_{k=1}^{m}\langle\mathbf{w}^{(T_p+1)}_{\mathcal{I}(\mathbf{r}_i),k}, \mathbf{p}\rangle - \frac{1}{m}\sum_{k=1}^{m}\langle\mathbf{w}^{(T_p+0)}_{\mathcal{I}(\mathbf{r}_i),k},\mathbf{p}\rangle\le& -\Theta\left(\frac{\eta_f\lambda}{N_f\beta}\frac{N_f\beta K}{|\mathcal{R}|}\frac{1}{Km}\right) + \mathcal{O}\left(\frac{\lambda\eta_f}{\sqrt{mN_f\beta}}\right)\\
    =& -\Theta\left(\frac{\lambda\eta_f}{|\mathcal{R}|m}\right).
\end{aligned}
\end{equation}
By Lemma \ref{lemma: sft grad 2}, for all $j\in \mathcal{J}$, we have
\begin{align}
\frac{1}{m}\sum_{k=1}^{m}\langle\mathbf{w}^{(T_p+1)}_{\mathcal{I}(\mathbf{a}_j),k}, \mathbf{p}\rangle - \frac{1}{m}\sum_{k=1}^{m}\langle\mathbf{w}^{(T_p+0)}_{\mathcal{I}(\mathbf{a}_j),k},\mathbf{p}\rangle =\mathcal{O}\left(\frac{\lambda\eta_f}{N_f\beta m}\right),
\end{align}
and
\begin{align}
	\frac{1}{m}\sum_{k=1}^{m}\langle\mathbf{w}^{(T_p+1)}_{\mathcal{I}(\mathbf{a}_j),k}, \mathbf{s}_i\rangle - \frac{1}{m}\sum_{k=1}^{m}\langle\mathbf{w}^{(T_p+0)}_{\mathcal{I}(\mathbf{a}_j),k},\mathbf{s}_i\rangle = \mathcal{O}\left(\frac{\lambda\eta_f}{N_f\beta m}\right).
\end{align}

When $\eta_f = \Theta(\frac{m|\mathcal{R}|\log(d)}{\lambda^2})$, within a single low-rank FT iteration, we have

\begin{align}
	\frac{1}{m}\sum_{k=1}^{m}\langle\mathbf{w}^{(T_p+1)}_{\mathcal{I}(\mathbf{r}_i),k}, \mathbf{p}\rangle - \frac{1}{m}\sum_{k=1}^{m}\langle\mathbf{w}^{(T_p+0)}_{\mathcal{I}(\mathbf{r}_i),k},\mathbf{p}\rangle = -\Theta\left(\frac{\log(d)}{\lambda}\right),
\end{align}
and
\begin{align}
	\frac{1}{m}\sum_{k=1}^{m}\langle\mathbf{w}^{(T_p+1)}_{\mathcal{I}(\mathbf{a}_i),k}, \mathbf{p}\rangle - \frac{1}{m}\sum_{k=1}^{m}\langle\mathbf{w}^{(T_p+0)}_{\mathcal{I}(\mathbf{a}_i),k},\mathbf{p}\rangle = \mathcal{O}\left(\frac{\lambda\eta_f}{N_f\beta m}\right) = \mathcal{O}\left(\frac{|\mathcal{R}|\log(d)}{\lambda N_f\beta}\right).
\end{align}
Consequently, when  $\beta N_f> C_2 m|\mathcal{R}|^2$, after one iteration, we already have
\begin{align}
	\max_{k\in\mathcal{R}} x^\text{output}_{\mathcal{I}(\mathbf{r}_k)}(\mathbf{W}^{(T_p+1)},\mathbf{Z}^{(T_p+1)},[\mathbf{s}_j\:\:\mathbf{p}]) < x^\text{output}_{\mathcal{I}(\mathbf{a}_j)}(\mathbf{W}^{(T_p+1)},\mathbf{Z}^{(T_p+1)},[\mathbf{s}_j\:\:\mathbf{p}]),
\end{align}
\begin{align}
	\max_{k\in[N]\backslash\{j\}}x^\text{output}_{\mathcal{I}(\mathbf{a}_k)}(\mathbf{W}^{(T_p+1)},\mathbf{Z}^{(T_p+1)},[\mathbf{s}_j\:\:\mathbf{p}]) < x^\text{output}_{\mathcal{I}(\mathbf{a}_j)}(\mathbf{W}^{(T_p+1)},\mathbf{Z}^{(T_p+1)},[\mathbf{s}_j\:\:\mathbf{p}]),
\end{align}
and 
\begin{align}
	\max_{l\notin\{\mathcal{I}(\mathbf{a}_j)\}_{j\in[N]}\cup\{\mathcal{I}(\mathbf{r}_i)\}_{i\in \mathcal{R}}}x^\text{output}_{l}(\mathbf{W}^{(T_p+1)},\mathbf{Z}^{(T_p+1)},[\mathbf{s}_j\:\:\mathbf{p}]) < x^\text{output}_{\mathcal{I}(\mathbf{a}_j)}(\mathbf{W}^{(T_p+1)},\mathbf{Z}^{(T_p+1)},[\mathbf{s}_j\:\:\mathbf{p}]),
\end{align}

Within $T_f = \mathcal{O}(d^{-0.01}\frac{N_f\beta}{|\mathcal{R}|})$, we have
\begin{align}
    	\frac{1}{m}\sum_{k=1}^{m}\langle\mathbf{w}^{(T_p+T_f)}_{\mathcal{I}(\mathbf{a}_i),k}, \mathbf{p}\rangle - \frac{1}{m}\sum_{k=1}^{m}\langle\mathbf{w}^{(T_p+T_f)}_{\mathcal{I}(\mathbf{a}_i),k},\mathbf{p}\rangle = \mathcal{O}\left(\frac{|\mathcal{R}|\log(d)}{\lambda N_f\beta}T_f\right) = \mathcal{O}\left(\frac{d^{-0.01}\log(d)}{\lambda}\right).
\end{align}
Hence, by Lemmas \ref{lemma: mag_n2} and \ref{lemma: mag3},
within $T_f = \mathcal{O}(d^{-0.01}\frac{N_f\beta}{|\mathcal{R}|})$ the outputs of the transformer also satisfy 
\begin{align}
	\max_{k\in\mathcal{R}} x^\text{output}_{\mathcal{I}(\mathbf{r}_k)}(\mathbf{W}^{(T_p+t_f)},\mathbf{Z}^{(T_p+t_f)},[\mathbf{s}_j\:\:\mathbf{p}]) < x^\text{output}_{\mathcal{I}(\mathbf{a}_j)}(\mathbf{W}^{(T_p+t_f)},\mathbf{Z}^{(T_p+t_f)},[\mathbf{s}_j\:\:\mathbf{p}]),
\end{align}
\begin{align}
	\max_{k\in[N]\backslash\{j\}}x^\text{output}_{\mathcal{I}(\mathbf{a}_k)}(\mathbf{W}^{(T_p+t_f)},\mathbf{Z}^{(T_p+t_f)},[\mathbf{s}_j\:\:\mathbf{p}]) < x^\text{output}_{\mathcal{I}(\mathbf{a}_j)}(\mathbf{W}^{(T_p+t_f)},\mathbf{Z}^{(T_p+t_f)},[\mathbf{s}_j\:\:\mathbf{p}]),
\end{align}
and 
\begin{align}
	\max_{l\notin\{\mathcal{I}(\mathbf{a}_j)\}_{j\in[N]}\cup\{\mathcal{I}(\mathbf{r}_i)\}_{i\in \mathcal{R}}}x^\text{output}_{l}(\mathbf{W}^{(T_p+t_f)},\mathbf{Z}^{(T_p+t_f)},[\mathbf{s}_j\:\:\mathbf{p}]) < x^\text{output}_{\mathcal{I}(\mathbf{a}_j)}(\mathbf{W}^{(T_p+t_f)},\mathbf{Z}^{(T_p+t_f)},[\mathbf{s}_j\:\:\mathbf{p}]),
\end{align}
for all $t_f\in[1,T_f]$.

This finishes the proof and concludes that
\begin{align}
	\mathcal{L}_e(\mathbf{W}^{(T_p+T_f)},\mathbf{Z}^{(T_p+T_f)}) \le |\mathcal{R}|\exp\left(-\frac{N_f\beta K^2}{2|\mathcal{R}|^2}\right).
\end{align}
This completes the proof.

\subsection{Proof of the result for low-rank \ac{ft} in Statement 3 of Theorem \ref{theorem: main}}
When $\beta N_fK/|\mathcal{R}| \le C_3$ with $0<C_3<1$, there are more than $(1-C_3)|\mathcal{R}|$ relation elements, denoted as set $\mathcal{C}$, not been covered in the \ac{ft} training set.
For the elements not covered, $j\in\mathcal{C}$,  by Lemma \ref{lemma: sft grad 2}, with $\eta_f = \Theta(m|\mathcal{R}|\log(d)/\lambda^2)$ and $T_f = \mathcal{O}(d^{-0.01}N_f\beta/|\mathcal{R|})$, we have
\begin{align}
	\frac{1}{m}\sum_{k=1}^{m}\sigma(\langle\mathbf{w}^{(T_p+t_f)}_{\mathcal{I}(\mathbf{r}_j),k}, \mathbf{p}\rangle)\le& \frac{1}{m}\sum_{k=1}^{m}\sigma(\langle\mathbf{w}^{(T_p+0)}_{\mathcal{I}(\mathbf{r}_j),k}, \mathbf{p}\rangle) + \mathcal{O}\left(t_f\frac{\eta_f\lambda}{mN_f\beta d}\right)\\
    =& \tilde{\mathcal{O}}\left(\frac{\sigma_0}{\sqrt{m}}\right) + \mathcal{O}\left(\frac{\log(d)}{d\lambda}\right).
\end{align}

For all $i\in\mathcal{C}$ and $\mathcal{D}_i$ and any $t_f$, we have
\begin{align}
    \frac{1}{m}\sum_{l=1}^m\sigma(\langle\mathbf{w}^{(T_p+t_f)}_{\mathcal{I}(\mathbf{r}_i),l},[\mathbf{s}_j\;\;\mathbf{p}]\rangle) = \Theta(1),
\end{align}
resulting in 
\begin{align}
     \frac{1}{m}\sum_{l=1}^m\sigma(\langle\mathbf{w}^{(T_p+t_f)}_{\mathcal{I}(\mathbf{r}_i),l},[\mathbf{s}_j\;\;\mathbf{p}]\rangle) >  \frac{1}{m}\sum_{l=1}^m\sigma(\langle\mathbf{w}^{(T_p+t_f)}_{\mathcal{I}(\mathbf{a}_j),l},[\mathbf{s}_j\;\;\mathbf{p}]\rangle).
\end{align}

For the remaining $(1-\beta)N_f+N_r$ data $([\mathbf{s}_j\;\;\mathbf{p}],\mathcal{I}(\mathbf{a}_j))\in\mathcal{Q}\backslash \mathcal{Q}_s$, we have
\begin{align}
	\mathbb{P}[j\in\mathcal{D}_i, \forall i\in \mathcal{C}] \ge 1- C_3^{\tilde{K}(j)}.
\end{align}

Therefore, we have
\begin{align}
	\mathcal{L}_e(\mathbf{W}^{(T_p+T_f)},\mathbf{Z}^{(T_p+T_f)}) \ge 1-C_3.
\end{align}
This completes the proof.

\section{Claim of $\delta$}
In total, using union bound, with probability $1-\Theta(\delta')$, the high probability lemmas in the above proofs including Lemmas \ref{lemma: attn_init_order}, \ref{lemma: init}, \ref{lemma: cover1}, \ref{lemma:init}, \ref{lemma: aux1}, \ref{lemma: pt_stage2}, \ref{lemma: s mag} hold.
In the formal Statement of Theorem \ref{theorem: main}, we use $\delta = \Theta(\delta')$.

\section{Experimental settings}\label{app: experiemntal settings}
\subsection{Settings for Figures 2a and 2b}\label{app: 1l details}
We train a one-layer transformer based on the architecture presented in Section 2.
The architecture and training parameters are as follows.
\begin{itemize}
	\item Number of neurons: $m = 100$
	\item Scaling factor: $\lambda = 200$ 
	\item Optimizer: SGD with learning rates $\eta_1 = 0.1,\eta_2 = 0.05$ and $\eta_3 = 0.01$
\end{itemize}
To ensure the model fully converges, during pre-training, we implement $h_1=3$ for training stage 1, $h_2 = 5003$ for training stage 2, and a total of $10003$ iterations.
During fine-tuning, we implement a total of 2000 iterations to determine the optimal number of iterations for generalization. 

To avoid the dominance of the gradient from frequent facts over those of rare facts as discussed in Section \ref{sec: results}, 
we reweight the gradients from \textsf{freq} and \textsf{rare} by multiplying the gradients from \textsf{rare} with multiplicity $K$.

All experiments are run over 5 random seeds.

\subsection{Settings for Figures 3a and 3b}\label{app: gpt2 details}
We train a standard GPT-2 with 12 layers, 12 heads, and 768 hidden dimensions without implementing any positional encoding.
For pre-training, we employ full fine-tuning and the AdamW optimizer with a learning rate of $1\times 10^{-5}$ and a weight decay of $0.1$.

For fine-tuning, we employ the Adam optimizer with learning rate $5\times 10^{-6}$.
We implement a total of 2000 iterations and find the best number of iterations for generalization. 

To avoid the dominance of the gradient from frequent facts over those of rare facts, as discussed in Section \ref{sec: results}, we reweight the gradients from frequent facts and rare facts by multiplying the gradients from rare facts by multiplicity $K$.

All experiments are run over 5 random seeds.

\subsection{Settings for Figure 4}\label{app: qa details}
We train a standard GPT-2 with 12 layers, 12 heads, and 786 hidden dimensions.

We employ the AdamW optimizer for pre-training, using a learning rate of $2\times 10^{-4}$ and a weight decay of $0.1$.
We also employ a learning rate warmup with 50 steps and a cosine learning rate schedule.
We implement a batch size of 128 and a total of $150$ epochs.

For fine-tuning, we employ full fine-tuning with the AdamW optimizer, using a learning rate of $1\times 10^{-4}$ and a weight decay of 0.01.
We employ a learning rate warmup with 20 steps and a constant learning rate schedule.
We implement a batch size of 128 and a total of $50$ epochs.

\subsection{Settings for Table 2}\label{app: table details}
We employ full fine-tuning and the AdamW optimizer with a learning rate of $1\times 10^{-4}$ and weight decay of 0.01.
We employ a constant learning rate schedule.
We implement a batch size of 128 and a total of $50$ epochs.

\section{Pre-training and fine-tuning question answering examples constructed from PopQA 'director' relation}\label{app: director data}
We show some pre-training and fine-tuning examples of the constructed dataset from PopQA 'director' relation.
\begin{table}[h]
	 \small
	\centering
    \caption{Examples of constructed pre-training and fine-tuning data with PopQA under the setting of Figure 2e.}
	\renewcommand{\arraystretch}{1.2}
	\setlength{\tabcolsep}{2pt}
	\begin{tabular}{p{5cm}|p{4cm}|p{3cm}}
		   \toprule
		\textbf{Pre-training examples}&\textbf{Question prompts}&\textbf{Answers}\\
		\midrule
		Atlantic was helmed by Ewald André Dupont [EOS] & Question: Who was the director of Atlantic? Answer: & Ewald André Dupont \\\midrule
		Bhoopathi Ranga was directed by Geethapriya [EOS] & Question: Who was the director of Bhoopathi Ranga? & Geethapriya\\
		\midrule
		The Cell was crafted by Tarsem Singh [EOS] & Question: Who was the director of Atlantic? Answer: & Tarsem Singh\\
		\midrule
		The Good Earth was shaped by Sidney Franklin [EOS] & Question: Who was the director of The Good Earth? Answer: & Sidney Franklin\\
\bottomrule
	\end{tabular}
	
	\label{table: example}
\end{table}

\section{Fine-tuning question answering examples constructed from PopQA 'capital' relation}\label{app: capital data}
We show fine-tuning examples of the constructed dataset from PopQA 'capital' relation.
\begin{table}[h]
	\small
	\centering
    \caption{Examples of constructed fine-tuning data with PopQA under the setting of Table 3.}
	\renewcommand{\arraystretch}{1.2}
	\setlength{\tabcolsep}{2pt}
	\begin{tabular}{p{4cm}|p{3cm}}
		\toprule
		\textbf{Question prompts}&\textbf{Answers}\\
		\midrule
		Question: Where is the capital of Ilocos Region? Answer: & San Fernando\\\midrule
		Question: Where is the capital of Cayman Islands? Answer: & George Town\\
		\midrule
		Question: Where is the capital of Massachusetts? Answer: & Boston\\
		\midrule
		Question: Where is the capital of Cherokee County? Answer: & Canton \\
		\bottomrule
	\end{tabular}
	\label{table: example2}
\end{table}

\section{Existing Assets and Licenses}
\label{app:assets}

We use several existing research assets in our experiments. PopQA is used as a publicly available question-answering benchmark, and we cite its original paper and repository. GPT-2 and Llama-3.2-1B are used as pretrained language models for evaluation and are cited accordingly. We use these assets only for research evaluation and follow their respective licenses and terms of use. We do not redistribute model weights, datasets, or derived high-risk systems.